\documentclass{article}

\usepackage{microtype}
\usepackage{graphicx}
\usepackage{subcaption}
\usepackage{booktabs} % for professional tables
\usepackage{multirow}

\usepackage{array}     % for m{<width>} columns
\usepackage{makecell}

\usepackage{hyperref}

\usepackage[preprint]{icml2026}
\makeatletter
\renewcommand{\ICML@preprint}{%
  \textit{Preprint. \today.} \\
  Individual contributions highlighted in Appendix~\ref{app:contributions}. \\
  Qualcomm AI Research is an initiative of Qualcomm Technologies, Inc.%
}
\makeatother

\usepackage{amsmath}
\usepackage{amssymb}
\usepackage{mathtools}
\usepackage{amsthm}
\usepackage{tikz}
\usepackage{aligned-overset}

\usepackage[capitalize,noabbrev]{cleveref}

\usepackage{siunitx}
\usepackage[most]{tcolorbox}

\hypersetup{
    colorlinks=true,
    linkcolor=blue,
    filecolor=magenta,      
    urlcolor=cyan,
    citecolor=blue,
}

\definecolor{systemgray}{RGB}{240,240,255}
\definecolor{userblue}{RGB}{230,245,255}
\definecolor{assistantgreen}{RGB}{230,255,240}
\definecolor{framegray}{RGB}{180,180,200}
\definecolor{frameblue}{RGB}{150,200,230}
\definecolor{framegreen}{RGB}{150,220,180}

\tcbset{
  system/.style={colback=systemgray, colframe=framegray, title=System Prompt, fonttitle=\bfseries},
  user/.style={colback=userblue, colframe=frameblue, title=User Query, fonttitle=\bfseries},
  assistant/.style={colback=assistantgreen, colframe=framegreen, title=Assistant Response, fonttitle=\bfseries}
}

\newtcolorbox[auto counter]{myfloatbox}[2][]{%
  enhanced,
  floatplacement=htbp,
  float,
  colback=white,
  colframe=black,
  boxrule=1pt,
  arc=8pt,
  title={Box~\thetcbcounter: #2},
  fonttitle=\bfseries,
  label={#1} % <-- This ensures the label matches the counter
}

\newcommand*\circled[1]{\tikz[baseline=(char.base)]{\node[shape=circle,draw,inner sep=2pt] (char) {#1};}}

\newcommand{\defeq}{\vcentcolon=}
\newcommand{\eqdef}{=\vcentcolon}

\newcommand{\mean}[1]{\mathbb{E}\left[{#1}\right]}
\newcommand{\var}[1]{\text{Var}\left({#1}\right)}

\newcommand{\E}{\mathbb{E}}
\newcommand{\means}[2]{\mathbb{E}_{#1}\left[{#2}\right]}
\newcommand{\tp}[1]{{#1}^{\top}}
\newcommand{\mtx}[1]{\mathbf{#1}}
\newcommand{\vc}[1]{\mathbf{#1}}

\newcommand{\elem}[2]{\left[{#1}\right]_{#2}}
\newcommand{\norm}[1]{\left\|{#1}\right\|}
\newcommand{\ones}{\mathbf{1}}
\newcommand{\expnumber}[2]{{#1}\mathrm{e}{#2}}

\newcommand{\reals}{{\mathbb{R}}}

\newcommand{\naturals}{{\mathbb{N}}}

\newcommand{\sphere}{{\mathbb{S}}}
\newcommand{\ball}{{\mathbb{B}}}

\newcommand{\calA}{{\cal A}}

\newcommand{\calD}{{\cal D}}

\newcommand{\calU}{{\cal U}}

\newcommand{\calR}{{\cal R}}

\newcommand{\calN}{{\cal N}}

\newcommand{\qwen}{Qwen-3 0.6B}
\newcommand{\llama}{Llama-3.2 1B}

\theoremstyle{plain}
\newtheorem{theorem}{Theorem}[section]
\newtheorem*{theorem*}{Theorem}
\newtheorem{proposition}[theorem]{Proposition}
\newtheorem*{proposition*}{Proposition}
\newtheorem{lemma}[theorem]{Lemma}
\newtheorem{corollary}[theorem]{Corollary}
\theoremstyle{definition}
\newtheorem{definition}[theorem]{Definition}
\newtheorem{assumption}[theorem]{Assumption}
\theoremstyle{remark}
\newtheorem{remark}[theorem]{Remark}

\usepackage{xcolor}
\usepackage[framemethod=TikZ]{mdframed} % for nice rounded corners
\usepackage{etoolbox} % to patch environments

\mdfdefinestyle{thmbox}{
    backgroundcolor=gray!10,
    linecolor=gray!50,
    linewidth=0.5pt,
    roundcorner=4pt,
    skipabove=6pt,
    skipbelow=6pt,
    innerleftmargin=6pt,
    innerrightmargin=6pt,
    innertopmargin=6pt,
    innerbottommargin=6pt
}

\usepackage[textsize=tiny]{todonotes}

\icmltitlerunning{On Adaptivity in Zeroth-Order Optimization}

\begin{document}

\twocolumn[
  \icmltitle{On Adaptivity in Zeroth-Order Optimization}

  % It is OKAY to include author information, even for blind submissions: the
  % style file will automatically remove it for you unless you've provided
  % the [accepted] option to the icml2026 package.

  % List of affiliations: The first argument should be a (short) identifier you
  % will use later to specify author affiliations Academic affiliations
  % should list Department, University, City, Region, Country Industry
  % affiliations should list Company, City, Region, Country

  % You can specify symbols, otherwise they arle numbered in order. Ideally, you
  % should not use this facility. Affiliations will be numbered in order of
  % appearance and this is the preferred way.
  \icmlsetsymbol{equal}{*}

  \begin{icmlauthorlist}
    \icmlauthor{Hassan Dbouk}{equal,yyy}
    \icmlauthor{Nidham Gazagnadou}{equal,yyy}
    \icmlauthor{Matthias Reisser}{equal,yyy}
    \icmlauthor{Christos Louizos}{yyy}
    %\icmlauthor{}{sch}
    %\icmlauthor{}{sch}
  \end{icmlauthorlist}

  \icmlaffiliation{yyy}{Qualcomm AI Research}
  % \icmlaffiliation{yyy}{Qualcomm AI Research. Qualcomm AI Research is an initiative of Qualcomm Technologies, Inc.. Individual contributions highlighted in Appendix~\ref{app:contributions}.}
% who wants to be the corresponding author/
  \icmlcorrespondingauthor{}{hdbouk@qualcomm.qti.com}
  % \icmlcorrespondingauthor{Firstname2 Lastname2}{first2.last2@www.uk}

  % You may provide any keywords that you find helpful for describing your
  % paper; these are used to populate the "keywords" metadata in the PDF but
  % will not be shown in the document
  \icmlkeywords{Machine Learning, ICML}

  \vskip 0.3in
]
% this must go after the closing bracket ] following \twocolumn[ ...

% This command actually creates the footnote in the first column listing the
% affiliations and the copyright notice. The command takes one argument, which
% is text to display at the start of the footnote. The \icmlEqualContribution
% command is standard text for equal contribution. Remove it (just {}) if you
% do not need this facility.

% Use ONE of the following lines. DO NOT remove the command.
% If you have no special notice, KEEP empty braces:
\printAffiliationsAndNotice{}  % no special notice (required even if empty)
% Or, if applicable, use the standard equal contribution text:
% \printAffiliationsAndNotice{\icmlEqualContribution}

\begin{abstract}
We investigate the effectiveness of adaptive zeroth-order (ZO) optimization for memory-constrained fine-tuning of large language models (LLMs). 
% Problem observed
Contrary to prior claims, we show that adaptive ZO methods such as ZO-Adam offer no convergence advantage over well-tuned ZO-SGD, while incurring significant memory overhead. 
Our analysis reveals that in high dimensions, ZO gradients lack coordinate-wise heterogeneity, rendering adaptive mechanisms memory inefficient. 
% Suggested solution
% and Theory - Convergence proof
Leveraging this insight, we propose MEAZO, a memory-efficient adaptive ZO optimizer that tracks only a single scalar for global step size adaptation. 
We support our method with theoretical convergence guarantees under standard assumptions.
% Experiments - Performance
Experiments across multiple LLM families and tasks demonstrate that MEAZO matches ZO-Adam’s performance with the memory footprint of ZO-SGD.
% Experiments - Robustness 
Additional experiments on synthetic quadratic problems and LLM fine‑tuning further demonstrate MEAZO’s enhanced robustness to step size choices, particularly in grouped or block‑structured optimization settings.
\end{abstract}

\section{Introduction}
Fine-tuning large language models (LLMs) can specialize their behavior to various use-cases: aligning to a specific domain, application, task~\citep{gururangan2020don} or integrating user preferences~\citep{gao2024aligning} are some examples thereof. 
% Model adaptation to each particular use-case is needed for better performance.
% One way to align pre-trained models is via fine-tuning, that is an update of the model weights via first-order (FO) optimization.
% Backpropagation is generally so accepted as the go-to solution for gradient computation that it is rarely challenged. Its need for activation caching and floating-point precision computation however make it a resource-intensive algorithm that might not pose the right trade-offs in memory and energy-bound scenarios. 
Backpropagation is a standard but memory and computationally heavy fine-tuning algorithm, due to its need of activation caching and floating-point precision.

%% Too energy-efficiency oriented
% Zeroth-order (ZO) gradient approximation on the other hand does not require activation caching, does not require differentiability and is therefore, at least theoretically, directly applicable to integer-based arithmetic on inference-optimized hardware. Such ZO optimization methods have recently gained attention in the LLM literature as they allow the fine-tuning of models using only forward passes~\citep{mezo}.
Zeroth-order (ZO) methods on the other hand have recently gained attention in the LLM literature as they allow model fine-tuning using only forward passes~\citep{mezo}, and are thus more memory and computationally efficient.
% In a nutshell, ZO methods randomly perturb trainable parameters and compute the finite difference approximation of the first-order directional derivative as in~\eqref{eq:zo}.
% In a nutshell, ZO optimization computes the loss on a mini-batch of data for a model with randomly-perturbed trainable parameters. The loss-delta between different model perturbations defines the ``projected gradient'' and allows for the approximation of parameter gradients as in \eqref{eq:zo}.
In a nutshell, ZO optimization approximates the gradient through the loss difference at randomly-perturbed trainable parameters. 
% The model weights are modified with random perturbations and the finite difference of loss is then used to build an approximation of the first-order gradient. 

Due to the inherent stochasticity of parameter perturbations, ZO gradients exhibit high-variance leading to slower convergence in terms of number of optimization steps compared to first-order (FO) methods.
However, each ZO step is typically less computationally intensive than the backpropagation step required by FO methods.
Therefore, ZO vs. FO poses a trade-off in speed/energy to convergence given specific hardware constraints and task-specific characteristics. 

A substantial line of recent work aims to address this variance issue by introducing adaptive ZO algorithms, such as ZO‑Adam and its variants~\citep{chen2019zo, nazari2020adaptive, jiang2024zo, shu2025refining}. These methods mirror FO adaptive optimizers by tracking first and second moments of ZO gradients with the goal of improving convergence.

In this work, we revisit this assumption and show that, contrary to prior claims, FO‑style adaptivity provides limited benefits for ZO optimization in high‑dimensional transformer models. Recent studies on FO training dynamics~\citep{zhang2024transformers, tomihari2025understanding} demonstrate that Adam’s advantage over SGD arises from gradient heterogeneity: a strong coordinate‑wise variation in gradient magnitudes. We provide evidence that ZO gradient estimates, due to the isotropic nature of random perturbations in high dimension, are essentially homogeneous. As a consequence, adaptive per‑coordinate normalization has little meaningful structure to exploit.

Building on this insight, we propose a \underline{m}emory-\underline{e}fficient \underline{a}daptive \underline{z}eroth-\underline{o}rder (MEAZO) algorithm, which achieves global step size adaptivity by tracking only a \emph{single scalar} for gradient normalization\footnote{We show experimentally in Appendix~\ref{app:experiments} that first moment normalization is not required for adaptive ZO.}. This minimalist design enables MEAZO to match the memory footprint of ZO-SGD while retaining the convergence properties of ZO-Adam. Our experiments across multiple models and datasets show that properly tuned ZO‑SGD already matches or surpasses existing adaptive ZO optimizers, and that MEAZO achieves performance comparable to ZO‑Adam with negligible additional memory cost. Moreover, while ZO‑SGD performs strongly when well tuned, MEAZO exhibits greater robustness to step size (learning rate) choices, making it a practical, stable, and scalable alternative for ZO fine‑tuning in memory‑constrained environments.

\section{Problem Setup and Notation}\label{sec:setup}

\textbf{Notation}. We adopt the following conventions. Vectors are denoted by bold lowercase letters (\textit{e.g.}, $\vc{x}$, $\vc{y}$). Matrices are denoted by bold uppercase letters (\textit{e.g.}, $\vc{W}$, $\vc{A}$, $\vc{B}$). We use the following indexing notation: $\elem{\vc{x}}{k} \coloneq x_k$ to represent the $k^\text{th}$ entry of $\vc{x}$. %Sets are denoted by calligraphic letters (\textit{e.g.}, $\mathcal{J}$, $\mathcal{V}$, $\mathcal{L}$). 

\textbf{Setup}. We study the standard stochastic optimization problem:
\begin{equation}\label{eq:opt}
    \min_{\vc{x} \in \reals^d} F(\vc{x}) \coloneq \means{\xi \sim \calD}{f(\vc{x}; \xi)},
\end{equation}
% \ngdi{According to ~\cite{ghadimi2013stochastic,nesterov2017random} it is common to have the opposite notations
% $\min_{\vc{x} \in \reals^d} f(\vc{x}) \coloneq \means{\xi \sim \calD}{F(\vc{x}; \xi)},$
% Eg, eq (3.1) and (3.2) of \cite{ghadimi2013stochastic}. 

%     Let us check other papers and fix the proofs we have now before changing notations.
% }
where \( F: \reals^d \to \reals \) is the objective function, $\xi \sim \calD$ denotes samples from a data distribution $\calD$, and $f$ is the sample loss function. The variable $\vc{x} \in \reals^d$ is the $d$-dimensional parameter vector we aim to optimize. When clear from context, we drop the dependence on $\xi$ and write $f(\vc{x};\xi)$ as $f(\vc{x})$ for simplicity.

We will frequently use the notion of smoothness, as it is standard in non-convex optimization. In particular, we assume that the sample loss function is $L$-smooth:
\begin{definition}[$L$-smoothness]\label{def:smoothness}
    A differentiable function $f:\reals^d \to \reals$ is said to be \emph{$L$-smooth} if its gradient is $L$-Lipschitz continuous, that is,
    \[
        \|\nabla f(\vc{x}) - \nabla f(\vc{y})\| \leq L \|\vc{x} - \vc{y}\|, \quad \forall \vc{x},\vc{y} \in \reals^d.
    \]
\end{definition}

\textbf{ZO Optimization}. In the ZO setting, we aim to solve \eqref{eq:opt} without access to exact gradients of $f$. Instead, optimization relies solely on function evaluations. A central building block is the finite-difference approximation of the directional derivative:
\begin{definition}[Projected Gradient]
For a point $\vc{x} \in \reals^d$, a direction vector $\vc{u} \in \reals^d$,
and a perturbation magnitude $\varepsilon > 0$, the \emph{ZO projected gradient}
(finite-difference directional derivative) is defined as  
\[
    \Delta f_\varepsilon(\vc{x};\vc{u})
    \coloneq
    \frac{f(\vc{x} + \varepsilon \vc{u}) 
          - f(\vc{x} - \varepsilon \vc{u})}{2\varepsilon}
    \approx \tp{\vc{u}} \nabla f(\vc{x}),
\]
which provides a zeroth-order estimate of the directional derivative of $f$ at $\vc{x}$
along the direction $\vc{u}$ (when $f$ is differentiable).
\end{definition}
Using this building block, we can construct an estimate of the full gradient through random sampling of direction vectors.

\begin{definition}[ZO Gradient Estimator]
Given a function $f: \reals^d \to \reals$, the $q$-sample zeroth-order (ZO) gradient estimator \citep{mezo} is defined as
\begin{align}\label{eq:zo}
    \hat{\nabla} f_{\varepsilon}^q(\vc{x})
    &\coloneq
    \frac{1}{q} \sum_{i=1}^q
    \Delta f_\varepsilon(\vc{x}; \vc{u}_i) \, \vc{u}_i,
\end{align}
where $\{\vc{u}_i\}_{i=1}^q$ are i.i.d.\ random direction vectors satisfying
\[
    \mathbb{E}[\vc{u}] = \vc{0}, \qquad 
    \mathbb{E}[\vc{u} \vc{u}^\top] = \sigma \mathbf{I}_d,
\]
and $\varepsilon > 0$ is the finite-difference perturbation scale.
\end{definition}

Variants of \eqref{eq:zo} exist (\emph{e.g.}, choice of sampling distribution\footnote{When $\vc{u} \sim \mathrm{Uniform}(\mathbb{S})$, the estimator is scaled by $d$}, forward vs. central differences). Unless stated otherwise, we assume $\vc{u}_i \sim \calN(\vc{0}, \mtx{I}_d)$. These choices typically have minor impact on performance, as ZO methods exhibit similar behavior under standard settings.

The estimator in \eqref{eq:zo} provides an unbiased estimate of the gradient of a \emph{smoothed} version of the objective, defined as
\[
    f_\varepsilon(\vc{x}) \coloneq \mathbb{E}_{\vc{v}\sim\mathcal{Q}}\big[f(\vc{x}+\varepsilon \vc{v})\big].
\]
The smoothing distribution \(\mathcal{Q}\) depends on the perturbation scheme: for uniform perturbations, smoothing occurs over the unit ball \(\mathbb{B}\) \citep{flaxman2004online}, whereas for Gaussian perturbations, smoothing is under a normal distribution \citep{nesterov2017random}. This smoothing interpretation is key: ZO methods do not approximate the true gradient of \(f\) directly but rather the gradient of \(f_\varepsilon\).

\begin{figure}[tb]
  \centering
  \includegraphics[width=0.48\textwidth]{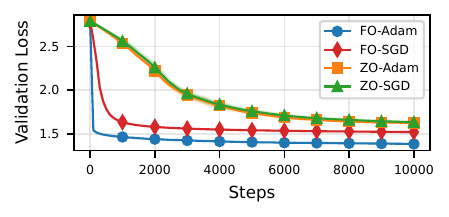}
  \caption{LoRA fine-tuning Llama-3.2 1B on XSum for various optimizers with optimal step size.}
  \label{fig:llama3_sst2_curves}
\end{figure}

\textbf{Block Coordinate ZO}. The estimator in \eqref{eq:zo} can suffer from high variance in high-dimensional settings. Several variance reduction techniques have been proposed, including \emph{block coordinate ZO} \citep{zhang2024revisiting}, which partitions the parameter space into smaller blocks and estimates gradients within each block.

\begin{figure*}[tb!]
\centering
\includegraphics[width=\textwidth]{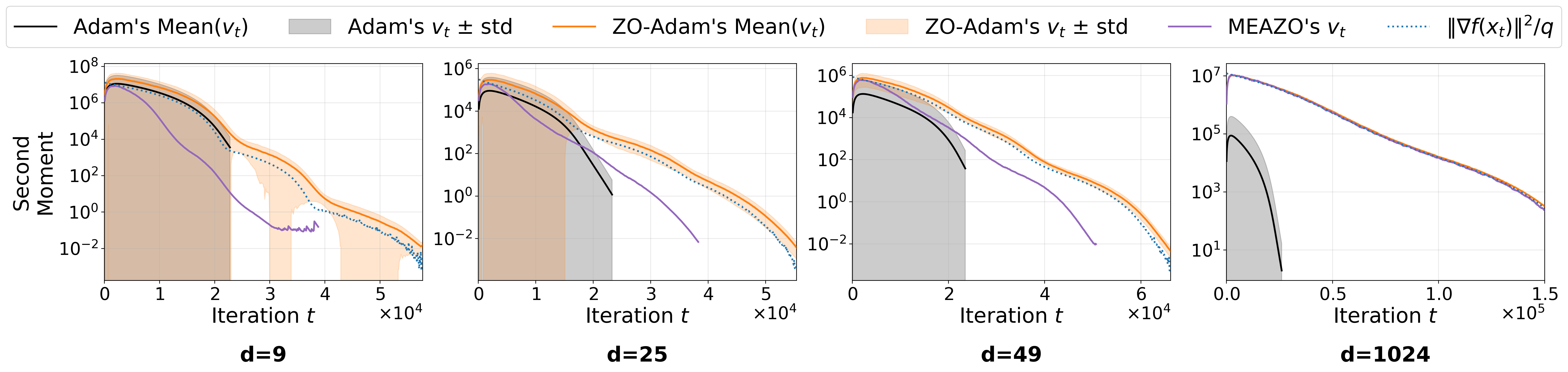}
\caption{Second moment estimate of FO-Adam, ZO-Adam and MEAZO across training on synthetic quadratic objectives.}
\label{fig:vt_track}
\end{figure*}

\begin{definition}[Grouped ZO Gradient]
Let \( f: \reals^d \to \reals \) be an objective function, and let $\{\mathcal{X}_j\}_{j=1}^p$ be a partition of the coordinate set $[d] \coloneq \{1,\dots,d\}$ into $p$ disjoint blocks. For each block $\mathcal{X}_j$, define a binary mask vector $\vc{m}_j \in \{0,1\}^d$ such that $\elem{\vc{m}_j}{k} = 1$ if $k \in \mathcal{X}_j$ and $0$ otherwise. The block $q$-sample ZO gradient estimator is:
\begin{equation}\label{eq:block-zo}
    % \hat{\nabla} f(\vc{x}) = \frac{1}{q} \sum_{i=1}^q \sum_{j=1}^p \Delta f_\varepsilon\big(\vc{x}; \vc{u}_i \odot \vc{p}_j\big)  (\vc{u}_i \odot \vc{p}_j),
    \hat{\nabla} f_{\varepsilon}^{q,p}(\vc{x}) = \frac{1}{q} \sum_{i=1}^q \sum_{j=1}^p \Delta f_\varepsilon\big(\vc{x}; \vc{m}_j \odot \vc{u}_i\big)  (\vc{m}_j \odot \vc{u}_i),
\end{equation}
where $\vc{u}_i$ are i.i.d. random direction vectors and $\odot$ denotes element-wise multiplication.
\end{definition}

When the partition consists of a single block, i.e., $p = 1$ and $\vc{m}_1 = \vc{1}_d$, the grouped estimator in \eqref{eq:block-zo} reduces to the standard ZO estimator in \eqref{eq:zo}.

\section{Adaptivity and Zeroth-order Optimization}\label{sec:adaptivity}

Coordinate-wise adaptive optimizers, such as Adam \citep{kingma2014adam}, have become the default choice for training and fine-tuning LLMs in first-order (FO) optimization settings. Their success stems from the ability to leverage per-coordinate gradient statistics, enabling effective step size adaptation across different dimensions. Specifically, Adam maintains bias-corrected exponential moving averages (EMAs) of the first moment, $\hat{\vc{m}}_t \in \reals^d$, and second moment, $\hat{\vc{v}}_t  \in \reals^d$, of the gradient and updates parameters $\vc{x}\in \reals^d$ as:
\begin{equation}
    \vc{x}_{t+1} = \vc{x}_t - \eta \frac{\hat{\vc{m}}_t}{\sqrt{\hat{\vc{v}}_t} + \zeta},
\end{equation}
where $\eta > 0$ is the step size and $\zeta > 0$ is a small quantity to improve numerical stability. 
This formulation allows Adam to adaptively scale updates for each coordinate based on historical gradient information.

This advantage does not naturally extend to the ZO optimization context, where such per-coordinate structure is absent. Unlike FO gradients, the estimator in \eqref{eq:zo} is dominated by randomness for practical values of \(q\) rather than per-coordinate gradient information. The estimator averages over only \(q\) random directions, where typically \(q \ll d\). Consequently, even when averaging, the signal remains highly isotropic and does not capture rich per-coordinate patterns. Consequently, applying coordinate-wise adaptive methods such as Adam in this setting introduces unnecessary complexity and computational overhead without providing meaningful benefits.

Proposition~\ref{prop:zo-sq} formalizes this intuition by analyzing the expected squared gradient:
% \begin{proposition}\label{prop:zo-sq}
%     Let \( f: \reals^d \to \reals \) be a differentiable function and $\hat{\nabla} f_{\varepsilon}^q(\vc{x}) $ be the $q$-sample ZO estimator in \eqref{eq:zo} with $\vc{u}\sim \calN(\vc{0}, \mtx{I}_d)$. Then in the $\varepsilon \rightarrow 0$ limit we have:
%     \begin{equation}
%         \elem{\means{\vc{u}}{\hat{g}^2_q(\vc{x})}}{k} \approx \frac{1}{q}\left( \norm{\nabla f(\vc{x})}^2 + \elem{\nabla f(\vc{x})}{k}^2\right) + \elem{\nabla f(\vc{x})}{k}^2.
%     \end{equation}
% \end{proposition}
% \ngdi{I checked it's correct.
% I can rewrite it in vector if need (see Prop below). Proof can use a nice theorem: Isserli's theorem.}
\begin{proposition}[Gaussian]\label{prop:zo-sq}
Let \( f: \reals^d \to \reals \) be an \(L\)-smooth function, and let 
\(\hat{\nabla} f_{\varepsilon}^q(\vc{x})\) denote the \(q\)-sample ZO estimator in \eqref{eq:zo}, 
where \(\vc{u}_i \sim \calN(\vc{0}, \mtx{I}_d)\) for all \( i \in [q] \defeq \{1,\ldots,q\}\). Then in the $\varepsilon \rightarrow 0$ limit we have:
    \begin{equation}
        \means{\vc{u}}{\hat{\nabla} f_{\varepsilon}^q(\vc{x})^2} \approx \frac{1}{q}\left( \norm{\nabla f(\vc{x})}^2 \ones_d + \nabla f(\vc{x})^2 \right) + \nabla f(\vc{x})^2,
    \end{equation}
    where $\vc{(\cdot)}^2$ stands for the element-wise power of 2 and $\E_{\vc{u}}$ denotes the expectation over $\vc{u}_1, \dots, \vc{u}_q$.
\end{proposition}
The proof of Proposition~\ref{prop:zo-sq} appears in Appendix~\ref{app:proof-zo-sq}, and the corresponding result for the uniform perturbation estimator is given in Appendix~\ref{app:proof-zo-sq-uniform}.
% \begin{proposition}[Uniform]\label{prop:zo-sq-uniform}
% Let \( f: \reals^d \to \reals \) be an \(L\)-smooth function, and let 
% \(\hat{\nabla} f_{\varepsilon}^q(\vc{x})\) denote the \(q\)-sample ZO estimator in \eqref{eq:zo}, 
% where \(\vc{u}_i \sim \text{Uniform}(\sphere)\) for all \( i \in [q] \defeq \{1,\ldots,q\}\). Then in the $\varepsilon \rightarrow 0$ limit we have:
% \begin{align}
%     \means{\vc{u}}{\hat{\nabla} f_{\varepsilon}^q(\vc{x})^2} &\approx \frac{d\left(\|\nabla f(\vc{x})\|^2\ones_d + 2\nabla f(\vc{x})^2\right)}{q(d+2)} \nonumber \\  &+\frac{q-1}{q}\nabla f(\vc{x})^2
% \end{align}
% where $\vc{z}^2$ stands for the element-wise power of 2 and $\E_{\vc{u}}$ denotes the expectation over $\vc{u}_1, \dots, \vc{u}_q$.
% \end{proposition}
% \cli{Do we need both the Gaussian and uniform propositions in the main paper? How about we only have one in the main paper and the other in the appendix?}
Proposition~\ref{prop:zo-sq} reveals that the use of ZO gradient estimators introduces a bias in the expected squared gradient. This bias term scales as \( \frac{1}{q}\|\nabla f(\vc{x})\|^2 \) and vanishes as \( q \to \infty \). However, in practical settings where \( q \ll d \), and under the reasonable assumption that \( \|\nabla f(\vc{x})\|^2 \) grows with dimension \( d \) faster than \( 1/q \), this bias becomes dominant. Importantly, the dominant component depends on the global gradient norm rather than individual coordinates, making it effectively coordinate-agnostic. The same observation also applies to Proposition~\ref{prop:zo-sq-uniform} with uniform perturbations. Consequently, using coordinate-wise adaptive optimizers (\emph{e.g.}, Adam) is misaligned with the underlying statistics, since the estimator tracks global rather than per-coordinate information in expectation.

% \begin{table}[tb]
% \centering
% \caption{Statistics of the second moment $\vc{v}_t$ at $t=10000$ for Llama-3.2 1B fine-tuning on XSum (seed 0).} \label{tab:v_10000_stats}
% \begin{tabular}{l | l | l}
% \toprule
% Statistic & ZO-Adam & FO-Adam \\
% \midrule
% Min.         & \num{1.39e+00} & \num{6.68e-19} \\
% Max.         & \num{2.60e+00} & \num{3.08e-04} \\
% Mean         & \num{1.87e+00} & \num{1.02e-06} \\
% Stan. Dev.   & \num{1.11e-01} & \num{4.05e-06} \\
% Kurtosis     & \num{2.36e-01} & \num{3.23e+02} \\
% \bottomrule
% \end{tabular}
% \end{table}

\textbf{Empirical Evidence.} We validate our theory by examining performance and second-moment statistics from LoRA fine-tuning of Llama‑3.2 1B on XSum (with tuned step sizes). Figure~\ref{fig:llama3_sst2_curves} shows that, after proper step size tuning, ZO-SGD matches the performance of ZO-Adam on XSum, despite Adam being the dominant optimizer in FO scenarios. This suggests that coordinate-wise adaptation provides little benefit in ZO settings.

To further understand dependence on $d$, Figure~\ref{fig:vt_track} shows the second-moment estimates (\(\vc{v}_t\)) maintained by Adam and ZO-Adam in synthetic experiments, where we minimize a quadratic function across increasing dimensions. Experimental details are deferred to Appendix~\ref{app:synth-exp-setup}.
Figure~\ref{fig:vt_track} confirms that in small dimension, Adam and ZO-Adam maintain a similar second moment vector $\vc{v}_t$ across training. 
As the dimension increases, Adam shows coordinate-wise variability, but all ZO-Adam's second moment coordinates collapse to a single value: $\norm{\nabla f(\vc{x}_t)}^2/q$. This corroborates Proposition~\ref{prop:zo-sq}. 
We show similar narrowing of the ZO-Adam's $\vc{v}_t$ distribution in the case of LLM finetuning in Appendix~\ref{app:v_t_tracking_llm}.

Collectively, these findings suggest that coordinate-wise adaptive optimizers are wasteful in large-scale ZO finetuning contexts. Simpler methods like ZO-SGD not only reduce computational overhead, but also deliver superior or matching performance when tuned.

\begin{algorithm}[t]
   \caption{MEAZO Algorithm}
   \label{alg:meazo}
\begin{algorithmic}[1]
   \STATE {\bfseries Input:} initial parameters $\vc{x}_0$, objective function $f$, step size $\eta$, decay rate $\beta$, number of noise samples $q$, distribution $P$, number of iterations $T$
   \STATE {\bfseries Output:} solution $\vc{x}_T$
   \STATE {\bfseries Initialize:} $v_0$
    \FOR{$t \in \{1,...,T\}$}
        \STATE sample mini-batch $\xi$
        \STATE sample $q$ noise vectors $\vc{u}_i \sim P$
        \STATE compute $q$ projected gradient terms $\Delta f_\varepsilon\left(\vc{x}_t; \vc{u}_i,\xi\right)$
        % \STATE $\Delta f_\varepsilon\big(\vc{x}; \vc{u}_i,\xi) \leftarrow (2\varepsilon)^{-1}\left(f(\vc{x}_t + \varepsilon \vc{u}_i; \xi) - f(\vc{x}_t - \varepsilon \vc{u}_i;\xi)\right)$
        \CCOMMENT{update $v_t$}
        \STATE $g \leftarrow \sum_{i=1}^q \frac{\Delta f_\varepsilon\left(\vc{x}_t; \vc{u}_i,\xi\right)}{q} $
        \STATE $v_t \leftarrow \beta v_{t-1} + (1-\beta) g^2$
        
        \STATE $\hat{v} \leftarrow \frac{v_t}{1-\beta^{t-1}}$ \COMMENT{bias correction (if $v_0=0$)}
        \CCOMMENT{update parameters}
        \STATE $\vc{x}_t \leftarrow \vc{x}_{t-1} - \frac{\eta}{\sqrt{\hat{v}} + \zeta} \left(\sum_{i=1}^q  \frac{\Delta f_\varepsilon\left(\vc{x}_t; \vc{u}_i,\xi\right)}{2 \varepsilon q} \vc{u}_i \right)$
        % \SHORTIF{condition}{result}
   \ENDFOR
   \STATE {\textbf{return}} $\vc{x}_T$
\end{algorithmic}
\end{algorithm}

% \subsection{Do we need $\beta_1$}
% \hd{In the appendix we show that beta1 does not affect the convergence of ZO-Adam at all, can we make a more theoretically sound claim to that? I am currently launching experiments with FO Adam with disabled beta1 to check if it also does not matter in that setting.}
\section{MEAZO}\label{sec:meazo}
\textbf{Design Intuition.} MEAZO (Algorithm~\ref{alg:meazo}) is inspired by adaptive methods like Adam but tailored for the ZO setting. In standard Adam, per-parameter first and second moments enable fine-grained adaptivity. However, as demonstrated in Sec.~\ref{sec:adaptivity}, perturbations are isotropic and directions change every iteration, making per-parameter statistics ineffective and wasting memory. Instead, MEAZO maintains a \emph{single global statistic} of the squared averaged projected gradient norm:
\[
v_t \approx \text{EMA}\big(g_t^2\big), \quad g_t = \frac{1}{q}\sum_{i=1}^q \Delta f_\varepsilon(\vc{x}_t; \vc{u}_i).
\]
This scalar tracks global curvature and enables adaptive step size scaling without sacrificing ZO’s simplicity.

\textbf{Update Rule.} The MEAZO update at iteration $t$ is:
\[
\vc{x}_{t+1} = \vc{x}_t - \frac{\eta}{\sqrt{\hat{v}_t} + \zeta} \left( \frac{1}{q}\sum_{i=1}^q \Delta f_\varepsilon(\vc{x}_t; \vc{u}_i)\vc{u}_i \right),
\]
where $\hat{v}_t$ is the bias-corrected EMA of $g_t^2$, and $\zeta$ is a small constant for numerical stability. This preserves the benefits of adaptivity while keeping memory overhead comparable to ZO-SGD.

\subsection{Theoretical Guarantees}
We consider the setup in Section~\ref{sec:setup}, and state the following assumptions on the stochastic loss $f(\cdot;\xi)$ and its gradient $\nabla f(\cdot;\xi)$:
% \hd{We can sharpen the language later, for now Im just writing the current results}
\begin{assumption}[$L$-smoothness]\label{assumption:smoothness}
    The function $f(\cdot;\xi): \reals^d\to\reals$ is differentiable and $L$-smooth for all $\xi$ for some $L\geq0$.
\end{assumption}
\begin{assumption}[Bounded Variance]\label{assumption:bounded-var}
    There exists $\sigma>0$ such that, for all $\xi$,
    \begin{equation}
        \means{\xi}{\norm{\nabla f(\vc{x};\xi)-\nabla F(\vc{x})}^2} \leq \sigma^2, \quad \forall\vc{x} \in \reals^d,
    \end{equation}
\end{assumption}
\begin{assumption}[Bounded Gradients]\label{assumption:bounded-grad}
    There exists $G >0$ such that with probability 1, 
    \begin{equation}
        \norm{\nabla f(\vc{x};\xi)}^2 \leq G, \quad \forall\vc{x} \in \reals^d,
    \end{equation}
\end{assumption}
Assumptions~\ref{assumption:smoothness} and \ref{assumption:bounded-var} are standard in analyzing the convergence of non-convex first-order SGD \cite{ward2020adagrad,ghadimi2013stochastic}. Although Assumption~\ref{assumption:bounded-grad} is somewhat restrictive, since it excludes strongly convex functions, it remains common in the study of adaptive methods \cite{chen2019zo,ward2020adagrad,zaheer2018adaptive}.
Similar to prior work on ZO optimization \cite{shu2025refining}, we develop formal guarantees for MEAZO (Algorithm~\ref{alg:meazo}) for the ZO estimator \eqref{eq:zo} with uniformly sampled perturbations over the unit sphere: $\vc{u}\sim \text{Uniform}(\sphere)$ where $\sphere \defeq \{\vc{a}\in \reals^d: \norm{\vc{a}}=1\}$. This restriction is without loss of generality: empirically (Appendix~\ref{app:perturbation}) we observe that the choice of perturbation distribution does not impact model performance, and extending the analysis to Gaussian perturbations is straightforward.

Since the ZO gradient is an unbiased estimator of a smoothed version of the original objective \citep{flaxman2004online,nesterov2017random}, our convergence analysis for MEAZO proceeds in two steps. First, we establish a general result for affine variance-bounded adaptive SGD (Theorem~\ref{thm:adaptive-sgd-affine-meazo}). Next, using Theorem~\ref{thm:smoothed-proxy-opt}, we lift this guarantee to the original objective, culminating in our main result for MEAZO (Algorithm~\ref{alg:meazo}) stated in Theorem~\ref{thm:meazo}.

% \hd{add a small discussion on how our work relates to these two works \url{https://arxiv.org/pdf/2302.08783} and \url{https://arxiv.org/pdf/2202.05791}}
\textbf{Adaptive SGD with Affine Variance Bound.} Consider this generic result on affine variance-bounded SGD with an adaptive step size:
\begin{theorem}[Adaptive SGD with Affine Variance Bound]\label{thm:adaptive-sgd-affine-meazo}
Let $F:\reals^d\to\reals$ be the population loss
\[
F(\vc{x}) \;=\; \means{\xi\sim\mathcal{D}}{f(\vc{x};\xi)},
\]
where each sample loss $f(\vc{x};\xi)$ is $L$-smooth and $G$-Lipschitz (equivalently, $\norm{\nabla f(\vc{x};\xi)}\le G$ for all $\vc{x},\xi$). Assume $F$ is bounded from below by $F^*$, i.e., $F^*=\inf_{\vc{x}}F(\vc{x})>-\infty$.

Suppose we have an unbiased stochastic gradient estimator $g(\vc{x};\xi)$, \textit{i.e.,} $\means{\xi\sim\mathcal{D}}{g(\vc{x};\xi)} = \nabla F(\vc{x})$, satisfying the affine variance bound
\[
\means{\xi\sim\mathcal{D}}{\norm{g(\vc{x};\xi)-\nabla F(\vc{x})}^2}
\;\le\; \sigma_0^2+\sigma_1^2\,\norm{\nabla F(\vc{x})}^2.
\]
Consider the adaptive SGD update
\[
\vc{x}_{t+1} \;=\; \vc{x}_t \;-\; \frac{\eta}{\sqrt{v_t}+\zeta}\,g(\vc{x}_t;\xi_t),
\]
with stepsize $\eta>0$, stability constant $\zeta>0$, and second-moment tracker
\[
v_t \;=\; \beta\,v_{t-1} \;+\; (1-\beta)\,\gamma_t^2, \qquad 0<\beta<1,
\]
where $\{\gamma_t\}_{t\ge0}$ is predictable (measurable w.r.t.\ the filtration generated by $\{\vc{x}_t,\xi_t\}$), satisfies $0\le \gamma_t\le G$, and obeys the condition
\begin{equation}\label{eq:gamma-assumption}
\means{\xi_t\sim\mathcal{D}}{\gamma_t n_t \mid \vc{x}_t}
\le\sigma_0^2+(\sigma_1^2+1)\,\norm{\nabla F(\vc{x}_t)}^2,
\end{equation}
where $n_t\coloneq\norm{g(\vc{x}_t;\xi_t)}$ is the stochastic gradient norm.
Assume the parameters $(\beta,\eta,\zeta)$ are chosen so that
\begin{equation}\label{eq:meazo-norm-constants}
\max\!\left\{
\frac{G(1+\sigma_1^2)\sqrt{1-\beta}}{\zeta},\;
\frac{L\,\eta}{2\,\zeta}
\right\}
\;\le\; \frac{1}{4}.
\end{equation}
Then, after $T$ iterations,
\begin{align*}
\frac{1}{T}\sum_{t=0}^{T-1}\mean{\norm{\nabla F(\vc{x}_t)}^2}
&\le
2\alpha
\frac{F(\vc{x}_0)-F^*}{\eta T} + \alpha \frac{\sigma_0^2}{\zeta},
\end{align*}
where $\alpha= \sqrt{\beta}G+\zeta$.
\end{theorem}
Notice that the step size is normalized by an EMA of a scalar quantity $\gamma_t$ satisfying condition~\eqref{eq:gamma-assumption}. One way to understand why this condition is reasonable is to restrict $\gamma_t$ to be the norm $n_t$; under this choice, the condition follows directly from the affine variance bound. The proof of Theorem~\ref{thm:adaptive-sgd-affine-meazo} is provided in Appendix~\ref{app:adaptive-sgd-affine-meazo-proof}.

\begin{remark}
Affine-variance bounded adaptive SGD has been previously studied in \cite{faw2022power, attia2023sgd} in the context of AdaGrad-Norm \cite{ward2020adagrad}. However, their motivation is fundamentally different from ours. These works employ the affine-variance assumption because they aim to derive guarantees for AdaGrad-Norm under the least restrictive possible conditions (without assuming bounded gradients or bounded variance). In contrast, our EMA-based method requires the affine-variance assumption solely because zero-order (ZO) gradient estimators inherently produce it: even under standard variance-bounded noise, ZO gradients introduce an extra gradient-dependent term, yielding an affine-variance structure. Thus, this assumption is not a modeling choice but an unavoidable consequence of using ZO gradients.
\end{remark}

\textbf{Smoothed Proxy Optimization.}
We now state this result on smoothed proxy optimization:

\begin{theorem}[Smoothed Proxy Optimization]\label{thm:smoothed-proxy-opt}
Let $F:\reals^d \to \reals$ be an $L$-smooth function. For a centered smoothing distribution $\mathcal{P}$ with bounded support\footnote{Needed to apply Leibniz's rule.} and satisfying $\means{\vc{u}\sim\mathcal{P}}{\norm{\vc{u}}^2}=C<\infty$, define the smoothed objective:
\[
F_\varepsilon(\vc{x}) = \means{\vc{u}\sim\mathcal{P}}{F(\vc{x}+\varepsilon\vc{u})},
\]
where $\varepsilon>0$ is the smoothing parameter.  
Assume the optimal values
\[
F^* = \inf_{\vc{x}\in\reals^d} F(\vc{x}), \qquad
F_\varepsilon^* = \inf_{\vc{x}\in\reals^d} F_\varepsilon(\vc{x})
\]
are finite.  
Suppose an algorithm $\calA$ generates iterates $\{\vc{x}_t\}_{t=0}^{T-1}$ such that
\begin{equation}\label{eq:smoothed-guarantee}
\frac{1}{T}\sum_{t=0}^{T-1}\mean{\norm{\nabla F_\varepsilon(\vc{x}_t)}^2}
\leq K_0\big(F_\varepsilon(\vc{x}_0)-F_\varepsilon^*\big)+K_1
\end{equation}
for some constants $K_0,K_1\geq 0$. Then, for the original function $F$, it holds that
\begin{align}
\frac{1}{T}\sum_{t=0}^{T-1}\mean{\norm{\nabla F(\vc{x}_t)}^2}
&\leq K_0\big(F(\vc{x}_0)-F^*\big)+K_1 \nonumber\\
&+\varepsilon^2L^2C\Big(\tfrac{K_0}{2}+2\Big).
\end{align}
\end{theorem}
The proof of Theorem~\ref{thm:smoothed-proxy-opt} is provided in Appendix~\ref{app:smoothed-proxy-proof}. Despite its simplicity, Theorem~\ref{thm:smoothed-proxy-opt} is remarkably powerful: it enables us to transfer guarantees established for the smoothed proxy back to the original function, provided $F$ is $L$-smooth and the smoothing distribution is well-behaved. As expected, in the $\varepsilon \to 0$ limit, the guarantee on the smoothed proxy and the original function are identical.

\textbf{MEAZO.} We can now state the main result:
\begin{theorem}[MEAZO]\label{thm:meazo}
Under Assumptions~\ref{assumption:smoothness}, \ref{assumption:bounded-var}, and \ref{assumption:bounded-grad}, 
define
\begin{align*}
    \sigma_0^2 &= \frac{d\varepsilon^2L^2}{2q}(8+d) +\left(\frac{2d-1}{q}+1\right) \sigma^2,\\
    \sigma_1^2 &= \frac{4d-1}{q}.
\end{align*}

If the parameters $\beta,\eta,\zeta$ satisfy
\begin{equation}\label{eq:meazo-norm-constants-main}
    \max\!\left\{
        \frac{G(1+\sigma_1^2)\sqrt{1-\beta}}{\zeta},\,
        \frac{L\eta}{2\zeta}
    \right\} \le \frac{1}{4},
\end{equation}
then after $T$ iterations of Algorithm~\ref{alg:meazo} with $\vc{u}\sim\text{Uniform}(\sphere)$, we have
\begin{align*}
\frac{1}{T}\sum_{t=0}^{T-1}\mean{\norm{\nabla F(\vc{x}_t)}^2}
&\le 
2\alpha\!\left[
    \frac{F(\vc{x}_0)-F^*}{\eta T} + \frac{\sigma_0^2}{2\zeta}
\right]
\\
&+ \varepsilon^2L^2\left(\frac{\alpha}{\eta T}+2\right),
\end{align*}
where $\alpha = \sqrt{\beta}G + \zeta$.
\end{theorem}
We defer the full proof of Theorem~\ref{thm:meazo} to Appendix~\ref{app:meazo-proof}. The core idea is to verify that the conditions of Theorem~\ref{thm:adaptive-sgd-affine-meazo} hold, and then invoke Theorem~\ref{thm:smoothed-proxy-opt} to lift the guarantee to the original objective $F$.

\begin{remark}
    Following the same proof architecture, we present an alternative derivation for vanilla ZO-SGD in Appendix~\ref{app:alternative-zo-sgd-proof}, which recovers the standard SGD rate in the limit as $q \to \infty$.
\end{remark}

%%%%%%%%%%%%%%%%  Commenting out old rate to replace with informal learning rate  %%%%%%%%%%%%%%%%
% We now derive the convergence rate obtained under the optimized choice of hyperparameters.
% \begin{corollary}[MEAZO Convergence Rate]\label{cor:meazo-convergence}
% Let $d_q \coloneqq \frac{4d-1}{q} + 1$ and choose
% \[
% \beta \;=\; 1 - \frac{1}{d_q^{\,2}},
% \quad
% \eta \;=\; \frac{2G}{L},
% \quad
% \zeta \;=\; 4G,
% \quad
% \varepsilon \;=\; \frac{c_\varepsilon}{L\sqrt{T}}.
% \]
% If the condition in Theorem~\ref{thm:meazo} holds, then after $T$ iterations of Algorithm~\ref{alg:meazo} with $\vc{u}\sim\mathrm{Uniform}(\sphere)$, the iterates satisfy
% \[
% \frac{1}{T}\sum_{t=0}^{T-1}\mean{\norm{\nabla F(\vc{x}_t)}^2}
% \;\le\;
% O\!\left(\frac{1}{T} + \frac{d}{q}\left(\frac{1}{T}+\sigma^2\right)\right),
% \]
% where the $O(\cdot)$ notation hides absolute constants independent of $d,q,T,$ and $\sigma^2$ (but possibly depending on $L,G,c_\varepsilon$ and the initial gap $F(\vc{x}_0)-F^*$).
% \end{corollary}
% The proof of Corollary~\ref{cor:meazo-convergence} is deferred to Appendix~\ref{app:proof-corollary}. \hd{need to relate this rate to prior rates.}
%%%%%%%%%%%%%%%%%%%%%%%%%%%%%%%%%%%%%%%%%%%%%%%%%%%%%%%%%%%%%%%%%%%%%%%%%%%%%%%%%%%%%%%%%%%%%%%%
We now derive the convergence rate we obtained and relate it to previous works.
Its proof is deferred to Appendix~\ref{app:proof-corollary-informal}. 
\begin{corollary}[MEAZO Convergence Rate]\label{cor:meazo-convergence-rate-informal}
If the assumptions and conditions in Theorem~\ref{thm:meazo} hold, then after $T$ iterations of Algorithm~\ref{alg:meazo} with $\vc{u}\sim\mathrm{Uniform}(\sphere)$, the iterates satisfy
\begin{equation}
    O\!\left(\ \frac{F(\vc{x}_0)-F^*}{\eta\,T} + \sigma^2 + \frac{d}{q} \sigma^2 + \frac{d^2}{q} \varepsilon^2  \right),
\end{equation}
% O \left( \underbrace{\frac{F(\vc{x}_0)-F^*}{\eta\,T} + \sigma^2}_{\text{FO Adam convergence rate}}
%     + \underbrace{\frac{d}{q} \sigma^2}_{\text{Cross-stochastic gradient noise}}
%     + \underbrace{\frac{d^2}{q} \varepsilon^2}_{\text{ZO gradient estimation error}} 
where the $O(\cdot)$ notation hides constants independent of $d,q,T,$ and $\sigma^2$, and non-dominant terms in $\varepsilon^2$.
\end{corollary}
We observe that the first term correspond to the standard convergence rate of FO Adam to a small neighborhood of size $O \left( \sigma^2 \right)$ around a stationary point, proved in~\cite{zaheer2018adaptive}, which is often sufficient for large-scale machine learning problems in line with the risk-computation tradeoff emphasized by~\cite{bottou2010large}.

As defined in Assumption~\ref{assumption:bounded-var}, the constant $O \left( \sigma^2 \right)$ is inherent to the stochasticity of the gradients, which are themselves estimated with ZO gradients. Thus the second term implies a convergence of MEAZO to a small neighborhood of size $O \left( \frac{d}{q} \sigma^2 \right)$ around a stationary point. 
Unlike the previous noise term, the latter can be reduced by increasing the number of samples $q$. 
%%%%%% We can remove this next sentence if it sounds to much like a flaw of our convergence proof
% Note that it is also coherent with standard decreasing step size analyses, for instance Corollary 3.3 in~\cite{ghadimi2013stochastic}, where $\eta = O \left( \frac{1}{\sigma \sqrt{d T}} \right)$ to get the noise term converge in $O \left( \frac{\sqrt{d}}{\sqrt{T}} \right)$.

To get a convergence in $O\left( \frac{d}{T} \right)$ for the third ZO gradient estimation term, we require setting the perturbation magnitude to a small value often proportional to $\varepsilon = O \left( \frac{1}{\sqrt{d T}} \right)$, as also observed for ZO-AdaMM under non-convex assumptions~\cite{chen2019zo}.

\subsection{Grouped MEAZO}
MEAZO naturally extends to the grouped ZO estimator in \eqref{eq:block-zo}. In block ZO, each block $j$ produces its own projected gradient:
\[
g_j = \frac{1}{q}\sum_{i=1}^q \Delta f_\varepsilon\big(\vc{x}; \vc{m}_j \odot\vc{u}_i\big),
\]
allowing us to maintain a scalar $v_j$ per block. The grouped MEAZO update becomes:
\[
\sum_{j=1}^p \frac{1}{\sqrt{\hat{v}_j} + \zeta}\left( \frac{1}{q}\sum_{i=1}^q \Delta f_\varepsilon\big(\vc{x}; \vc{m}_j \odot\vc{u}_i\big)(\vc{m}_j \odot\vc{u}_i)\right),
\]
providing block-level adaptivity without incurring per-parameter memory costs.

\textbf{Efficient Implementation}. Following the decomposition strategy in \cite{zhang2024revisiting}, we treat each decoder block as a single optimization block. To exploit the sequential structure of the network, we apply perturbations beginning at the first block, ending at the last block. To avoid recomputing the base activation for block $j$ during each perturbation, we cache the unperturbed activation and reuse it for all queries within that block. This reduces redundant computation, especially for $q>1$, while preserving correctness. See Appendix~\ref{app:group-cost} for details.

\begin{figure}[tb]
  \centering
  \includegraphics[width=0.48\textwidth]{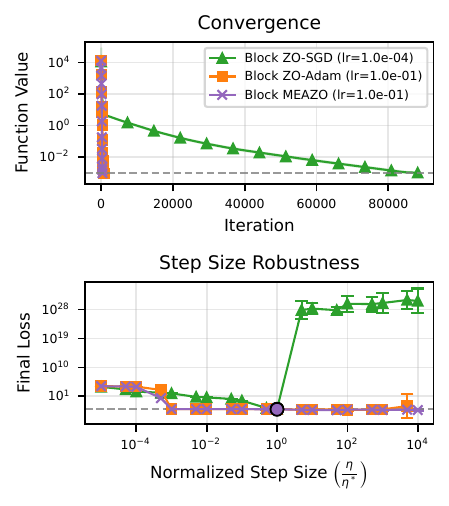}
  \caption{The performance of BCGD-ZO methods on heterogeneous block quadratics under convergence analysis (top) and step size robustness evaluation (bottom).}
\label{fig:synthetic_conv_and_robustness_hetero_bcd}
\end{figure}

\begin{table*}[t]
\centering
\caption{Evaluation metrics before and after 16-rank LoRA fine-tuning (L: Llama-3.2, Q: Qwen-3). Results are averaged over 3 seeds.}
\label{tab:eval_metrics_lora}
\resizebox{\textwidth}{!}{
\begin{tabular}{lccccccccc}
\toprule
 & \multicolumn{3}{c}{\shortstack{XSum \\\\ (ROUGE-L)}} & \multicolumn{3}{c}{\shortstack{SST-2 \\\\ (Accuracy)}} & \multicolumn{3}{c}{\shortstack{SQuAD \\\\ (F1)}} \\
Method & L1B & Q0.6B & Q8B & L1B & Q0.6B & Q8B & L1B & Q0.6B & Q8B \\
\midrule
Base & 13.62 & 15.09 & 18.60 & 49.08 & 50.69 & 93.00 & 24.63 & 31.76 & 31.50 \\
\midrule
ZO-SGD & $21.99_{\pm 0.07}$ & $20.08_{\pm 0.12}$ & $26.96_{\pm 0.28}$ & $87.96_{\pm 0.95}$ & $87.19_{\pm 1.09}$ & $93.73_{\pm 0.46}$ & $76.16_{\pm 0.40}$ & $69.30_{\pm 0.84}$ & $82.69_{\pm 0.26}$ \\
ZO-Adam & $22.36_{\pm 0.07}$ & $20.08_{\pm 0.09}$ & $27.45_{\pm 0.10}$ & $88.88_{\pm 0.43}$ & $87.27_{\pm 0.50}$ & $93.46_{\pm 0.52}$ & $76.24_{\pm 0.43}$ & $68.91_{\pm 0.83}$ & $83.16_{\pm 0.18}$ \\
RAdaZO & $22.39_{\pm 0.17}$ & $20.12_{\pm 0.05}$ & $27.34_{\pm 0.19}$ & $88.42_{\pm 0.80}$ & $87.31_{\pm 0.38}$ & $93.69_{\pm 0.61}$ & $74.54_{\pm 0.98}$ & $69.50_{\pm 0.56}$ & $82.88_{\pm 0.21}$ \\
MEAZO & $22.39_{\pm 0.19}$ & $20.08_{\pm 0.04}$ & $27.31_{\pm 0.27}$ & $88.38_{\pm 1.13}$ & $87.04_{\pm 1.50}$ & $93.39_{\pm 0.52}$ & $75.77_{\pm 0.41}$ & $68.72_{\pm 0.33}$ & $82.94_{\pm 0.87}$ \\
\bottomrule
\end{tabular}
}
\end{table*}

\begin{table}[t]
\centering
\caption{Evaluation metrics before and after 20-token prompt tuning (L: Llama-3.2, Q: Qwen-3). Results are averaged over 3 seeds.}
\label{tab:eval_metrics_prompt}
\resizebox{\columnwidth}{!}{
\begin{tabular}{lcccccc}
\toprule
 & \multicolumn{2}{c}{\shortstack{XSum \\\\ (ROUGE-L)}} & \multicolumn{2}{c}{\shortstack{SST-2 \\\\ (Accuracy)}} & \multicolumn{2}{c}{\shortstack{SQuAD \\\\ (F1)}} \\
Method & L1B & Q0.6B & L1B & Q0.6B & L1B & Q0.6B \\
\midrule
Base & 13.62 & 15.09 & 49.08 & 50.69 & 24.63 & 31.76 \\
\midrule
ZO-SGD & $17.56_{\pm 0.22}$ & $16.00_{\pm 0.19}$ & $84.90_{\pm 1.42}$ & $73.05_{\pm 6.34}$ & $49.98_{\pm 0.46}$ & $52.44_{\pm 0.83}$ \\
ZO-Adam & $17.14_{\pm 0.19}$ & $15.62_{\pm 0.33}$ & $83.60_{\pm 2.64}$ & $72.44_{\pm 5.25}$ & $48.87_{\pm 0.52}$ & $54.28_{\pm 2.58}$ \\
RAdaZO & $16.90_{\pm 0.39}$ & $13.25_{\pm 0.96}$ & $84.06_{\pm 0.80}$ & $69.19_{\pm 1.61}$ & $49.75_{\pm 0.29}$ & $43.30_{\pm 2.96}$ \\
MEAZO & $16.85_{\pm 0.33}$ & $15.72_{\pm 0.30}$ & $82.57_{\pm 2.30}$ & $67.35_{\pm 5.35}$ & $47.78_{\pm 1.15}$ & $56.46_{\pm 1.09}$ \\
\bottomrule
\end{tabular}
}
\end{table}

\section{Experimental Results}
\subsection{Synthetic Problems}\label{sec:synthetic_problems}
\textbf{Setup}. We first consider two synthetic quadratic problems $F(\vc{x}) = \frac{1}{2} \tp{\vc{x}} \mathbf{H} \vc{x}$ used in section D of~\cite{orvieto2025search}.
The Hessian $\mathbf{H} \in \mathbb{R}^{9\times 9}$ has a $3\times3$ diagonal block structure.
In the heterogeneous setting, the blocks have eigenvalues of different magnitudes, whereas in the homogeneous one, they share a similar spectrum.
% The Hessian $\mathbf{H} \in \reals^{9 \times 9}$ has a homogeneous or heterogeneous $3\times3$ diagonal block structure detailed in the Appendix~\ref{app:synth-exp-setup}. 
% In the heterogeneous case each block has eigenvalues of different magnitudes, and in the homogeneous one, each block contains a small, a medium and a large eigenvalue.
% The Hessian $\mathbf{H} \in \reals^{9 \times 9}$ has a $3\times3$ diagonal block structure detailed in the Appendix~\ref{app:synth-exp-setup}. 
% In the heterogeneous case each block has eigenvalues of different magnitudes, and in the homogeneous one, each block contains a small, a medium and a large eigenvalue.
We minimize these functions with ZO-SGD, ZO-Adam and MEAZO\footnote{With abuse of notation as we use the true gradients on the full synthetic dataset.} in a Gradient Descent (GD) or a Block-Coordinate Gradient Descent (BCGD)~\cite{tseng2009coordinate} manner.

% \subsection{Increased Benefit of Adaptivity for Heterogeneous Block settings}

\textbf{Results}. As observed in recent works~\cite{zhang2024transformers,orvieto2025search}, adaptive optimizers like Adam tend to perform better than SGD in the presence of a heterogeneous loss landscape. We confirm this behavior in the \emph{low-dimensional} ZO setting, as visualized in Fig.~\ref{fig:synthetic_conv_and_robustness_hetero_bcd} (top). Furthermore, we observe (Fig. ~\ref{fig:synthetic_conv_and_robustness_hetero_bcd} (bottom)) an increased robustness of ZO adaptive methods w.r.t. step size, something we study in further detail in section \ref{ssec:llm}. Details and the homogeneous case are deferred to Appendix~\ref{app:synth-exp-setup}.

\subsection{LLM Fine-tuning}\label{ssec:llm}
\textbf{Setup}. We fine tune two model families, Llama 3~\citep{grattafiori2024llama} and Qwen 3~\citep{yang2025qwen3}, across diverse datasets including SST-2, XSum, and SQuAD. We compare MEAZO against ZO-SGD (MeZO)\footnote{MeZO~\citep{mezo} is a memory efficient implementation of ZO-SGD.}, ZO-Adam~\citep{chen2019zo}, and $\mathcal{R}$-AdaZO~\citep{shu2025refining}, where $\mathcal{R}$-AdaZO improves ZO-Adam stability by updating the second moment using the first moment rather than the raw gradient estimate. Full experimental details are provided in Appendix~\ref{app:exp-setup}. Additional results, including comparisons on OPT~\cite{zhang2022opt} with FZOO~\cite{dang2025fzoo}, ViT-B/16 \cite{wu2020visual} experiments on the Oxford Pets dataset, and further evaluations on Ministral~14B \cite{liu2026ministral}, are reported in Appendix~\ref{app:experiments}.
% We fine-tune two model families - \mr{move OPT ref to appendix}OPT~\citep{zhang2022opt}, Llama-3~\citep{grattafiori2024llama}, and Qwen-3~\citep{yang2025qwen3} - across diverse datasets (SST-2, XSum, SQuAD). We compare MEAZO against ZO-SGD (MeZO)\footnote{MeZO~\citep{mezo} is a memory-efficient implementation of ZO-SGD.}, ZO-Adam~\citep{chen2019zo}, and two recently proposed methods: $\mathcal{R}$-AdaZO~\citep{shu2025refining} and FZOO~\citep{dang2025fzoo}. 
% $\mathcal{R}$-AdaZO modifies ZO-Adam by updating the second moment using the first moment instead of the gradient estimate, aiming for improved stability. 
% FZOO claims FO Adam-like convergence speed by combining adaptive step size control with normalized SGD and efficient binary perturbations. 
% Details of the experimental setup are provided in Appendix~\ref{app:exp-setup}, and additional results are deferred to Appendix~\ref{app:experiments} for brevity.

\begin{figure*}[t]
  \centering
  \includegraphics[width=0.9\textwidth]{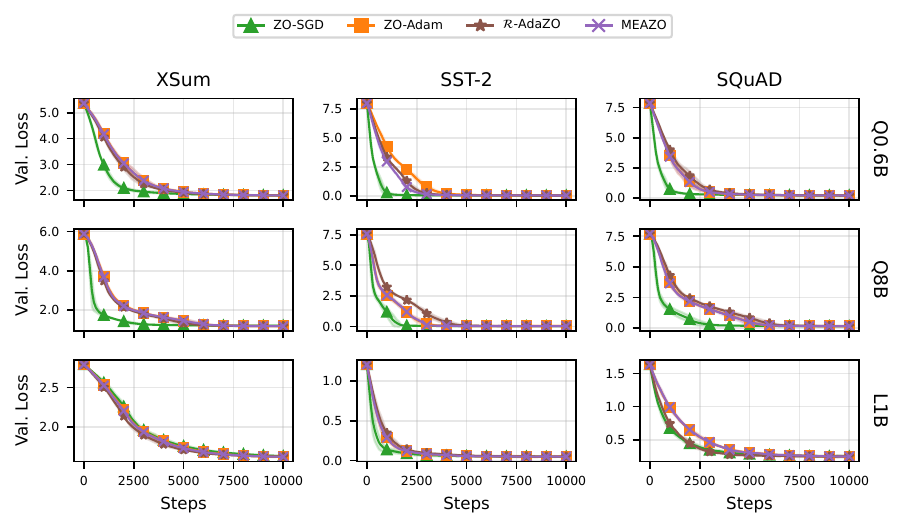}
  \caption{Performance comparison of zeroth-order methods for LoRA fine-tuning across models (L: Llama-3.2, Q: Qwen-3) and tasks using the best step size for each method (see Appendix~\ref{app:lr-sweep-results}). Curves show averages over three seeds with shaded $\pm 2$ standard deviations.}
  \label{fig:experiments}
\end{figure*}

\textbf{ZO‑SGD vs Adaptive Optimizers}.
% We evaluate two PEFT approaches, prompt tuning and LoRA, with their corresponding task‑level metrics summarized in Tables~\ref{tab:eval_metrics_prompt} and~\ref{tab:eval_metrics}, respectively. 
We evaluate two PEFT approaches, prompt tuning and LoRA (Tables~\ref{tab:eval_metrics_lora}-\ref{tab:eval_metrics_prompt}). In prompt-tuning, ZO‑Adam shows a slight advantage in most tasks, consistent with a higher gradient heterogeneity in small parameter spaces (Prop.~\ref{prop:zo-sq}). Moving to the higher‑dimensional LoRA setup: a well‑tuned ZO‑SGD matches or surpasses adaptive ZO methods, while our MEAZO performs on par with ZO‑Adam, better visualized in Fig.~\ref{fig:experiments}\footnote{Best performing runs based on lowest validation loss during training, assuming early stopping.}. These results show that per‑coordinate adaptivity provides limited benefit in high‑dimensional ZO training. We defer additional training and validation loss curves for all methods to Appendix~\ref{app:training-curves}.

\textbf{Robustness}. An often underexplored aspect of ZO optimization is robustness, particularly in the $q=1$ regime. We revisit this question for LLM fine-tuning using ZO-SGD, ZO-Adam, and MEAZO on SQuAD. For all methods, we perform a fine-grained step size sweep (Appendix~\ref{app:lr-sweep-results}), identify the best-performing step size $\eta^*$ across seeds, and plot the \textbf{best} (left) and \textbf{last} (right) validation losses against the normalized step size (Fig.~\ref{fig:robustness-non-grouped-squad}). The results show that adaptivity (ZO-Adam and MEAZO) substantially strengthens robustness by widening the stable learning-rate window and reducing performance degradation at suboptimal $\eta$, particularly for larger step sizes. While a well-tuned ZO-SGD can match or exceed adaptive methods, it remains comparatively brittle. MEAZO is consistently as robust as ZO-Adam in both settings, while reducing memory. Similar trends are observed on SST-2 and XSum and are deferred to Appendix~\ref{app:robustness}.

\begin{figure}[th]
  \centering
  \includegraphics[width=0.45\textwidth]{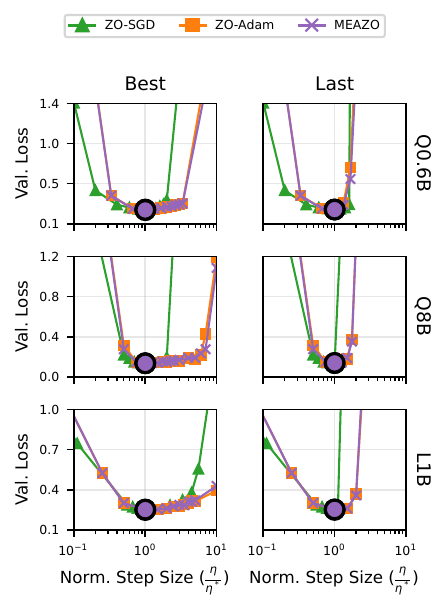}
  \caption{Robustness of ZO-SGD, ZO-Adam, and MEAZO to step size ($\eta$) initialization across models fine-tuned on SQuAD.}
\label{fig:robustness-non-grouped-squad}
\end{figure}

\textbf{Grouped ZO}. Figure~\ref{fig:meazo-grouped-vs-non-grouped} presents our primary analysis of grouped (block) zeroth-order updates from \eqref{eq:block-zo} on adaptivity and stability, using \llama{} fine-tuned on SQuAD as a representative setting. Results on additional models and datasets are deferred to Appendix~\ref{app:grouping}.
We first examine the effect of grouping on optimization behavior. Grouping is a known variance reduction technique, and consistent with this intuition, the grouped variants converge faster in terms of number of steps than their non-grouped counterparts, as shown in Fig.~\ref{fig:meazo-grouped-vs-non-grouped} (top). This improvement does not translate to better raw FLOP efficiency, since grouped updates incur higher computational cost. Nevertheless, grouping reveals an important stability effect. Despite being a variance reduction method, grouped ZO-SGD can be brittle, particularly toward the end of training, where some random seeds numerically diverge. This instability is more clearly reflected in the step-size robustness curves, which we report in Appendix~\ref{app:grouping}.

Using the same experiments and their respective optimal step sizes, we analyze the evolution of the MEAZO scalar $v_t$ in Fig.~\ref{fig:meazo-grouped-vs-non-grouped} (bottom). In the grouped setting, MEAZO adapts to group-specific statistics and exhibits a clear and stable ordering of $v_t$ across group indices throughout training. Groups corresponding to earlier layers maintain larger $v_t$ values, leading to more conservative effective steps in regions where perturbations have stronger influence on the loss, while later-layer groups converge to smaller $v_t$ values that permit relatively larger steps. This behavior is consistent with prior observations that perturbations in earlier layers tend to exert greater influence on model behavior \cite{sakr2018analytical,zhang2022all}. Importantly, the overall temporal trajectory of $v_t$ remains similar across groups, indicating that the underlying optimization dynamics are preserved even as MEAZO adapts to heterogeneous parameter sensitivities.

\begin{figure}[tbh]
  \centering
  \includegraphics[width=0.48\textwidth]{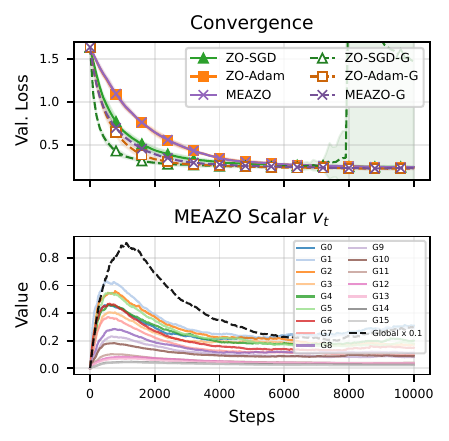}
  \caption{
  Convergence behavior of various optimizers (top) and evolution of the MEAZO scalar $v_t$ (bottom) during \llama{} fine-tuning on SQuAD under grouped and non-grouped settings. 
  For visualization clarity, the non-grouped $v_t$ is renormalized.
  }
  \label{fig:meazo-grouped-vs-non-grouped}
\end{figure}

%\section{Related Work}

\section{Discussion and Future Work}
Please find a discussion of relevant related work in Appendix \ref{App:Related_Work}.
Our results show that well-tuned ZO-SGD can match or outperform adaptive methods like ZO-Adam, due to the lack of coordinate-wise structure in ZO gradients in high dimensions, which limits the utility of adaptivity. We also show that ZO-SGD is less robust to the choice of step size, where adaptive methods have an advantage. We believe this robustness is particularly valuable for edge‑device personalization, where memory is limited, ZO fine‑tuning can be executed using built‑in inference engines, and data heterogeneity across devices leads to differing optimal step sizes. Motivated by this, we introduce MEAZO, a scalar adaptive ZO method that achieves ZO-Adam-level performance with the memory efficiency of ZO-SGD and show its robustness across step sizes. Future directions include refining our convergence results to $0$ rather than to a noise-dominated neighborhood of a stationary point. We also consider extending MEAZO to subspace ZO frameworks to leverage adaptation (sparsity, block-coordinate descent, low-rank perturbations)~\cite{park2025elucidating}, as well as full-model fine-tuning, where the low-memory property of MEAZO is most pronounced.

\section*{Impact Statement}

%Authors are \textbf{required} to include a statement of the potential broader impact of their work, including its ethical aspects and future societal consequences. This statement should be in an unnumbered section at the end of the paper (co-located with Acknowledgements -- the two may appear in either order, but both must be before References), and does not count toward the paper page limit. In many cases, where the ethical impacts and expected societal implications are those that are well established when advancing the field of Machine Learning, substantial discussion is not required, and a simple statement such as the following will suffice:

This paper presents work whose goal is to advance the field of Machine
Learning. There are many potential societal consequences of our work, only one
of which bears relevance here: ZO optimization specifically targets low-memory training environments and therefore opens an avenue for more researchers and practitioners with access to limited resources to participate in fine-tuning models. 

% The above statement can be used verbatim in such cases, but we encourage authors to think about whether there is content which does warrant further discussion, as this statement will be apparent if the paper is later flagged for ethics review.

\bibliography{refs}
\bibliographystyle{icml2026}

%%%%%%%%%%%%%%%%%%%%%%%%%%%%%%%%%%%%%%%%%%%%%%%%%%%%%%%%%%%%%%%%%%%%%%%%%%%%%%%
%%%%%%%%%%%%%%%%%%%%%%%%%%%%%%%%%%%%%%%%%%%%%%%%%%%%%%%%%%%%%%%%%%%%%%%%%%%%%%%
% APPENDIX
%%%%%%%%%%%%%%%%%%%%%%%%%%%%%%%%%%%%%%%%%%%%%%%%%%%%%%%%%%%%%%%%%%%%%%%%%%%%%%%
%%%%%%%%%%%%%%%%%%%%%%%%%%%%%%%%%%%%%%%%%%%%%%%%%%%%%%%%%%%%%%%%%%%%%%%%%%%%%%%
\newpage
\appendix
\onecolumn

\clearpage
\section{Contributions}\label{app:contributions}

\paragraph{Hassan Dbouk$^*$}
Conceptualization; theory (formalism and mathematical proofs); devised the MEAZO algorithm; software; experimental design, execution, and management (LLM non‑grouped and grouped settings, ViT); visualization (figures and tables); article writing; research discussions.
\paragraph{Nidham Gazagnadou$^*$}
Conceptualization; theory (formalism and mathematical proofs); bibliographic research on zeroth‑order optimization; software; experimental design, execution, and management (LLM prompt‑tuning, Qwen3‑8B, synthetic experiments); article writing; research discussions.
\paragraph{Matthias Reisser$^*$}
Conceptualization; core software design (optimizers, grouping); experimental design, execution, and management (LLM non‑grouped and grouped settings, Ministral); article writing; research discussions.
\paragraph{Christos Louizos}
Conceptualization; supervision; theoretical guidance (including connections to Adam convergence); feedback on experimental design and manuscript; article writing; research discussions.

\section{Related Work}
\label{App:Related_Work}
\subsection{Zeroth-Order Gradient Estimation}

Theoretical foundations of ZO optimization were established in early works such as~\citep{ghadimi2013stochastic} and~\citep{nesterov2017random}, which analyzed convergence properties of ZO-SGD using Gaussian perturbations, or the convergence in the online setting with perturbations sampled uniformly from the unit sphere~\citep{flaxman2004online}. 
% Other recent methods have focused on refining the gradient estimate with historical data~\citep{meier2019improving,cheng2021convergence}, or scale-up the usage of ZO methods using sparsity~\citep{wang2018stochastic,chen2023deepzero}.
Recently, MeZO~\citep{mezo} introduced a memory-efficient variant of ZO-SGD by storing scalar seeds instead of sampled perturbations, significantly reducing memory usage. 

Comprehensive benchmarking of ZO methods has been conducted by \citep{zhang2024revisiting} to evaluate their empirical performance across tasks, highlighting trade-offs between estimator variance, convergence speed, and memory efficiency.

\subsection{Adaptive Zeroth-Order Methods}

To enhance convergence and robustness, adaptive ZO algorithms have incorporated techniques inspired by first-order optimization. ZO-AdaMM~\citep{chen2019zo} and ZEMA~\citep{nazari2020adaptive} introduced adaptive step sizes based on first and second moment estimates of the gradient approximation. However, these methods did not demonstrate theoretical convergence improvements over ZO-SGD in {convex and non-convex constrained settings.} %\ngd{to double check}

Recent advances have focused on refining the use of momentum and variance reduction. ZO-AdaMU~\citep{jiang2024zo} introduced adaptive momentum for perturbations, smoothing both gradient estimates and parameter updates. 
FZOO~\citep{dang2025fzoo} jointly improved ZO gradient estimation, by combining Rademacher perturbations, and  step size adaptation, using the standard deviation of the batch losses.
$\mathcal{R}$-AdaZO~\citep{shu2025refining} leveraged variance-reduced gradient information embedded in the momentum to refine second moment estimates, leading to improved convergence on synthetic data.

\section{Synthetic Quadratic Experimental Setup}\label{app:synth-exp-setup}

\subsection{Second moment monitoring with increasing dimension}

\paragraph{Objective functions.} The experiment uses a block-structured quadratic function $f(x) = 0.5\tp{\vc{x}}\mtx{H} \vc{x}$, where $\mtx{H}$ is a block-diagonal Hessian matrix. For a dimension $d$, the Hessian is constructed from $\sqrt{d}$ blocks of size $\sqrt{d}$ each. We consider the heterogeneous configuration: each block contains consecutive eigenvalues centered around exponentially increasing values (logarithmically spaced from 1 to 1000). For more details, we refer to Appendix D of~\cite{orvieto2025search} for the case of $d=9$ that we use below and extend here for larger dimensions. This creates blocks with different curvature, which helps to visualize second moment coordinate-wise variability of Adam. This structure mimics real optimization landscapes where different parameter groups have different sensitivities.

\paragraph{Experiment.} We track the evolution of the second moment $\vc{v}_t$ (vector or scalar value for MEAZO) across 3 optimizers: FO-Adam, ZO-Adam and MEAZO. We run the experiment for 5 different problem dimensions: $9, 25, 49, 100$ and $1024$. All optimizers use the same learning rate $\expnumber{1}{-4}$ until convergence to the same threshold $\expnumber{1}{-3}$ with 10 noise samples for the ZO methods.
For ZO-Adam, we also monitor $\norm{\nabla f (\vc{x}_t)}^2/q$ at iteration $t$.

\subsection{Homogeneous vs Heterogeneous low-dimensional setting}
We exactly reproduce the problem of appendix D of~\cite{orvieto2025search}.
% In the heterogeneous case each block has eigenvalues of different magnitudes, and in the homogeneous one, each block contains a small, a medium and a large eigenvalue.
The only difference being that we do not introduce random batch samples stochasticity by sampling rows of the design matrix.
Thus, the only source of randomness lies in the zeroth-order gradients approximation.
To find the best step size for each run, we perform step size tuning over the grid $\{\expnumber{1}{-6}, \expnumber{5}{-6}\ldots, \expnumber{1}{-1}, \expnumber{5}{-1}\}$, with $10$ different seeds.
For each of the displayed run, we display the average run over the $10$ seeds along with a plus/minus standard deviation envelope (in log space).
Curves are smoothed over a window of $10$ step for better display of ZO-Adam and MEAZO.

On the one hand, in the homogeneous case in Figure~\ref{fig:synthetic_conv_homo_vs_hetero_2x2} (top), one can notice the similar performance of all three optimizers.
On the other hand, Figure~\ref{fig:synthetic_conv_homo_vs_hetero_2x2} (bottom) highlights the acceleration of the optimization with MEAZO, resp. ZO-Adam, when rescaling the step size per block, resp. per parameter within a block, in the presence of landscape heterogeneity.

% \begin{figure}[t]
%     \centering
%     % --- Row 1: Homogeneous ---
%     \begin{subfigure}{0.49\columnwidth}
%         \centering
%         \includegraphics[width=\linewidth]{img/synthetic_quadratic_exp_20260210/20260210-gd_zo_optimizers_comparison_homo.png}
%         \caption{ZO-GD (Homogeneous)}
%         \label{fig:synthetic_conv_gd_homo}
%     \end{subfigure}\hfill
%     \begin{subfigure}{0.49\columnwidth}
%         \centering
%         \includegraphics[width=\linewidth]{img/synthetic_quadratic_exp_20260210/20260210-bcd_zo_optimizers_comparison_homo.png}
%         \caption{ZO-BCGD (Homogeneous)}
%         \label{fig:synthetic_conv_bcgd_homo}
%     \end{subfigure}
%     \vspace{0.5em}
%     % --- Row 2: Heterogeneous ---
%     \begin{subfigure}{0.49\columnwidth}
%         \centering
%         \includegraphics[width=\linewidth]{img/synthetic_quadratic_exp_20260210/20260210-gd_zo_optimizers_comparison_hetero.png}
%         \caption{ZO-GD (Heterogeneous)}
%         \label{fig:synthetic_conv_gd_hetero}
%     \end{subfigure}\hfill
%     \begin{subfigure}{0.49\columnwidth}
%         \centering
%         \includegraphics[width=\linewidth]{img/synthetic_quadratic_exp_20260210/20260210-bcd_zo_optimizers_comparison_hetero.png}
%         \caption{ZO-BCGD (Heterogeneous)}
%         \label{fig:synthetic_conv_bcgd_hetero}
%     \end{subfigure}
%     \caption{Convergence of ZO methods on block homogeneous (top) vs heterogeneous (bottom) quadratics.}
%     \label{fig:synthetic_conv_homo_vs_hetero_2x2}
% \end{figure}
\begin{figure}[t]
    \centering
        \includegraphics[width=\linewidth]{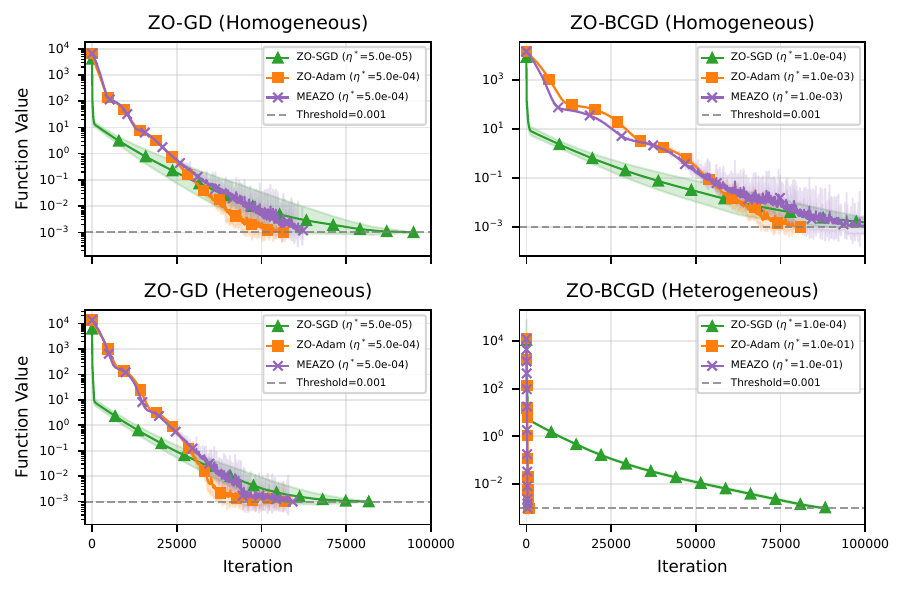}
\caption{Convergence of ZO methods on block homogeneous (top) vs heterogeneous (bottom) quadratics.}
\label{fig:synthetic_conv_homo_vs_hetero_2x2}
\end{figure}
In Figure~\ref{fig:synthetic_robust_homo_vs_hetero_2x2}, we show the better robustness around the optimal step size of ZO adaptive methods, \textit{i.e.,} ZO-Adam and MEAZO, compared to ZO-SGD.
In the heterogeneous setting, this robustness of ZO adaptive methods is even amplified by applying block-coordinate updates, as shown by way flatter curves for ZO-Adam and MEAZO in Figure~\ref{fig:synthetic_robust_homo_vs_hetero_2x2} (bottom right) compared to Figure~\ref{fig:synthetic_robust_homo_vs_hetero_2x2} (bottom left), than for ZO-SGD.

% \begin{figure}[t]
%     \centering
%     % --- Row 1: Homogeneous ---
%     \begin{subfigure}{0.49\columnwidth}
%         \centering
%         \includegraphics[width=\linewidth]{img/synthetic_quadratic_exp_20260210/20260210-gd_lr_robustness_n1_homo.png}
%         \caption{ZO-GD (Homogeneous)}
%     \end{subfigure}\hfill
%     \begin{subfigure}{0.49\columnwidth}
%         \centering
%         \includegraphics[width=\linewidth]{img/synthetic_quadratic_exp_20260210/20260210-bcd_lr_robustness_n1_homo.png}
%         \caption{ZO-BCGD (Homogeneous)}
%     \end{subfigure}
%     \vspace{0.5em}
%     % --- Row 2: Heterogeneous ---
%     \begin{subfigure}{0.49\columnwidth}
%         \centering
%         \includegraphics[width=\linewidth]{img/synthetic_quadratic_exp_20260210/20260210-gd_lr_robustness_n1_hetero.png}
%         \caption{ZO-GD (Heterogeneous)}
%         \label{fig:synthetic_robust_gd_hetero}
%     \end{subfigure}\hfill
%     \begin{subfigure}{0.49\columnwidth}
%         \centering
%         \includegraphics[width=\linewidth]{img/synthetic_quadratic_exp_20260210/20260210-bcd_lr_robustness_n1_hetero.png}
%         \caption{ZO-BCGD (Heterogeneous)}
%         \label{fig:synthetic_robust_bcgd_hetero}
%     \end{subfigure}
%     \caption{Step size robustness of ZO methods on block homogeneous (top) vs heterogeneous (bottom) quadratics.}
%     \label{fig:synthetic_robust_homo_vs_hetero_2x2}
% \end{figure}
\begin{figure}[t]
    \centering
        \includegraphics[width=\linewidth]{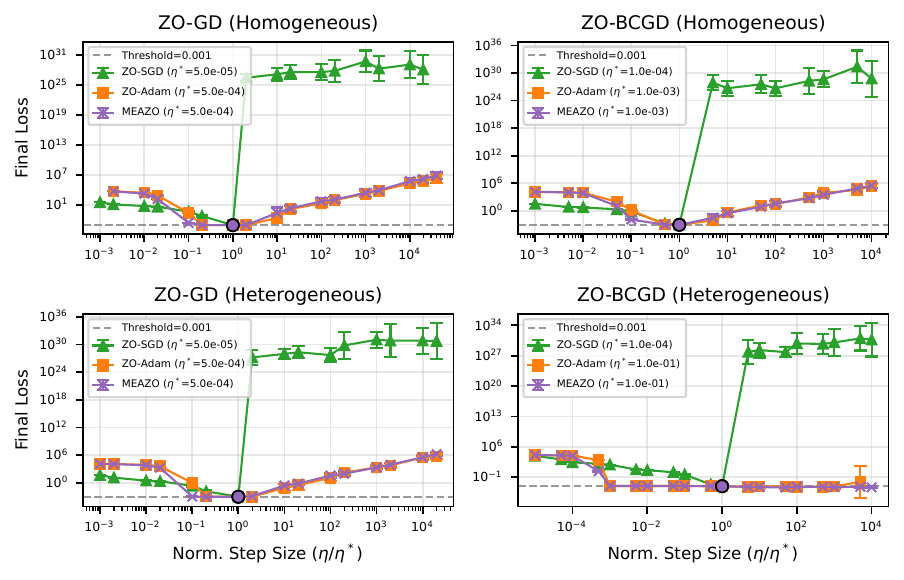}
\caption{Step size robustness of ZO methods on block homogeneous (top) vs heterogeneous (bottom) quadratics.}
\label{fig:synthetic_robust_homo_vs_hetero_2x2}
\end{figure}
\section{LLM Fine-tuning Experimental Setup}\label{app:exp-setup}

\subsection{Training Details} Unless stated otherwise, all models are trained for $T=\num{10000}$ steps with a batch size of \num{8} on single NVIDIA A100/H100 GPUs equipped with \num{80}GB of memory. We employ a constant step size schedule throughout training, without applying weight decay or gradient norm clipping. For adaptive optimization methods, we use the default hyperparameters: $\beta_1 = 0.9$ and $\beta_2 = 0.999$. For zeroth-order (ZO) optimization methods, we adopt a fixed perturbation scale of $\varepsilon = \expnumber{1}{-6}$ and use a single noise sample per mini-batch ($q = \num{1}$).

To ensure a fair comparison across optimization methods, we conduct extensive step size sweeps to identify the optimal initial step size for each optimizer, detailed in Appendix~\ref{app:lr-sweep-results}. %Specifically, we begin with a coarse sweep, determine an initial optimal step size, then perform a finer-grained sweep around this value using a log-linear scale, evaluating \num{10} values below and \num{10} values above the initial optimum. For robustness-plots, we additionally manually select learning rates that "complete the picture" and extend the range towards the lower and higher regimes.

\subsection{Classification Training Loss} \citet{zhang2024revisiting} formulate classification for LLM fine-tuning using a multiclass contrastive objective: for each input, the model is paired with \emph{every} verbalized label, computes an average log-probability score over the label tokens, and these scores are used as logits in a softmax cross-entropy loss. This explicitly contrasts the true label against all incorrect ones but requires evaluating $C$ candidate sequences during training.

In our setup, we instead treat classification (SST-2) as straightforward causal language modeling (CLM). During training, we provide only the ground-truth continuation (the correct verbalized label) and maximize its next-token likelihood (no negative labels and no classification loss). However, at evaluation time, we adopt the same procedure: for each class, we compute the same average log-probability score for its verbalized label and select the label with the highest score. Despite the simpler training objective, we find minimal degradation in final model performance.

\subsection{Models}\label{app:models}
We download all pre-trained models from HuggingFace and load them in full precision (FP32) except for Ministral 14B loaded in BF16.
For language tasks we consider Llama-3.2 1B~\cite{grattafiori2024llama}, Qwen-3 0.6B and 8B~\cite{yang2025qwen3} and Ministral-3 14B Instruct~\cite{liu2026ministral}. For the vision task, we finetune ViT-B/16~\cite{wu2020visual}.
The training configurations detailed below are summarized in Table~\ref{tab:exp_detail_training_config}.

\begin{table}[t]
    \centering
    \caption{Summary of training configurations. For all experiments, LoRA uses rank $r=16$. Prompt tuning uses 20 learned input tokens.}
    \label{tab:exp_detail_training_config}
    \renewcommand{\arraystretch}{1.1}
    \begin{tabular}{>{\raggedright\arraybackslash}p{3.4cm}rrrrr}
        \toprule
        Method & Llama-3.2 1B & Qwen-3 0.6B & Qwen-3 8B & Ministral-3 14B & ViT-B/16 \\
        \midrule
        \multicolumn{6}{l}{\emph{Base model}} \\
        \#Total parameters
        & 1 235 814 400 & 596 049 920 & 8 190 735 360 & 13 945 031 680 & 85 827 109 \\
        \midrule
        \multicolumn{6}{l}{\emph{LoRA}} \\
        \#Trainable parameters
        & 1 703 936 & 2 293 760 & 7 667 712 & 11 403 264 & 618 277 \\
        \%Trainable parameters
        & 0.1377\% & 0.3834\% & 0.0935\% & 0.0817\% & 0.7152\% \\
        \midrule
        \multicolumn{6}{l}{\emph{Prompt tuning}} \\
        \#Trainable parameters
        & 40 960 & 20 480 & -- & -- & -- \\
        \%Trainable parameters
        & 0.0033\% & 0.0034\% & -- & -- & -- \\
        \bottomrule
    \end{tabular}
\end{table}

% rank = 16
% 1,235,855,360
% rank = 32
% trainable params: 3,407,872 || all params: 1,239,222,272 || trainable%: 0.2750
% 
% rank = 64
% trainable params: 6,815,744 || all params: 1,242,630,144 || trainable%: 0.5485
% Llama - prompt tuning 100
% trainable params: 204,800 || all params: 1,236,019,200 || trainable%: 0.0166

% Qwen3 0.6B - prompt tuning 20
% trainable params: 20,480 || all params: 596,070,400 || trainable%: 0.0034
% Qwen3 0.6B - prompt tuning 100
% trainable params: 102,400 || all params: 596,152,320 || trainable%: 0.0172

% Qwen3 8B - lora r=16
% trainable params: 7,667,712 || all params: 8,198,403,072 || trainable%: 0.0935

% Ministral 3 14B - lora r=16
% trainable params: 11,403,264 || all params: 13,956,434,944 || trainable%: 0.0817

% ViT-Base - lora r=16
% trainable params: 618,277 || all params: 86,445,386 || trainable%: 0.7152

\paragraph{LoRA fine-tuning.}
For LoRA fine-tuning \citep{hu2022lora}, we use a rank of $r = 16$, set the scaling factor to $\alpha = 16$, and disable dropout. LoRA layers are attached to all query (Q) and value (V) projection matrices within each transformer decoder block. 
% An exception is made for Phi-3, where LoRA is applied to the combined QKV projection layer.
Our implementation is based on Hugging Face PEFT's library\footnote{\url{https://huggingface.co/docs/peft}}.

\paragraph{Prompt tuning.}
For prompt tuning~\cite{lester2021power}, we append 20 trainable tokens to the input in embedding space and train for $T=20 000$ steps.
We also rely on Hugging Face PEFT's library.
We initialize all our prompt's embeddings with the following task-specific initialization prompt:
\begin{itemize}
    \item SST2: \textit{"Sentiment task: classify text as positive or negative. Output only positive or negative."}
    \item SQuAD: \textit{"Extract answer from passage only using exact text span answering the question precisely and concisely."}
    \item XSum: \textit{"Summarize the article into one concise sentence capturing the main event, key facts, and essential context clearly."}
\end{itemize}

\subsection{Dataset Pre-processing}
For all datasets, we utilize their publicly available implementations from HuggingFace. For each model, we format every data point using its corresponding tokenizer's chat template.

\textbf{SST-2}. The SST-2 dataset \citep{sst2} is pre-divided into training, validation, and test splits. However, since the original test set lacks labels, we reassign the original validation set as our test set. To construct a new validation set, we split the original training data (comprising \num{67349} samples) into training and validation subsets using a \num{95}/\num{5} ratio. The chat template employed for training is illustrated in Box~\ref{box:sst2}. For evaluation, we compute the accuracy on the full test set.

\begin{myfloatbox}[box:sst2]{SST-2 Chat Template}
  \begin{tcolorbox}[system]
  You are a sentiment analysis assistant. You will be given a user provided phrase or sentence, and your job will be to output terrible if it has negative sentiment or great if it has positive sentiment.
  \end{tcolorbox}

  \begin{tcolorbox}[user]
  Sentence: \verb|<$sentence>|
  \end{tcolorbox}

  \begin{tcolorbox}[assistant]
  Sentiment: \verb|<$sentiment>|
  \end{tcolorbox}

\end{myfloatbox}

\textbf{XSum}. The XSum \citep{xsum} dataset retains its original training, validation, and test splits. Prior to training, we filter out samples that exceed the predefined context length (e.g., \num{2048} tokens) after tokenization. This filtering process is tokenizer-dependent; using the Llama 3.2 tokenizer, we exclude 1,202 samples across all splits. We then select the first 50,000 training samples and the first 1,000 validation samples.The chat template used for training is shown in Box~\ref{box:xsum}. For evaluation, we compute the ROUGE-L score on the full test set.

\begin{myfloatbox}[box:xsum]{XSum Chat Template}
  \begin{tcolorbox}[system]
  You are a writing assistant that helps with summarizing text. You write concise, one-sentence summaries based on a user-provided text or article.
  \end{tcolorbox}

  \begin{tcolorbox}[user]
  Article: \verb|<$article>|
  \end{tcolorbox}

  \begin{tcolorbox}[assistant]
  Summary: \verb|<$summary>|
  \end{tcolorbox}

\end{myfloatbox}

\textbf{SQuAD}. The SQuAD dataset \citep{squad} includes only training and validation splits. Following the approach used for SST-2, we reassign the validation split as the test set. The original training set contains \num{87599} samples, which we divide into training and validation subsets using a \num{95}/\num{5} split. From these, we select the first \num{50000} training samples and the first \num{1000} validation samples. The chat template used for training is shown in Box~\ref{box:squad}. For evaluation, we compute the F-1 score on the full test set.

\begin{myfloatbox}[box:squad]{SQuAD Chat Template}
  \begin{tcolorbox}[system]
    You are a question answering assistant. You will be given a user-provided context (paragraph or sentence), followed by a question. Your job will be to answer the question based on the paragraph or sentence.
  \end{tcolorbox}

  \begin{tcolorbox}[user]
  Title: \verb|<$title>|
  
  Context: \verb|<$context>|
  
  Question: \verb|<$question>|
  \end{tcolorbox}

  \begin{tcolorbox}[assistant]
  Answer: \verb|<$answer>|
  \end{tcolorbox}

\end{myfloatbox}

\textbf{Oxford-Pets}. To extend our experiments to vision tasks, we use the Oxford‑IIIT Pet dataset~\cite{parkhi2012cats}, which consists of 37 cat and dog breeds with roughly 200 images per class, and evaluate performance on the 37‑way fine‑grained classification task.
\clearpage
\section{Additional Experiments} \label{app:experiments}

\subsection{Coordinate-wise Variability of the Second Moment: FO- vs ZO-Adam}
\label{app:v_t_tracking_llm}

We examine the second-moment estimates (\(\vc{v}_t\)) maintained by Adam at the final LoRA finetuning step (\(t = 10000\)) of a Llama-3.2 1B model on the XSum dataset. Table~\ref{tab:v_10000_stats} reports the minimum, maximum, mean, standard deviation, and kurtosis of \(\vc{v}_t\) across all parameters ($d=\text{1'703'936}$). FO-Adam exhibits large variability and high kurtosis, reflecting meaningful per-coordinate differences that justify adaptive scaling. In contrast, ZO-Adam shows nearly uniform values with low variance and kurtosis, confirming that its adaptive mechanism is largely redundant when the gradient signal is isotropic.

\begin{table}[tbh]
\centering
\caption{Statistics of the second moment $\vc{v}_t$ at $t=10000$ for Llama-3.2 1B fine-tuning on XSum (seed 0).} \label{tab:v_10000_stats}
\begin{tabular}{l | l | l}
\toprule
Statistic & ZO-Adam & FO-Adam \\
\midrule
Min.         & \num{1.39e+00} & \num{6.68e-19} \\
Max.         & \num{2.60e+00} & \num{3.08e-04} \\
Mean         & \num{1.87e+00} & \num{1.02e-06} \\
Stan. Dev.   & \num{1.11e-01} & \num{4.05e-06} \\
Kurtosis     & \num{2.36e-01} & \num{3.23e+02} \\
\bottomrule
\end{tabular}
\end{table}

\subsection{Step Size Sweeps}\label{app:lr-sweep-results}

To properly ensure well‑tuned methods, we perform extensive step‑size sweeps conducted over two stages.

\textbf{Coarse grid.} We first sweep the step size over a coarse grid: $\eta \in \{\expnumber{1}{-6}, \expnumber{5}{-6}, \ldots, \expnumber{1}{-1}\}$,
running each configuration across three random seeds. We then select the step size that achieves the lowest average validation loss.

\textbf{Fine grid.} Let $\eta^\star$ be the best step size from the coarse sweep and let $\eta^-,\eta^+$ be its adjacent coarse neighbors (extending the coarse grid by one virtual neighbor at the boundary if needed). We then enumerate candidate step sizes with integer mantissas within the bracket(s) $(\eta^-,\eta^\star)$ and $(\eta^\star,\eta^+)$, i.e., values of the form $\eta=m\times 10^k$ with $m\in\{1,\dots,9\}$ whose magnitude lies within the bracket. When a bracket crosses a decade boundary, we enumerate the remaining integer-mantissa values in the lower decade and the leading values in the upper decade.

Tables~\ref{tab:optimal-lr-lora}-\ref{tab:optimal-lr-pt} summarizes the best‑performing step sizes for LoRA and prompt tuning, respectively.

\begin{table}[thb]
\centering
\caption{Optimal step sizes for 16-rank LoRA fine-tuning (L: Llama-3.2, Q: Qwen-3).}
\label{tab:optimal-lr-lora}
\begin{tabular}{lccccccccc}
\toprule
 & \multicolumn{3}{c}{XSum} & \multicolumn{3}{c}{SST-2} & \multicolumn{3}{c}{SQuAD} \\
Method & L1B & Q0.6B & Q8B & L1B & Q0.6B & Q8B & L1B & Q0.6B & Q8B \\
\midrule
ZO-SGD & $\expnumber{1}{-4}$ & $\expnumber{6}{-5}$ & $\expnumber{2}{-4}$ & $\expnumber{1}{-4}$ & $\expnumber{4}{-5}$ & $\expnumber{1}{-4}$ & $\expnumber{9}{-5}$ & $\expnumber{5}{-5}$ & $\expnumber{1}{-4}$ \\
ZO-Adam & $\expnumber{2}{-4}$ & $\expnumber{2}{-4}$ & $\expnumber{4}{-4}$ & $\expnumber{3}{-4}$ & $\expnumber{2}{-4}$ & $\expnumber{6}{-4}$ & $\expnumber{2}{-4}$ & $\expnumber{3}{-4}$ & $\expnumber{4}{-4}$ \\
RAdaZO & $\expnumber{5}{-5}$ & $\expnumber{5}{-5}$ & $\expnumber{1}{-4}$ & $\expnumber{6}{-5}$ & $\expnumber{6}{-5}$ & $\expnumber{1}{-4}$ & $\expnumber{7}{-5}$ & $\expnumber{6}{-5}$ & $\expnumber{8}{-5}$ \\
MEAZO & $\expnumber{2}{-4}$ & $\expnumber{2}{-4}$ & $\expnumber{4}{-4}$ & $\expnumber{3}{-4}$ & $\expnumber{3}{-4}$ & $\expnumber{6}{-4}$ & $\expnumber{2}{-4}$ & $\expnumber{3}{-4}$ & $\expnumber{4}{-4}$ \\
\bottomrule
\end{tabular}
\end{table}

\begin{table}[thb]
\centering
\caption{Optimal step sizes for 20-token prompt tuning (L: Llama-3.2, Q: Qwen-3).}
\label{tab:optimal-lr-pt}
\begin{tabular}{lcccccc}
\toprule
 & \multicolumn{2}{c}{XSum} & \multicolumn{2}{c}{SST-2} & \multicolumn{2}{c}{SQuAD} \\
Method & L1B & Q0.6B & L1B & Q0.6B & L1B & Q0.6B \\
\midrule
ZO-SGD & $\expnumber{3}{-4}$ & $\expnumber{9}{-4}$ & $\expnumber{5}{-4}$ & $\expnumber{2}{-4}$ & $\expnumber{3}{-4}$ & $\expnumber{1}{-3}$ \\
ZO-Adam & $\expnumber{6}{-4}$ & $\expnumber{1}{-4}$ & $\expnumber{9}{-5}$ & $\expnumber{1}{-2}$ & $\expnumber{6}{-4}$ & $\expnumber{7}{-5}$ \\
RAdaZO & $\expnumber{6}{-5}$ & $\expnumber{1}{-3}$ & $\expnumber{7}{-4}$ & $\expnumber{2}{-3}$ & $\expnumber{2}{-4}$ & $\expnumber{2}{-3}$ \\
MEAZO & $\expnumber{2}{-4}$ & $\expnumber{9}{-5}$ & $\expnumber{9}{-5}$ & $\expnumber{9}{-3}$ & $\expnumber{9}{-4}$ & $\expnumber{7}{-5}$ \\
\bottomrule
\end{tabular}
\end{table}

\begin{remark}
    As predicted by our theory, ZO-Adam and MEAZO behave extremely similarly in high dimension and require the same optimal learning rate in almost all cases. We also observe a similar behavior than in FO finetuning: ZO-Adam (and also MEAZO) can use larger learning rates than ZO-SGD as observed in~\citep{pan2023toward}.
\end{remark}
\clearpage
\subsection{Training Curves}\label{app:training-curves}

Due to space constraints in the main paper, we present additional loss curves in this appendix. Specifically, we show training loss curves for LoRA fine‑tuning (Figure~\ref{fig:lora-train-curves}), as well as both training and validation loss curves for prompt tuning (Figures~\ref{fig:pt-train-curves} and~\ref{fig:pt-val-curves}).

Figure~\ref{fig:lora-train-curves} reports the training loss dynamics of different zeroth‑order methods for LoRA across models and tasks, using the best step size selected for each method. Overall, the trends are consistent with the results reported in the main paper and provide additional insight into the optimization behavior over the course of training.

Figures~\ref{fig:pt-train-curves} and~\ref{fig:pt-val-curves} show the corresponding training and validation loss curves for prompt tuning. We observe that, particularly on \qwen{}, prompt‑tuning exhibits noticeably higher variance and less stable optimization under zeroth‑order methods compared to LoRA. This instability is reflected in both the training dynamics and validation trends, and is consistent across multiple random seeds.

\begin{figure}[thb]
  \centering
  \includegraphics[width=0.9\textwidth]{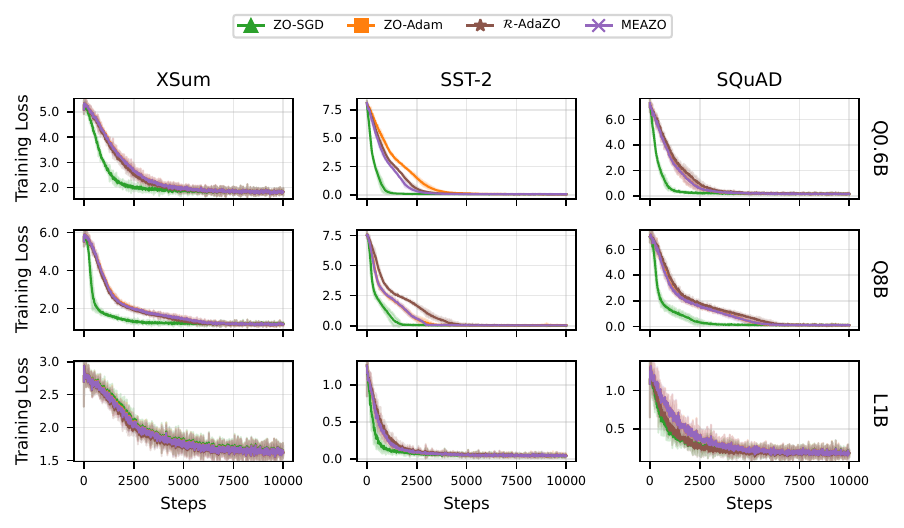}
  \caption{Training loss curves for zeroth‑order methods applied to LoRA fine‑tuning across models (L: Llama-3.2, Q: Qwen-3) and tasks, using the best step size for each method. Curves show the mean over three seeds, with shaded regions denoting $\pm 2$ standard deviations.}
  \label{fig:lora-train-curves}
\end{figure}

\begin{figure}[bht]
  \centering
  \includegraphics[width=0.9\textwidth]{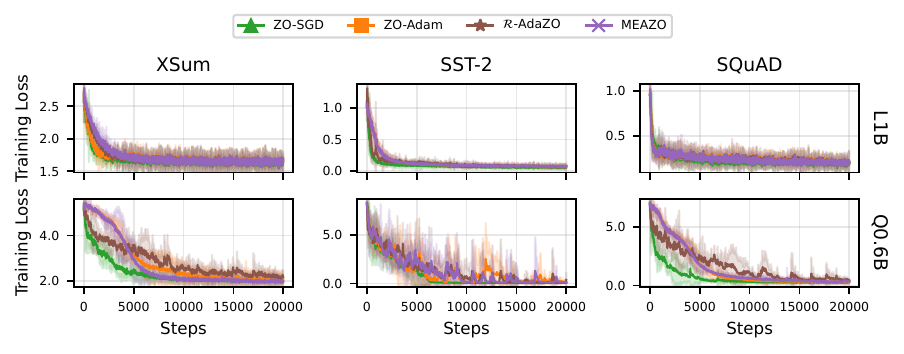}
  \caption{Training loss curves for zeroth‑order methods applied to to prompt‑tuning across models and tasks, using the best step size for each method. Curves show the mean over three seeds, with shaded regions denoting $\pm 2$ standard deviations.}
  \label{fig:pt-train-curves}
\end{figure}

\begin{figure}[tbh]
  \centering
  \includegraphics[width=0.9\textwidth]{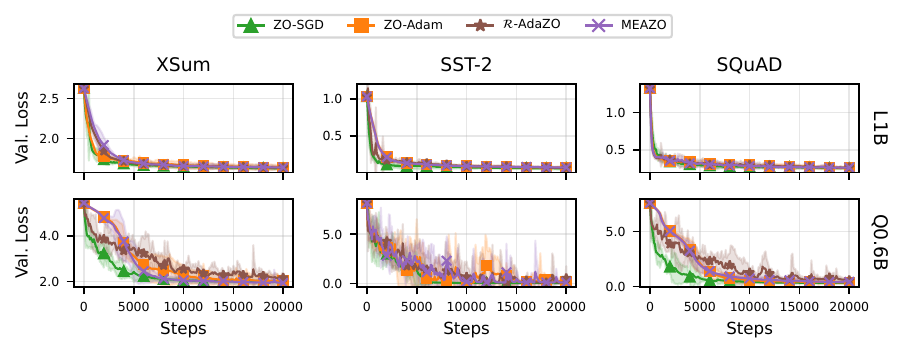}
  \caption{Validation loss curves for zeroth‑order methods applied to prompt‑tuning across models and tasks, using the best step size for each method. Curves show the mean over three seeds, with shaded regions denoting $\pm 2$ standard deviations.}
  \label{fig:pt-val-curves}
\end{figure}

\clearpage
\subsection{Robustness Curves} \label{app:robustness}
We complete the step size robustness analysis by reporting results on the remaining datasets, SST-2 and XSum, in Figure~\ref{fig:robustness-sst2-xsum} (the main paper reports SQuAD). Consistent with the SQuAD findings, adaptivity (ZO-Adam and MEAZO) substantially strengthens robustness by widening the stable learning-rate window and reducing performance degradation at suboptimal $\eta$, particularly for larger step sizes. While a well-tuned ZO-SGD can match or even exceed adaptive methods at its optimal step size, it remains comparatively brittle. MEAZO is consistently as robust as ZO-Adam across both datasets.

\begin{figure}[tbh]
  \centering
    \includegraphics[width=0.9\textwidth]{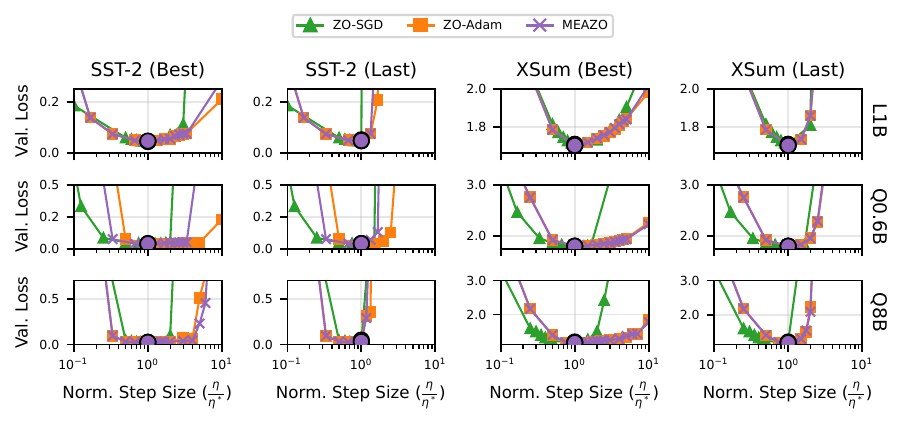}
  \caption{Robustness of ZO-SGD, ZO-Adam, and MEAZO to step size ($\eta$) initialization across models fine-tuned on SST-2 (left) and XSum (right).}
  \label{fig:robustness-sst2-xsum}
\end{figure}

\clearpage
\subsection{Grouping}\label{app:grouping}
We present a comprehensive comparison of grouped versus non-grouped zeroth-order optimization for ZO-SGD, ZO-Adam, and MEAZO across \llama{} and \qwen{}, fine-tuned on SST-2 and SQuAD. This section complements the main paper results (Fig.~\ref{fig:meazo-grouped-vs-non-grouped}), which focus on \llama{} on SQuAD as a representative setting.
\begin{figure}[thb]
    \centering
    \includegraphics[width=0.9\linewidth]{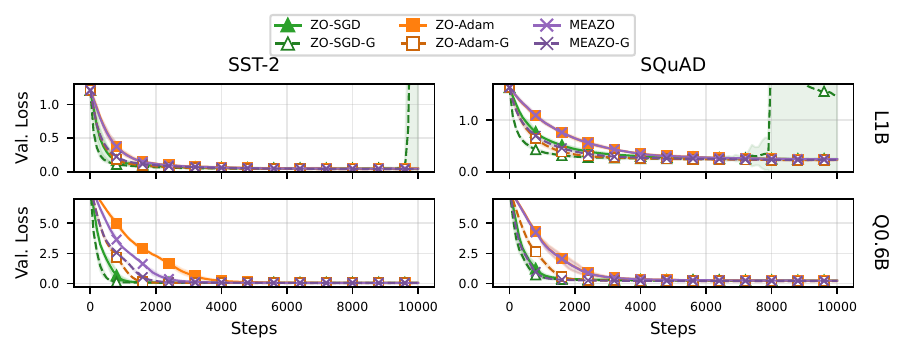}
    \caption{Convergence behavior of various optimizers under grouped and non-grouped settings for LLM fine-tuning.}
    \label{fig:group-loss-tiled}
\end{figure}

\begin{figure}[thb]
    \centering
    \includegraphics[width=0.9\linewidth]{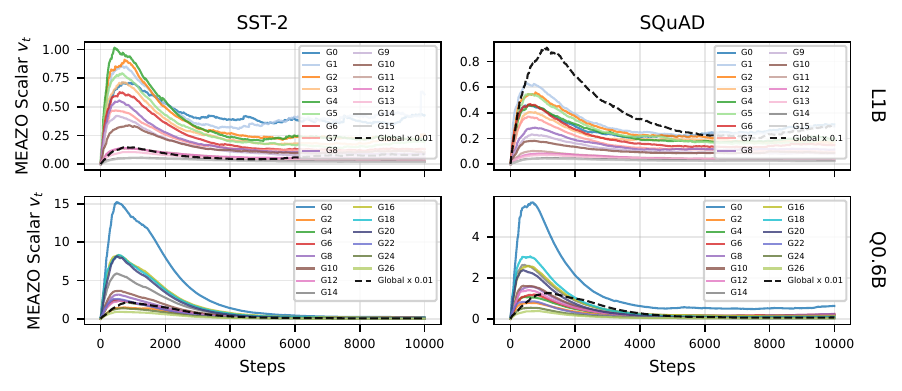}
    \caption{Evolution of the MEAZO scalar $v_t$ during LLM fine-tuning under grouped and non-grouped settings. For visualization clarity, the non-grouped $v_t$ is renormalized.}
    \label{fig:group-vt-tiled}
\end{figure}

\begin{figure}[thb]
    \centering
    \includegraphics[width=0.9\linewidth]{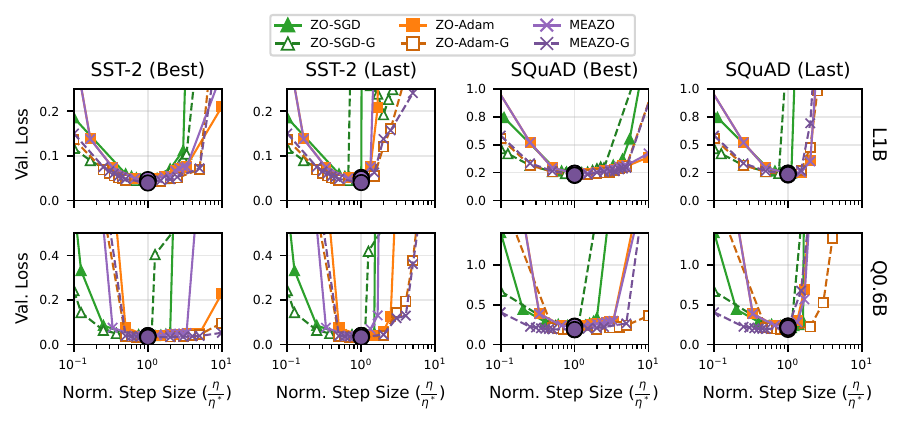}
    \caption{Robustness of grouped and non-grouped ZO-SGD, ZO-Adam, and MEAZO to step size ($\eta$) initialization across models fine-tuned on SST-2 (left) and SQuAD (right).}
    \label{fig:group-robustness-tiled}
\end{figure}

Figure~\ref{fig:group-loss-tiled} reports validation loss as a function of optimization steps for all methods, models, and datasets. Consistent with the main paper, grouping acts as a variance reduction mechanism and leads to faster convergence in terms of steps across most settings. However, this improvement does not translate directly to FLOP efficiency due to the increased cost of grouped updates, an effect that is consistent across both \llama{} and \qwen{} models.

Figure~\ref{fig:group-vt-tiled} shows the evolution of the per-group MEAZO scalar $v_t$ over training when grouping is enabled. Across all settings, earlier-layer groups consistently exhibit larger $v_t$ values than deeper-layer groups, indicating higher sensitivity to parameter perturbations and consequently more conservative effective step sizes. This stratification is particularly clear and stable in \llama{} models, while it is more heterogeneous in \qwen{}. Importantly, this behavior is not an optimizer artifact but reflects an intrinsic property of deep neural networks: perturbations in earlier layers tend to have a larger impact on the model’s output and loss, whereas later layers are comparatively more robust. As a result, adaptive control of step sizes across parameter groups aligns naturally with these layer-wise sensitivities.

Finally, Figure~\ref{fig:group-robustness-tiled} compares step-size robustness for grouped and non-grouped variants. While grouped updates typically improve stability, grouped ZO-SGD remains brittle across several settings, particularly near the end of training, where some seeds exhibit numerical divergence. These robustness curves further support the conclusion that grouping, while not always computationally optimal, provides a tangible stabilizing effect that complements adaptive methods such as MEAZO.

\clearpage
\subsection{More Models}\label{app:vit_and_ministral}

We present additional experiments beyond LLM fine-tuning, covering both vision and multimodal settings. In particular, we fine-tune a Vision Transformer (ViT-B/16, $\sim$86M parameters) on the Oxford-IIIT Pets dataset, and the multimodal Ministral 14B model on SQuAD. For both settings, we follow the same experimental protocol as in our LLM experiments, including three random seeds and the same two-stage step-size sweep procedure.

\paragraph{ViT-B/16 on Oxford-IIIT Pets.}
We fine-tune ViT-B/16 using LoRA on the Oxford-IIIT Pets dataset. Figure~\ref{fig:vit} reports the training and validation loss curves at each method’s best step size, as well as robustness to step-size perturbations. Table~\ref{tab:vit_pets_accuracy} summarizes the final test accuracy for each method alongside the corresponding optimal step size.

\begin{figure}[tbh]
    \centering
    \begin{subfigure}[t]{0.9\textwidth}
        \centering
        \includegraphics[width=\textwidth]{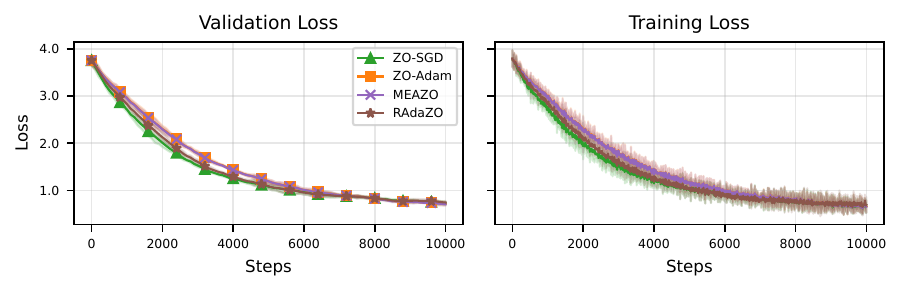}
        \caption{Loss curves}
        \label{fig:vit-loss-curves}
    \end{subfigure}
    \vspace{0.5em}
    \begin{subfigure}[t]{0.9\textwidth}
        \centering
        \includegraphics[width=\textwidth]{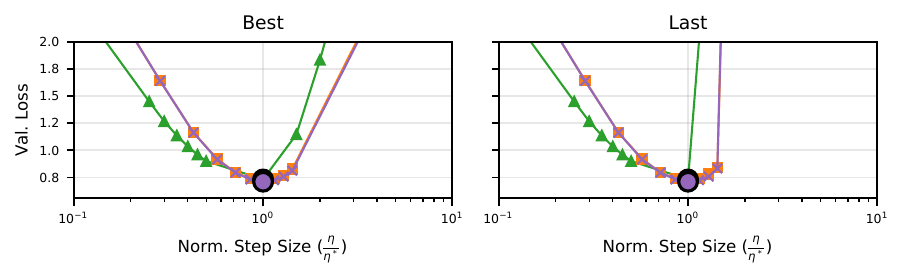}
        \caption{Robustness}
        \label{fig:vit-robustness}
    \end{subfigure}
    \caption{Training-loss dynamics and step-size robustness for zeroth-order methods on ViT-B/16 fine-tuning on Pets. (a) Training loss curves at each method’s best step size, showing the mean over three seeds with shaded regions
  denoting $\pm 2$ standard deviations. (b) Validation-loss robustness as a function of normalized step size $\eta/\eta^*$, where $\eta^*$ is each method’s best step size.}
    \label{fig:vit}
\end{figure}

\begin{table}[t]
\centering
\caption{Pets accuracy before and after 16-rank LoRA fine-tuning. Results are averaged over seeds.}
\label{tab:vit_pets_accuracy}
\begin{tabular}{lcc}
\toprule
Method & Accuracy (\%) & Best $\eta$ \\
\midrule
Base & 3.65 & - \\
\midrule
ZO-SGD & $78.64_{\pm 1.05}$ & $\expnumber{2}{-4}$ \\
ZO-Adam & $79.52_{\pm 0.89}$ & $\expnumber{7}{-4}$ \\
RAdaZO & $76.76_{\pm 1.30}$ & $\expnumber{2}{-4}$ \\
MEAZO & $79.55_{\pm 1.07}$ & $\expnumber{7}{-4}$ \\
\bottomrule
\end{tabular}
\end{table}

\paragraph{Ministral 14B on SQuAD.}
Using the same setup, we fine-tune the multimodal Ministral 14B model on the SQuAD dataset. Figure~\ref{fig:ministral} reports training and validation loss curves at the best step size for each method, as well as robustness to step-size perturbations. Final test performance and optimal step sizes are summarized in Table~\ref{tab:ministral_squad_f1}.

\begin{figure}[tbh]
    \centering
    \begin{subfigure}[t]{0.9\textwidth}
        \centering
        \includegraphics[width=\textwidth]{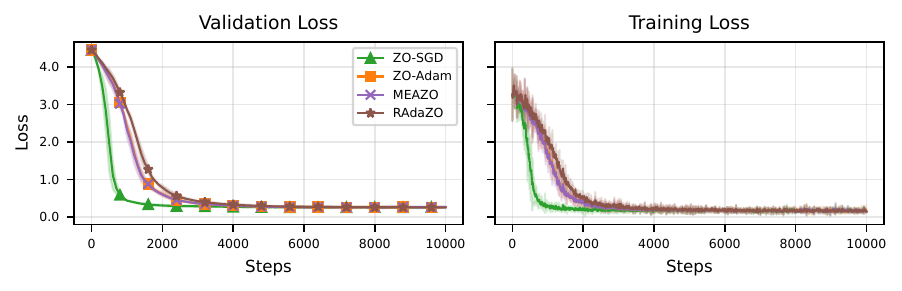}
        \caption{Loss curves}
        \label{fig:mninistral-loss-curves}
    \end{subfigure}
    \vspace{0.5em}
    \begin{subfigure}[t]{0.9\textwidth}
        \centering
        \includegraphics[width=\textwidth]{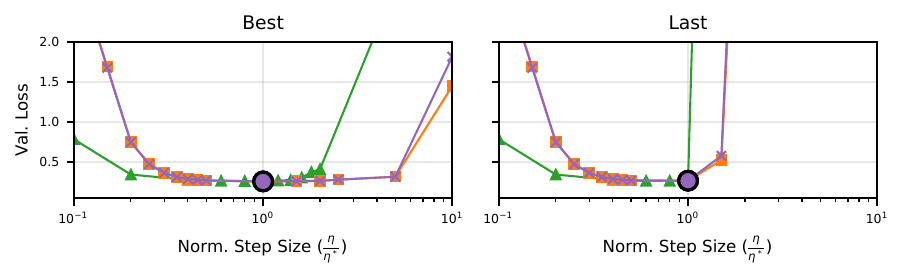}
        \caption{Robustness}
        \label{fig:ministral-robustness}
    \end{subfigure}
    \caption{Training-loss dynamics and step-size robustness for zeroth-order methods on Ministral 14B fine-tuning on SQuAD. (a) Training loss curves at each method’s best step size, showing the mean over three seeds with shaded regions
  denoting $\pm 2$ standard deviations. (b) Validation-loss robustness as a function of normalized step size $\eta/\eta^*$, where $\eta^*$ is each method’s best step size.}
    \label{fig:ministral}
\end{figure}
\begin{table}[t]
\centering
\caption{SQuAD F1 before and after 16-rank LoRA fine-tuning on Ministral-3 14B. Results are averaged over seeds.}
\label{tab:ministral_squad_f1}
\begin{tabular}{lcc}
\toprule
Method & F1 (\%) & Best $\eta$ \\
\midrule
Base & 19.85 & - \\
\midrule
ZO-SGD & $85.29_{\pm 0.56}$ & $\expnumber{5}{-5}$ \\
ZO-Adam & $85.41_{\pm 0.43}$ & $\expnumber{2}{-4}$ \\
RAdaZO & $85.99_{\pm 0.20}$ & $\expnumber{4}{-5}$ \\
MEAZO & $85.75_{\pm 0.09}$ & $\expnumber{2}{-4}$ \\
\bottomrule
\end{tabular}
\end{table}

\paragraph{Conclusion.}
Overall, the conclusions drawn from our LLM experiments carry over to both the vision and multimodal settings. First, when properly tuned, ZO-SGD performs comparably to adaptive methods. Second, per-coordinate adaptivity offers no tangible benefit: scalar-based MEAZO closely tracks the performance of ZO-Adam. Finally, the primary advantage of adaptivity lies in improved robustness to step-size selection, rather than in per-coordinate scaling itself.

% \begin{figure}[tb]
%   \centering
%   \includegraphics[width=0.48\textwidth]{img/qwen3_8B_sst2_non_grouped_lr_robustness.pdf}
%   \caption{The robustness of ZO-SGD, ZO-Adam, and MEAZO to step size ($\eta$) initialization for Qwen3 8B on SST-2.}
% \label{fig:sst2-qwen3-8b-robustness}
% \end{figure}

% \begin{figure}[tb]
%   \centering
%   \includegraphics[width=0.48\textwidth]{img/qwen3_8B_sst2_non_grouped_lr_robustness.pdf}
%   \caption{The robustness of ZO-SGD, ZO-Adam, and MEAZO to step size ($\eta$) initialization for Qwen3 8B on SQuAD.}
% \label{fig:squad-qwen3-8b-robustness}
% \end{figure}

% \begin{figure}[tb]
%   \centering
%   \includegraphics[width=0.48\textwidth]{img/qwen3_8B_sst2_non_grouped_lr_robustness.pdf}
%   \caption{The robustness of ZO-SGD, ZO-Adam, and MEAZO to step size ($\eta$) initialization for Qwen3 8B on XSum.}
% \label{fig:xsum-qwen3-8b-robustness}
% \end{figure}
\clearpage
\subsection{Impact of $\beta_1$ in Adaptive ZO Optimization}
In MEAZO (Alg.~\ref{alg:meazo}) we employ second moment normalization of the projected gradients only. MEAZO can be viewed as a compressed variant of ZO-Adam with $\beta_1 = 0$, by design. The key quantity is the average projected gradient $g$ (line 8), which captures directional information.

In the update step (line 12),
\[
\vc{x}_t \leftarrow \vc{x}_{t-1} - \frac{\eta}{\sqrt{\hat{v}} + \zeta} \left(\sum_{i=1}^q  \frac{\Delta f_i(\vc{x}_t;\xi)}{q} \vc{u}_i \right),
\]
each noise vector $\vc{u}_i$ is weighted by its corresponding projected gradient. This ensures proper directional credit assignment. Replacing the sum with a single smoothed scalar would discard this structure and degrade the update quality.

Figure~\ref{fig:zo_adam_beta1_exp} further supports this design choice: disabling $\beta_1$ in ZO-Adam has little to no impact on convergence when fine-tuning Llama 3.2 1B on XSum. This reinforces our decision to avoid smoothed first-moment updates in MEAZO.
\begin{figure}[thb]
  \centering
  \includegraphics[width=0.45\textwidth]{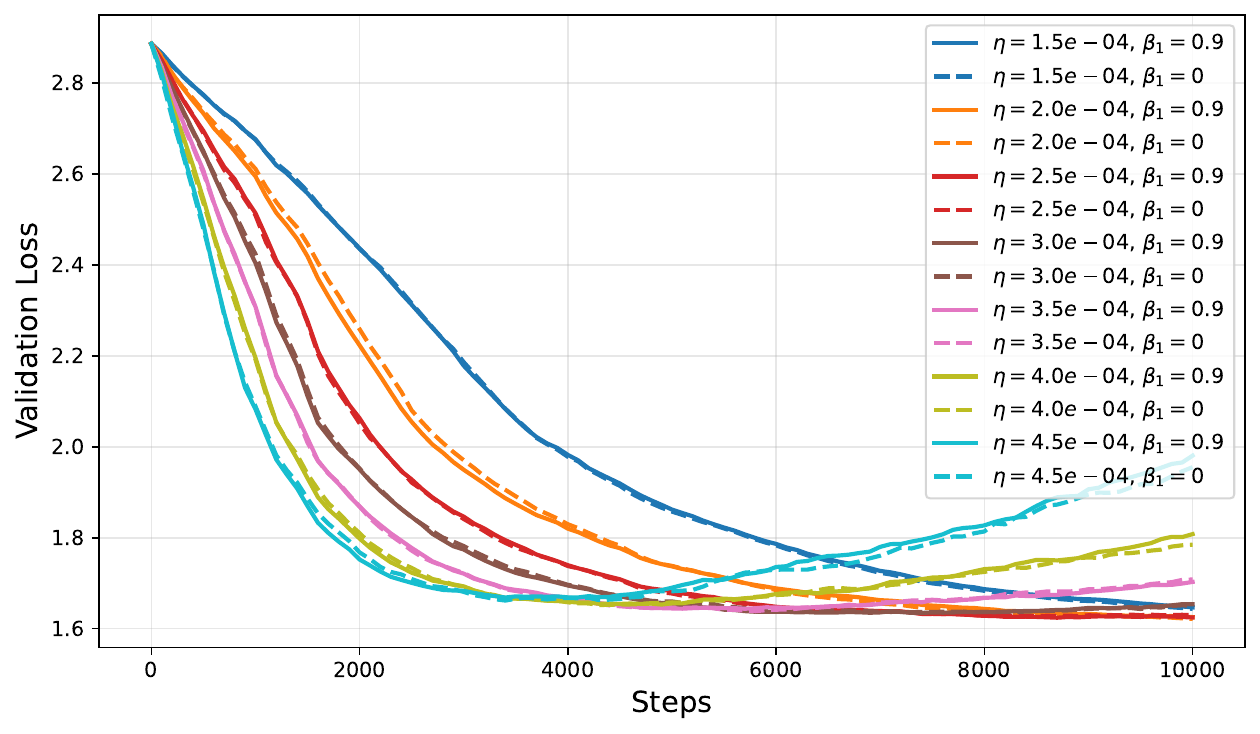}
  \caption{Validation loss curves for ZO-Adam with $\beta_1 = 0$ and $\beta_1 = 0.9$ across different step sizes during fine-tuning of Llama 3.2 1B on XSum.}
  \label{fig:zo_adam_beta1_exp}
\end{figure}

\clearpage
\subsection{Perturbation Distribution Impact}\label{app:perturbation}
We examine how varying the perturbation distribution used in the ZO gradient estimator in \eqref{eq:zo} affects performance. Specifically, we consider normal perturbations ($\vc{u}\sim\calN(\vc{0},\mtx{I}_d)$), uniform over the sphere ($\vc{u}\sim\text{Uniform}(\sphere)$), Rademacher ($u_i \stackrel{\text{i.i.d.}}{\sim} \{-1,1\}$), and ternary perturbations ($u_i \stackrel{\text{i.i.d.}}{\sim} \{-1,0,1\}$). As shown in Fig.~\ref{fig:sst2_perturbation}, across ZO-SGD, ZO-Adam, $\calR$-AdaZO, and MEAZO, the choice of perturbation distribution has negligible effect on the optimization trajectory. 
% \mr{Do we have results for the other datasets? HD: we have but incomplete (missing zo-adam, or missing ternary ..., commenting this statement out for now.)}
% In our runs over the three datasets , we consistently observed an overlap of normal and uniform perturbations curves, a relatively similar behavior of Rademacher perturbation and a slightly slower convergence of ternary perturbations that could be explained by the perturbation updating only 2/3 of the parameters per step. Eventually, all runs led to similar validation performance.

\begin{figure}[thb]
  \centering
  \includegraphics[width=0.45\textwidth]{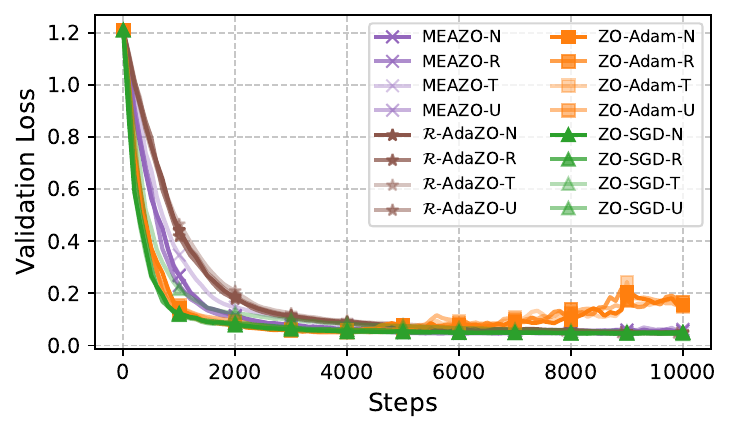}
  \caption{The impact of perturbation distributions on zero-order optimization methods for Llama-3.2 1B on SST-2: \textbf{N}ormal, \textbf{R}ademacher, \textbf{T}ernary, and \textbf{U}niform. \textit{Best seen in color.}}
\label{fig:sst2_perturbation}
\end{figure}

\clearpage
\subsection{Comparison with FZOO \citep{dang2025fzoo}}\label{app:fzoo}
FZOO proposed an alternative way to dynamically tune the step size. Given a minibatch $\xi$ at time step $t$, FZOO first computes $q$ one-sided perturbed function evaluations $\{f_i = f(\vc{x}_t + \varepsilon \vc{u}_i;\xi)\}_{i=1}^q$ and one unperturbed evaluation $f_0  = f(\vc{x}_t;\xi)$. Using these $q+1$ evaluations, FZOO computes the following normalized update:
\begin{equation}\label{eq:fzoo}
    \vc{g}_t = \frac{1}{\varepsilon q \sigma_t}\sum_{i=1}^q (f_i-f_0)\vc{u}_i
\end{equation}
where $\sigma_t = \mathsf{std}(\{f_i\})$ is the perturbed losses' standard deviation. 

Using their publicly available \href{https://github.com/DKmiyan/FZOO}{code}, we experimented with full model fine-tuning of OPT 1.3B \citep{zhang2022opt} on SST-2 using their default hyperparameters for FZOO and ZO-SGD. To test our hypothesis that a well-tuned ZO-SGD can often match adaptive methods, we experimented with the following step size\footnote{At the time of writing, the FZOO \href{https://github.com/DKmiyan/FZOO/blob/db86ec782aaea499d071b923554c25635f7cafa2/trainer.py\#L760}{implementation} did not normalize their update with $\varepsilon$ appropriately, which lead us to re-scale $\eqref{eq:fzoo-lr}$ by multiplying with $\varepsilon=\num{0.001}$.}:
\begin{equation}\label{eq:fzoo-lr}
    \eta_{\text{ZO-SGD}} = \eta_{\text{FZOO}} \Big(\frac{1}{T}\sum_{t=1}^T\sigma_t\Big)^{-1} \approx \num{2.6e-7}
\end{equation}
Computing \eqref{eq:fzoo-lr} requires first running FZOO for $T$ iterations with initial step size $\eta_{\text{FZOO}}$. We then collect all $T$ standard deviation values $\{\sigma_t\}$ and use that to \emph{guess} an appropriate step size for ZO-SGD by normalizing $\eta_{\text{FZOO}}$ with the mean of $\{\sigma_t\}$.

Figure~\ref{fig:fzoo_results} demonstrates that ZO-SGD with a carefully selected step size via \eqref{eq:fzoo-lr} can match the performance of FZOO, despite requiring $4.5\times$\footnote{ZO-SGD and MEAZO require $2q$ forward passes, whereas FZOO requires only $q+1$ forward passes.} fewer forward passes per mini-batch. For completeness, we also run MEAZO with the same step size as FZOO ($\eta=\expnumber{1}{-5}$), but with the same number of samples as ZO-SGD ($q=1$). Figure~\ref{fig:fzoo_results} also plots the performance of our proposed MEAZO, demonstrating that it can match ZO-SGD with fewer perturbations than FZOO and without the need for careful step size tuning. Notably, we observe that reducing $q$ for FZOO leads to instability, highlighting the sensitivity of its performance to the number of perturbations.

\begin{figure}[thb]
  \centering
  \includegraphics[width=0.9\textwidth]{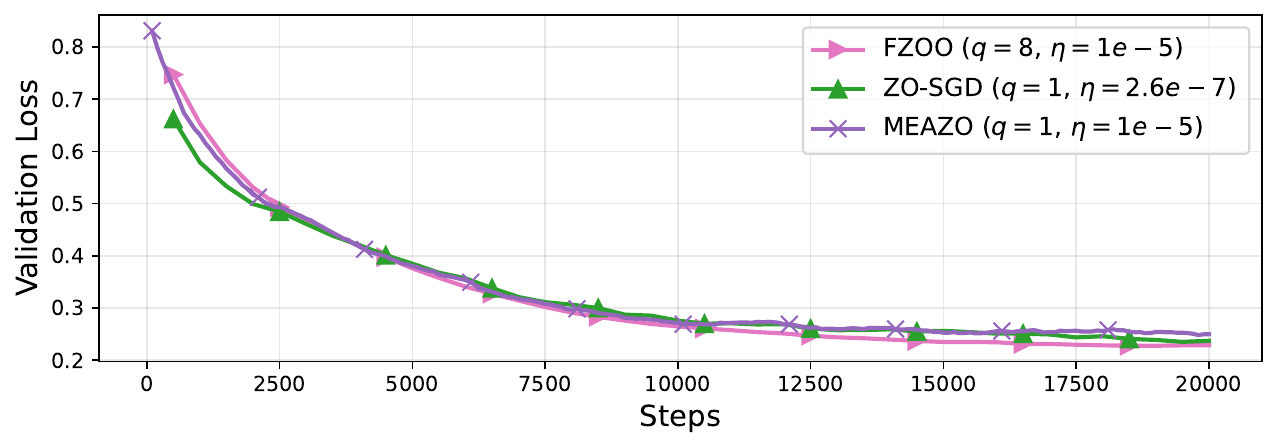}
  \caption{Performance comparison of FZOO, MEAZO, and ZO-SGD in full-model fine-tuning of OPT-1.3B on the SST-2 dataset. Results were obtained using the publicly available FZOO \href{https://github.com/DKmiyan/FZOO}{code}.}
  \label{fig:fzoo_results}
\end{figure}
\clearpage
\section{Computational Analysis of Grouped ZO}\label{app:group-cost}
Consider a $p$-layer decoder-only transformer architecture. Computing the vanilla grouped ZO gradient in \eqref{eq:zo} requires $2qp$ decoder forward passes (ignoring the language modeling head and loss computation). In contrast, a naive implementation of the grouped ZO estimator in \eqref{eq:block-zo} would require $2qp^2$ forward passes, incurring an additional factor of $p$ (e.g., $p=16$ for LLaMA~3.2~1B).

To reduce this overhead, we process the grouped ZO gradient from left to right using Alg.~\ref{alg:eff-grouped-zo}. By caching the prefix activation at each step, the total number of decoder forward calls becomes:
\begin{equation}
    (p-1) + q\sum_{j=1}^p 2(p-j) = pq(p+1) + p - 1,
\end{equation}
which scales as $\mathcal{O}(qp^2)$ but with a significantly smaller constant factor than the naive approach. For example, when $q=1$ and $p=16$, this results in roughly a $1.8\times$ reduction in forward calls compared to the naive grouped implementation, at the cost of storing a single intermediate activation.
\begin{algorithm}[thb]
   \caption{Efficient Grouped ZO}
   \label{alg:eff-grouped-zo}
\begin{algorithmic}[1]
    \STATE \textbf{Inputs:} initial input $\vc{h}_0$; full parameter vector $\vc{x} \in \reals^d$; partition $\{\mathcal{X}_j\}_{j=1}^p$; step size $\varepsilon>0$; number of queries $q$
    \STATE \textbf{Output:} grouped ZO gradient estimate $\hat{\nabla} f(\vc{x}) \in \reals^d$
    \STATE Initialize cache with prefix $\mathsf{cache} \gets \vc{h}_0$; \quad $\hat{\nabla} f(\vc{x}) \gets \vc{0}$
    \FOR{$j = 1$ \textbf{to} $p$}
        \STATE read prefix activation: $\vc{h}_{j-1} \gets \mathsf{cache}$
        \STATE compute $\vc{h}_j$ from the unperturbed block $j$
        \STATE $\mathsf{cache} \gets \vc{h}_j$ \COMMENT{update cache}
        \CCOMMENT{Sample block-local noise}
        \STATE draw $\vc{u}^{(j)}_1,\ldots,\vc{u}^{(j)}_q \in \reals^{|\mathcal{X}_j|}$ i.i.d. from $P$
        \STATE Initialize $\widehat{\nabla}^{(j)} \gets \vc{0} \in \reals^d$
       \FOR{$i = 1$ \textbf{to} $q$}
           \CCOMMENT{perturb layer $j$ and compute $f^+$} 
           \STATE $\vc{x}_j^{+} \gets \vc{x}_j + \varepsilon \vc{u}^{(j)}_i$
           \STATE $f^{+} \gets$ forward layers $j{:}p$ starting from $\vc{h}_{j-1}$ with $\vc{x}_j^{+}$ and unperturbed $\vc{x}_{k>j}$
           
           \CCOMMENT{perturb layer $j$ and compute $f^-$} 
           \STATE $\vc{x}_j^{-} \gets \vc{x}_j - \varepsilon \vc{u}^{(j)}_i$
           \STATE $f^{-} \gets$ forward layers $j{:}p$ starting from $\vc{h}_{j-1}$ with $\vc{x}_j^{-}$ and unperturbed $\vc{x}_{k>j}$
                   
           \CCOMMENT{accumulate} 
           \STATE $\widehat{\nabla}^{(j)} \gets \widehat{\nabla}^{(j)} + \dfrac{f^{+} - f^{-}}{2\varepsilon}\tilde{\vc{u}}^{(j)}_i$ \COMMENT{$\tilde{\vc{u}}^{(j)}_i \in \reals^d$ is its zero-padded version with nonzeros only on $\mathcal{X}_j$}
        \ENDFOR
        \CCOMMENT{block contribution} 
        \STATE $\hat{\nabla} f(\vc{x}) \gets \hat{\nabla} f(\vc{x}) + \dfrac{1}{q}\widehat{\nabla}^{(j)}$
        \CCOMMENT{evict old prefix} 
        \STATE discard $\vc{h}_{j-1}$; retain only $\vc{h}_j$
    \ENDFOR
\end{algorithmic}
\end{algorithm}

%%%%%%%%%%%%%%%%%%%%%%%%%%%%%%%%%%%%%%%%%%%%%%%%%%%%%%%%%%%%%%%%%%%%%%%%%%%%%%%
\clearpage

\section{General Definitions and Lemmas}
Let us state first useful optimization definitions and lemmas needed for our convergence proofs.
\begin{definition}[$L$-smoothness]
    A differentiable function $f:\reals^d \to \reals$ is said to be \emph{$L$-smooth} if its gradient is $L$-Lipschitz continuous, that is,
    \[
        \|\nabla f(\vc{x}) - \nabla f(\vc{y})\| \leq L \|\vc{x} - \vc{y}\|, \quad \forall \vc{x},\vc{y} \in \reals^d.
    \]
    Equivalently, for all $\vc{x},\vc{y} \in \reals^d$,
    \begin{equation}\label{eq:smooth_quadratic_upper_bound}
        f(\vc{y}) \leq f(\vc{x}) + \langle \nabla f(\vc{x}), \vc{y} - \vc{x} \rangle + \frac{L}{2}\|\vc{y} - \vc{x}\|^2.
    \end{equation}
\end{definition}
\begin{lemma}[Lemma 1.2.3 of~\cite{nesterov2013introductory}]
    If a differentiable function $f:\reals^d \to \reals$ is \emph{$L$-smooth}, then for all $\vc{x},\vc{y} \in \reals^d$,
    \begin{equation}\label{eq:abs_smooth_sandwich_bound}
        |f(\vc{y}) - f(\vc{x}) - \langle \nabla f(\vc{x}), \vc{y} - \vc{x} \rangle| \leq \frac{L}{2}\|\vc{y} - \vc{x}\|^2.
    \end{equation}
\end{lemma}
% \begin{lemma}
%     If a differentiable function $f:\reals^d \to \reals$ is \emph{$L$-smooth}, then for all $\vc{x},\vc{y} \in \reals^d$,
%     \begin{equation}\label{eq:smooth_neg_quadratic_lower_bound}
%         f(\vc{y}) \geq f(\vc{x}) + \langle \nabla f(\vc{x}), \vc{y} - \vc{x} \rangle - \frac{L}{2}\|\vc{y} - \vc{x}\|^2.
%     \end{equation}
% \end{lemma}
% \begin{proof}
%     Let us write $\vc{d} \defeq \vc{y} - \vc{x}$. From the fundamental theorem of calculus we have
%     \begin{align*}
%         f(\vc{y}) - f(\vc{x}) &= \int_0^1 \langle \nabla f(\vc{x}+t\vc{d}), \vc{d} \rangle \mathrm{d} t \\
%         &= \langle \nabla f(\vc{x}), \vc{d} \rangle + \int_0^1 \langle \nabla f(\vc{x}+t\vc{d}) - \nabla f(\vc{x}), \vc{d} \rangle \mathrm{d} t \\
%         &\geq \langle \nabla f(\vc{x}), \vc{d} \rangle - \int_0^1 |\langle \nabla f(\vc{x}+t\vc{d}) - \nabla f(\vc{x}), \vc{d} \rangle| \mathrm{d} t .
%     \end{align*}
%     Let us focus on the second term. Using Cauchy-Schwarz and $L$-smoothness, one gets
%     \begin{equation*}
%         |\langle \nabla f(\vc{x}+t\vc{d}) - \nabla f(\vc{x}), \vc{d} \rangle| \leq L t \norm{\vc{d}}^2.
%     \end{equation*}
%     Plugging in the later upper bound in the previous derivation, it leads to
%     \begin{align*}
%         f(\vc{y}) - f(\vc{x}) &\geq \langle \nabla f(\vc{x}), \vc{d} \rangle - \int_0^1 L t \norm{\vc{d}}^2 \mathrm{d} t \\
%         \overset{\vc{d} = \vc{y} - \vc{x}}&{=} \langle \nabla f(\vc{x}), \vc{y} - \vc{x} \rangle - \frac{L}{2} \norm{\vc{y} - \vc{x}}^2.
%     \end{align*}
% \end{proof}

\begin{lemma}[ZO Projected Gradient upper bound]\label{lem:zo_proj_grad_upper_bound}
    Let $\vc{x}, \vc{u} \in \reals^d$ and $\varepsilon > 0$. 
    If a differentiable function $f:\reals^d \to \reals$ is \emph{$L$-smooth}, then
    \begin{equation}\label{eq:zo_proj_grad_upper_bound}
        \langle \nabla f(\vc{x}), \vc{u} \rangle - \frac{L \varepsilon}{2} \norm{\vc{u}}^2 \leq \frac{f(\vc{x} + \varepsilon \vc{u}) - f(\vc{x} - \varepsilon \vc{u})}{2 \varepsilon} \leq \langle \nabla f(\vc{x}), \vc{u} \rangle + \frac{L \varepsilon}{2} \norm{\vc{u}}^2 ,
    \end{equation}
    which implies that
    \begin{equation}\label{eq:zo_proj_grad_abs_upper_bound}
        \left| \frac{f(\vc{x} + \varepsilon \vc{u}) - f(\vc{x} - \varepsilon \vc{u})}{2 \varepsilon} \right| \leq \left| \langle \nabla f(\vc{x}), \vc{u} \rangle \right| + \frac{L \varepsilon}{2} \norm{\vc{u}}^2 .
    \end{equation}
\end{lemma}
\begin{proof}
    Applying \eqref{eq:smooth_quadratic_upper_bound} with $\vc{y} \defeq \vc{x} + \varepsilon \vc{u}$ leads to:
    \begin{equation*}
        f(\vc{x} + \varepsilon \vc{u}) \leq f(\vc{x}) + \varepsilon \langle \nabla f(\vc{x}), \vc{u} \rangle + \frac{L \varepsilon^2}{2}\norm{\vc{u}}^2.
    \end{equation*}
    Similarly, applying \eqref{eq:abs_smooth_sandwich_bound} with $\vc{y} \defeq \vc{x} - \varepsilon \vc{u}$ and multiplying the left-hand side by $-1$ gives
    \begin{equation*}
        - f(\vc{x} - \varepsilon \vc{u}) \leq -f(\vc{x}) + \varepsilon \langle \nabla f(\vc{x}), \vc{u} \rangle + \frac{L \varepsilon^2}{2}\norm{\vc{u}}^2.
    \end{equation*}
    Summing up the above two equations and dividing by $2 \varepsilon$ proves to~\eqref{eq:zo_proj_grad_upper_bound}:
    \begin{equation}\label{eq:zo_proj_grad_upper_bound_u}
        \frac{f(\vc{x} + \varepsilon \vc{u}) - f(\vc{x} - \varepsilon \vc{u})}{2 \varepsilon} \leq \langle \nabla f(\vc{x}), \vc{u} \rangle + \frac{L \varepsilon}{2} \norm{\vc{u}}^2 .
    \end{equation}    
    One can also apply the proved inequality to $\vc{u} \leftarrow - \vc{u}$ and multiply it by $-1$, which proves the left-hand side inequality of:
    \begin{equation}\label{eq:zo_proj_grad_upper_bound_minus_u}
        \langle \nabla f(\vc{x}), \vc{u} \rangle - \frac{L \varepsilon}{2} \norm{\vc{u}}^2 \leq \frac{f(\vc{x} + \varepsilon \vc{u}) - f(\vc{x} - \varepsilon \vc{u})}{2 \varepsilon} \leq \langle \nabla f(\vc{x}), \vc{u} \rangle + \frac{L \varepsilon}{2} \norm{\vc{u}}^2 .
    \end{equation}
    Applying the fact that for any $A,M \in \reals$ and $B \geq 0$, $A-B \leq M \leq A+B \implies \left| M \right| \leq \left| A \right| + B$ to inequalities in~\eqref{eq:zo_proj_grad_upper_bound_minus_u} concludes the proof.
\end{proof}

\begin{lemma}[Smoothing - $G$-Lipschitzness Transfer]\label{lemma:smoothing-preserves-lipschitzness}
    Let $F:\reals^d \to \reals$ be a $G$-Lipschitz function. Then its smoothed version $F_\varepsilon(\vc{x}) = \mathbb{E}_{\vc{u} \sim \mathcal{D}} \big[ F(\vc{x} + \varepsilon \vc{u}) \big]$ is also $G$-Lipschitz.
\end{lemma}
\begin{proof}
    Since $F$ is $G$-Lipschitz, for any $\vc{x}, \vc{y} \in \reals^d$ we have
    \[
        \left| F(\vc{x}) - F(\vc{y}) \right| \le G \left\| \vc{x} - \vc{y} \right\|.
    \]
    For the smoothed function $F_\varepsilon$, consider
    \[
        \left| F_\varepsilon(\vc{x}) - F_\varepsilon(\vc{y}) \right|
        = \left| \means{\vc{u} \sim \mathcal{D}}{ F(\vc{x} + \varepsilon \vc{u}) - F(\vc{y} + \varepsilon \vc{u}) } \right|.
    \]
    By Jensen's inequality and the Lipschitz property of $F$,
    \[
        \left| F_\varepsilon(\vc{x}) - F_\varepsilon(\vc{y}) \right|
        \le \means{\vc{u} \sim \mathcal{D}}{ \left| F(\vc{x} + \varepsilon \vc{u}) - F(\vc{y} + \varepsilon \vc{u}) \right| }
        \le \means{\vc{u} \sim \mathcal{D}}{ G \left\| \vc{x} - \vc{y} \right\| }
        = G \left\| \vc{x} - \vc{y} \right\|.
    \]
    Hence $F_\varepsilon$ is $G$-Lipschitz.
\end{proof}

\begin{lemma}[Smoothing - $L$-smoothness Transfer]\label{lemma:smoothing-preserves-smoothness}
Let $F:\reals^d \to \reals$ be an $L$-smooth function. Then its smoothed version $F_\varepsilon(\vc{x}) = \mathbb{E}_{\vc{u} \sim \mathcal{D}} \big[ F(\vc{x} + \varepsilon \vc{u}) \big]$ is also $L$-smooth.
\end{lemma}
\begin{remark}
    The above lemma can be found in eq (12) of~\cite{nesterov2017random} for normal perturbation, Lemma 4.1 (a) of~\cite{gao2018information} for uniform spherical perturbation.
\end{remark}
\begin{proof}
    By Leibniz's rule (differentiation under the expectation), we have
    \[
        \nabla F_\varepsilon(\vc{x}) = \mathbb{E}_{\vc{u}} \big[ \nabla F(\vc{x} + \varepsilon \vc{u}) \big].
    \]
    For any $\vc{x}, \vc{y} \in \reals^d$,
    \[
        \|\nabla F_\varepsilon(\vc{x}) - \nabla F_\varepsilon(\vc{y})\| = \Big\| \mathbb{E}_{\vc{u}} \big[ \nabla F(\vc{x} + \varepsilon \vc{u}) - \nabla F(\vc{y} + \varepsilon \vc{u}) \big] \Big\|.
    \]
    Applying Jensen's inequality (since $\|\cdot\|$ is convex), we get
    \[
        \|\nabla F_\varepsilon(\vc{x}) - \nabla F_\varepsilon(\vc{y})\| \leq \mathbb{E}_{\vc{u}} \big[ \|\nabla F(\vc{x} + \varepsilon \vc{u}) - \nabla F(\vc{y} + \varepsilon \vc{u})\| \big] \tag{a}.
    \]
    Since $F$ is $L$-smooth, its gradient is $L$-Lipschitz, so
    \[
        \|\nabla F(\vc{x} + \varepsilon \vc{u}) - \nabla F(\vc{y} + \varepsilon \vc{u})\| \leq L \|\vc{x} - \vc{y}\| \tag{b}.
    \]
    Combining (a) and (b), we obtain
    \[
        \|\nabla F_\varepsilon(\vc{x}) - \nabla F_\varepsilon(\vc{y})\| \leq L \|\vc{x} - \vc{y}\|.
    \]
    Hence, $F_\varepsilon$ is also $L$-smooth.
\end{proof}

\begin{lemma}[Smoothing - Function Error]\label{lemma:func-and-smoothed-func-under-smoothness}
    Let $F:\reals^d \to \reals$ be an $L$-smooth function, and let $\mathcal{P}$ be any centered probability distribution. Define its smoothed version as
    \[
        F_\varepsilon(\vc{x}) = \means{\vc{u}\sim\mathcal{P}}{ F(\vc{x} + \varepsilon \vc{u}) }, \quad \forall \vc{x} \in \reals^d .
    \]
    Then, for any $\vc{x} \in \reals^d$,
    \begin{equation}
        | F_\varepsilon(\vc{x}) - F(\vc{x}) |
        \leq \frac{\varepsilon^2}{2} L \means{\vc{u}\sim\mathcal{P}}{ \norm{\vc{u}}^2 } .
    \end{equation}
\end{lemma}
\begin{proof}
    Let $\vc{x}, \vc{u} \in \reals^d$ and $\varepsilon >0$, by $L$-smoothness of $F$ one has
    \begin{align}
        F ( \vc{x} + \varepsilon \vc{u} ) - F ( \vc{x} ) &\leq \varepsilon \langle \nabla F (\vc{x}), \vc{u} \rangle + \frac{\varepsilon^2}{2} L \norm{\vc{u}}^2 \nonumber \\
        \implies \means{\vc{u}\sim\mathcal{P}}{F ( \vc{x} + \varepsilon \vc{u} ) - F ( \vc{x} )} &\leq \varepsilon \langle \nabla F (\vc{x}), \means{\vc{u}\sim\mathcal{P}}{\vc{u}} \rangle + \frac{\varepsilon^2}{2} L \means{\vc{u}\sim\mathcal{P}}{\norm{\vc{u}}^2} \nonumber \\
        \overset{\mathcal{P}\text{ centered}}{\iff} F_\varepsilon(\vc{x}) - F ( \vc{x} ) &\leq  \frac{\varepsilon^2}{2} L \means{\vc{u}\sim\mathcal{P}}{\norm{\vc{u}}^2} .
    \end{align}
    By re-applying the above reasoning by inverting $\vc{x}$ and $\vc{x} + \varepsilon \vc{u}$ gives the symmetric bound.
\end{proof}

\begin{lemma}[Smoothing - Gradient Error]\label{lemma:gradient-and-smoothed-gradient-under-smoothness}
    Let $F:\reals^d \to \reals$ be an $L$-smooth function, and let $\mathcal{P}$ be any probability distribution supported on a bounded set in $\reals^d$. Define its smoothed version as
    \[
        F_\varepsilon(\vc{x}) = \means{\vc{u}\sim\mathcal{P}}{ F(\vc{x} + \varepsilon \vc{u}) }, \quad \forall \vc{x} \in \reals^d .
    \]
    Then, for any $\vc{x} \in \reals^d$,
    \begin{equation}
        \left\| \nabla F_\varepsilon(\vc{x}) - \nabla F(\vc{x}) \right\|
        \le \varepsilon L \means{\vc{u}\sim\mathcal{P}}{ \left\| \vc{u} \right\| }.
    \end{equation}
\end{lemma}
\begin{remark}
    A similar proof scheme is used in~\cite{berahas2022theoretical} to prove similar bounds for normal, eq. (2.10), and uniform distributions, eq. (2.35).
\end{remark}
\begin{proof}
    We have
    \begin{align*}
        \norm{\nabla F_\varepsilon(\vc{x}) - \nabla F(\vc{x})} &= \norm{\nabla \means{\vc{u}\sim\mathcal{P}}{F(\vc{x}+\varepsilon \vc{u})} - \nabla F(\vc{x})}\\
        \overset{\text{Leibniz's}}&{=} \norm{ \means{\vc{u}\sim\mathcal{P}}{\nabla F(\vc{x}+\varepsilon \vc{u})} - \nabla F(\vc{x})} \\
         &= \norm{ \means{\vc{u}\sim\mathcal{P}}{\nabla F(\vc{x}+\varepsilon \vc{u}) - \nabla F(\vc{x})}} \\
         \overset{\text{Jensen's}}&{\leq}  \means{\vc{u}\sim\mathcal{P}}{\norm{\nabla F(\vc{x}+\varepsilon \vc{u}) - \nabla F(\vc{x})}} \\
         \overset{L\text{-smoothness}}&{\leq} \varepsilon L \means{\vc{u}\sim\mathcal{P}}{\norm{\vc{u}}} .
    \end{align*}
\end{proof}

%% TODO: delete?
\begin{lemma}[Smoothing - Stochastic Gradient Variance Bound]\label{lemma:sto_grad_smoothed_var_bound}
    Let $F: \reals^d \to \reals$ be the population loss defined as
    \begin{equation}\label{eq:restated_full_loss}
        F(\vc{x}) = \means{\xi \sim \mathcal{D}}{f(\vc{x}; \xi)} \quad \forall \vc{x} \in \reals^d,
    \end{equation}
    where each individual sample loss $f(. ; \xi)$. 
    Let $\mathcal{P}$ be any probability distribution supported on a bounded set in $\reals^d$. 
    Let us recalled the definition of the smoothed version of $F$
    \begin{equation}\label{eq:restated_smoothed_full_loss}
        F_\varepsilon(\vc{x}) \defeq \means{\vc{u}\sim\mathcal{P}}{ F(\vc{x} + \varepsilon \vc{u}) }, \quad \forall \vc{x} \in \reals^d ,
    \end{equation}
    and similarly for $f(., \xi)$
    \begin{equation}\label{eq:restated_smoothed_sample_loss}
        f_\varepsilon(\vc{x}; \xi) \defeq \means{\vc{u}\sim\mathcal{P}}{ f(\vc{x} + \varepsilon \vc{u}; \xi) }, \quad \forall \vc{x} \in \reals^d, \forall \xi .
    \end{equation}
    % Assume that there exists $\sigma > 0$ such that, for all $\vc{x} \in \reals^d$, the variance of the sample loss gradient is bounded by $\sigma^2$, \textit{i.e.,}
    % \begin{equation}\label{eq:restated_sto_grad_var_bound}
    %     \means{\xi}{\norm{\nabla f(\vc{x}; \xi) - \nabla F(\vc{x})}^2} \le \sigma^2 .
    % \end{equation}
    Assume that, for all $\vc{x} \in \reals^d$, the sample loss gradient is an unbiased estimate of the the full loss gradient\footnote{If stochastic losses are regular enough, \textit{e.g.,} $\means{\xi}{\norm{\nabla f(\vc{x} ; \xi)} } < \infty$ such that dominated convergence theorem can be applied, then $\means{\xi}{ \nabla f(\vc{x ; \xi}) } = \nabla F (\vc{x})$.}, 
    and that there exists $\sigma > 0$ such that, its variance is bounded by $\sigma^2$, \textit{i.e.,}
    \begin{align}
        &\means{\xi}{ \nabla f(\vc{x ; \xi}) } = \nabla F (\vc{x}) , \nonumber \\
        &\means{\xi}{\norm{\nabla f(\vc{x}; \xi) - \nabla F(\vc{x})}^2} \le \sigma^2 . \label{eq:restated_sto_grad_var_bound}
    \end{align}
    Then, for any $\vc{x} \in \reals^d$, the smoothed sample loss gradient is also an unbiased estimator of the smoothed full loss and its variance of is also bounded by $\sigma^2$, \textit{i.e.,}
    \begin{align}
        &\means{\xi}{ \nabla f_\varepsilon (\vc{x ; \xi}) } = \nabla F_\varepsilon (\vc{x}) , \\
        &\means{\xi}{\norm{\nabla f_\varepsilon (\vc{x}; \xi) - \nabla F_\varepsilon(\vc{x})}^2} \le \sigma^2 .
    \end{align}
\end{lemma}
\begin{proof}
    Let $\vc{x} \in \reals^d$. We have that the mean of the smoothed sample loss equals
    \begin{equation*}
        \means{\xi}{ \nabla f_\varepsilon (\vc{x ; \xi}) } 
        \overset{\eqref{eq:restated_smoothed_sample_loss}}{=}\means{\xi}{ \nabla \means{\vc{u}}{f (\vc{x} + \varepsilon \vc{u}; \xi)}}
        \overset{\text{Leibniz' rule}}{=} \nabla \means{\vc{u}}{ \means{\xi}{ f (\vc{x} + \varepsilon \vc{u}; \xi)}}
        \overset{\eqref{eq:restated_full_loss}}{=} \nabla \means{\vc{u}}{ F (\vc{x} + \varepsilon \vc{u})}
        \overset{\eqref{eq:restated_smoothed_full_loss}}{=} \nabla F_\varepsilon (\vc{x}) .
    \end{equation*}
    We also have that the variance of the smoothed sample loss equals
    \begin{align*}
        \means{\xi}{\norm{\nabla f_\varepsilon(\vc{x};\xi)-\nabla F_\varepsilon(\vc{x})}^2}
        &= \means{\xi}{\norm{\nabla \means{\vc{u}}{f(\vc{x}+\varepsilon\vc{u};\xi)- F(\vc{x}+\varepsilon\vc{u})}}^2} \\
        \overset{\text{Leibniz' rule}}&{=} \means{\xi}{\norm{\means{\vc{u}}{\nabla f(\vc{x}+\varepsilon\vc{u};\xi)-\nabla F(\vc{x}+\varepsilon\vc{u})}}^2} \\
        \overset{\text{Jensen's ineq.}}&{\leq}\means{\xi}{\means{\vc{u}}{\norm{\nabla f(\vc{x}+\varepsilon\vc{u};\xi)-\nabla F(\vc{x}+\varepsilon\vc{u})}^2}} \\
        \overset{\eqref{eq:restated_sto_grad_var_bound}}&{\leq} \sigma^2
    \end{align*}    
\end{proof}

\subsection{Uniform Perturbations}
In this section we will derive results for the \emph{$q$-uniform samples centered ZO estimator}:
\begin{equation}\label{eq:zo-uniform}
    \hat{\nabla} f_{\varepsilon}^q(\vc{x}) = \frac{d}{q}\sum_{i=1}^q\frac{f(\vc{x}+\varepsilon \vc{u}_i)-f(\vc{x}-\varepsilon \vc{u}_i)}{2\varepsilon}\vc{u}_i,
\end{equation}
where $\vc{u}_i \sim \text{Uniform}(\mathbb{S})$ with $\mathbb{S} = \{\vc{u}\in\reals^d: \norm{\vc{u}}=1\}$ being the unit sphere of dimensions $d$ and $\varepsilon > 0$.

We now state useful and known results.

\begin{definition}[Uniform Smoothed Function]
    Let us denote by $f_\varepsilon$, the smoothed function over the unit ball $\ball$ of dimensions $d$, that is for all $\vc{x} \in \reals^d$
    \begin{equation}\label{eq:smoothed_funct_ball}
        f_\varepsilon(\vc{x}) \defeq \means{\vc{v}\sim \ball}{f(\vc{x}+\varepsilon\vc{v})}.
    \end{equation}
\end{definition}
\begin{remark}
    This is the same definition as in eq. (4.1) of~\cite{shalev2011online} and eq. (50) of~\cite{gao2018information}.
\end{remark}

\begin{lemma}[Unbiased ZO Gradient Estimate of the Smoothed Function]\label{lem:unbiased_zo_grad_sphere}
    The ZO estimator in~\eqref{eq:zo-uniform} provides an unbiased estimate of the gradient of the smoothed function $f_\varepsilon$ defined in~\eqref{eq:smoothed_funct_ball}:
    \begin{equation}
        \means{\calU_q \sim (\sphere)^q}{\hat{\nabla} f_{\varepsilon}^q(\vc{x})} = \nabla f_\varepsilon(\vc{x}),
    \end{equation}
    where $\calU_q \defeq \{\vc{u}_1, \ldots, \vc{u}_q\}$ denotes the $q$ random vectors sampled i.i.d. from the unit sphere $\sphere$.
\end{lemma}
\begin{proof}
    This is a direct consequence of Lemma 1 of~\cite{flaxman2004online} or of Lemma 4.4 of~\cite{shalev2011online} applied to $\vc{u}$ and $-\vc{u}$.
    % This is a known result, due to Stoke's theorem, first (I think) shown in .
\end{proof}
% \hd{I currently never use this result, but keeping it here}
% \begin{lemma}[Gradient and Smoothed Gradient Norms, Lemma 4.1 (b) of~\cite{gao2018information}]
% Let $\varepsilon > 0$. If $f$ is $L$-smooth, then:
% \begin{equation}
%     \norm{\nabla f_\varepsilon(\vc{x}) - \nabla f(\vc{x})} \leq  \frac{\varepsilon}{2} L d .
% \end{equation}
% Moreover, using the $2ab \leq a^2 + b^2$ inequality, it implies that
% \begin{equation}\label{eq:smooth-grad-bound-uniform}
%     \norm{\nabla f_\varepsilon(\vc{x})}^2 \leq 2\norm{\nabla f(\vc{x})}^2 +\frac{\varepsilon^2}{2} L^2 d^2.
% \end{equation}
% and symmetrically
% \begin{equation}\label{eq:grad-bound-uniform-smooth-grad}
%     \norm{\nabla f(\vc{x})}^2 \leq 2\norm{\nabla f_\varepsilon(\vc{x})}^2 +\frac{\varepsilon^2}{2} L^2 d^2.
% \end{equation}
% \end{lemma}

\newpage

\clearpage
\section{Omitted Proofs and Derivations}
\subsection{Proof of Proposition~\ref{prop:zo-sq}}\label{app:proof-zo-sq}
We provide the proof of Proposition~\ref{prop:zo-sq} (restated below):
% \begin{proposition*}[Restated]
%     Let \( f: \reals^d \to \reals \) be a differentiable function and $\hat{\nabla} f_{\varepsilon}^q(\vc{x}) $ be the $q$-sample ZO estimator in \eqref{eq:zo} with $\vc{u}_i \sim \calN(\vc{0}, \mtx{I}_d), \forall i \in [q] \defeq \{1, \ldots, q\}$. Then in the $\varepsilon \rightarrow 0$ limit we have:
%     \begin{equation}
%         \elem{\means{\vc{u}}{\hat{\vc{g}}^2_q(\vc{x})}}{k} \approx \frac{1}{q}\left( \norm{\nabla f(\vc{x})}^2 + \elem{\nabla f(\vc{x})}{k}^2\right) + \elem{\nabla f(\vc{x})}{k}^2
%     \end{equation}   
% \end{proposition*}

\begin{proposition*}[Restated]
Let \( f: \reals^d \to \reals \) be an \(L\)-smooth function, and let 
\(\hat{\nabla} f_{\varepsilon}^q(\vc{x})\) denote the \(q\)-sample ZO estimator in \eqref{eq:zo}, 
where \(\vc{u}_i \sim \calN(\vc{0}, \mtx{I}_d)\) for all \( i \in [q] \defeq \{1,\ldots,q\}\). Then in the $\varepsilon \rightarrow 0$ limit we have:
    \begin{equation}
        \means{\vc{u}}{\hat{\nabla} f_{\varepsilon}^q(\vc{x})^2} \approx \frac{1}{q}\left( \norm{\nabla f(\vc{x})}^2 \ones_d + \nabla f(\vc{x})^2 \right) + \nabla f(\vc{x})^2,
    \end{equation}
    where $\vc{(\cdot)}^2$ stands for the element-wise power of 2 and $\E_{\vc{u}}$ denotes the expectation over $\vc{u}_1, \dots, \vc{u}_q$.
\end{proposition*}

% \ngdi{Need to be careful: $\means{\vc{u}}{.}$ is in fact taking expectation over $\vc{u}_1, \dots, \vc{u}_q$}
\begin{proof}
We first note that from Lemma~\ref{lem:zo_proj_grad_upper_bound}, we have $\forall \vc{u} \in \reals^d$:
\begin{equation}
    \lim_{\varepsilon \to 0}\frac{f(\vc{x}+\varepsilon\vc{u}) - f(\vc{x}-\varepsilon\vc{u})}{2\varepsilon} = \tp{\vc{u}}\nabla f(\vc{x}).
\end{equation}
Then, in the $\varepsilon \to 0 $ limit we have:
\begin{equation}\label{eq:proj-grad-approx}
    \hat{\nabla} f_{\varepsilon}^q(\vc{x}) = \frac{1}{q} \sum_{i=1}^q \frac{f(\vc{x}+\varepsilon \vc{u}_i) - f(\vc{x}-\varepsilon \vc{u}_i)}{2\varepsilon} \vc{u}_i \approx  \frac{1}{q} \sum_{i=1}^q  \left(\tp{\vc{u}}_i\nabla f(\vc{x})\right)\vc{u}_i = \frac{1}{q} \sum_{i=1}^q \left(\tp{\vc{g}}\vc{u}_i\right)\vc{u}_i
\end{equation}
where $\vc{g}=\nabla f(\vc{x})$ and we drop the dependency on $\vc{x}$ for clarity.
Taking the expectation of the $k^{\text{th}}$ element squared, we get:
\begin{align}\label{eq:prop-expanded}
% \begin{split}
    \elem{\means{\vc{u}}{\hat{\nabla} f_{\varepsilon}^q(\vc{x})^2}}{k} &\overset{\text{\eqref{eq:proj-grad-approx}}}{\approx} \frac{1}{q^2} \means{\vc{u}}{\left(\sum_{i=1}^q \left(\tp{\vc{g}}\vc{u}_i\right)\vc{u}_i\right)^2} \\
    &=\frac{1}{q^2} \means{\vc{u}}{\sum_{i=1}^q \left(\tp{\vc{g}}\vc{u}_i\right)^2u^2_{i,k} + \sum_{i=1}^q \sum_{j\neq i} (\tp{\vc{g}}\vc{u}_i)u_{i,k}(\tp{\vc{g}}\vc{u}_j)u_{j,k}} \\
    &= \frac{1}{q^2} \sum_{i=1}^q \means{\vc{u}}{\left(\tp{\vc{g}}\vc{u}_i\right)^2u^2_{i,k} } +\frac{1}{q^2} \sum_{i=1}^q \sum_{j\neq i}\means{\vc{u}}{  (\tp{\vc{g}}\vc{u}_i)u_{i,k}(\tp{\vc{g}}\vc{u}_j)u_{j,k}}
% \end{split}
\end{align}

Expanding the first term with $\vc{v}\coloneq \vc{u}_i \sim \calN(\vc{0}, \mtx{I}_d)$:
\begin{align*}
\means{\vc{v}}{\big(\tp{\vc{g}}\vc{v}\big)^2 v_k^2} 
&= \means{\vc{v}}{\Big(\sum_{l=1}^d g_l v_l\Big)^2 v_k^2} \\
&= \means{\vc{v}}{\Big(\sum_{l=1}^d g_l^2 v_l^2 + \sum_{l\neq s} g_l g_s v_l v_s\Big)v_k^2} \\
&= \sum_{l=1}^d g_l^2 \means{\vc{v}}{v_l^2 v_k^2} 
\;+\; \underbrace{\sum_{l\neq s} g_l g_s \means{\vc{v}}{v_l v_s v_k^2}}_{\text{vanishes since odd moments = 0}} \\
&= g_k^2 \means{\vc{v}}{v_k^4} + \sum_{l\neq k} g_l^2 \means{\vc{v}}{v_l^2}\means{\vc{v}}{v_k^2} \\
&= g_k^2 \cdot 3 + \sum_{l\neq k} g_l^2 \cdot 1 \\
&= 3g_k^2 + \sum_{l\neq k} g_l^2 \\
&= \sum_{l=1}^d g_l^2 + 2g_k^2 \\
&= \|\vc{g}\|^2 + 2g_k^2.
\end{align*}

For the second term, with $\vc{v}\coloneq \vc{u}_i$, $\vc{w}\coloneq \vc{u}_j, i\neq j$:

\begin{align}
\means{\vc{v},\vc{w}}{(\tp{\vc{g}}\vc{v})v_k(\tp{\vc{g}}\vc{w})w_k}
&= \means{\vc{v}}{(\tp{\vc{g}}\vc{v})v_k} \cdot \means{\vc{w}}{(\tp{\vc{g}}\vc{w})w_k} \quad (\vc{v} \perp \vc{w}) \\
&= \Big(\sum_{a=1}^d g_a \means{\vc{v}}{v_a v_k}\Big) \cdot \Big(\sum_{b=1}^d g_b \means{\vc{w}}{w_b w_k}\Big) \\
&= \Big(\sum_{a=1}^d g_a \delta_{a,k}\Big) \cdot \Big(\sum_{b=1}^d g_b \delta_{b,k}\Big) \\
&= g_k^2.
\end{align}
Plugging back into \eqref{eq:prop-expanded}:
\begin{align*}
\elem{\means{\vc{u}}{\hat{\vc{g}}^2_q(\vc{x})}}{k}
&\approx \frac{1}{q^2}\Big(q(\|\vc{g}\|^2 + 2g_k^2) + q(q-1)g_k^2\Big) \\
&= \frac{1}{q^2}\Big(q\|\vc{g}\|^2 + 2qg_k^2 + q(q-1)g_k^2\Big) \\
&= \frac{\|\vc{g}\|^2}{q} + \frac{2g_k^2}{q} + \frac{(q-1)g_k^2}{q} \\
&= \frac{\|\vc{g}\|^2}{q} + \frac{(q+1)g_k^2}{q} \\
&= \frac{1}{q}\big(\|\vc{g}\|^2 + g_k^2\big) + g_k^2.
\end{align*}
Since $\vc{g} = \nabla f(\vc{x})$, this matches the proposition in vector form:
\[
\boxed{\means{\vc{u}}{\hat{\nabla} f_{\varepsilon}^q(\vc{x})^2} \approx \frac{1}{q}\left( \norm{\nabla f(\vc{x})}^2 \ones_d + \nabla f(\vc{x})^2 \right) + \nabla f(\vc{x})^2.}
\]
\end{proof}
\subsection{Proof of Proposition~\ref{prop:zo-sq-uniform}}\label{app:proof-zo-sq-uniform}
We provide the proof of the following porpositio for the uniform perturbation case:
\begin{proposition}[Uniform]\label{prop:zo-sq-uniform}
Let \( f: \reals^d \to \reals \) be an \(L\)-smooth function, and let 
\(\hat{\nabla} f_{\varepsilon}^q(\vc{x})\) denote the \(q\)-sample ZO estimator in \eqref{eq:zo}, 
where \(\vc{u}_i \sim \text{Uniform}(\sphere)\) for all \( i \in [q] \defeq \{1,\ldots,q\}\). Then in the $\varepsilon \rightarrow 0$ limit we have:
\begin{equation}
    \means{\vc{u}}{\hat{\nabla} f_{\varepsilon}^q(\vc{x})^2} \approx \frac{d\left(\|\nabla f(\vc{x})\|^2\ones_d + 2\nabla f(\vc{x})^2\right)}{q(d+2)} + \frac{q-1}{q}\nabla f(\vc{x})^2
\end{equation}
where $\vc{(\cdot)}^2$ stands for the element-wise power of 2 and $\E_{\vc{u}}$ denotes the expectation over $\vc{u}_1, \dots, \vc{u}_q$.
\end{proposition}
\begin{proof}
Let $\vc{u}_i \sim \text{Uniform}(\sphere)$ be sampled from the unit sphere in $\reals^d$. By Lemma~\ref{lem:zo_proj_grad_upper_bound} and in the $\varepsilon \to 0$ limit, the $q$-sample ZO estimator becomes:
\begin{equation}{\label{eq:zo_grad_unit_sphere}}
    \hat{\nabla} f_{\varepsilon}^q(\vc{x}) \approx \frac{d}{q} \sum_{i=1}^q (\tp{\vc{g}}\vc{u}_i)\vc{u}_i,
\end{equation}
where $\vc{g} = \nabla f(\vc{x})$ for brevity.

Taking the expectation of the $k^{\text{th}}$ element squared, we get:
\begin{align}
\elem{\means{\vc{u}}{\hat{\nabla} f_{\varepsilon}^q(\vc{x})^2}}{k} &\approx \frac{d^2}{q^2} \means{\vc{u}}{\Big(\sum_{i=1}^q (\tp{\vc{g}}\vc{u}_i)\vc{u}_i\Big)^2}\\
&= \frac{d^2}{q^2}\Bigg(\sum_{i=1}^q \means{\vc{u}}{(\tp{\vc{g}}\vc{u}_i)^2 u_{i,k}^2} + \sum_{i\neq j} \means{\vc{u}}{(\tp{\vc{g}}\vc{u}_i)u_{i,k}(\tp{\vc{g}}\vc{u}_j)u_{j,k}}\Bigg).
\end{align}

For $\vc{v}\sim \text{Uniform}(\sphere)$, the moments are:
\[
\mathbb{E}[v_i^2] = \frac{1}{d}, \quad \mathbb{E}[v_i^4] = \frac{3}{d(d+2)}, \quad \mathbb{E}[v_i^2 v_j^2] = \frac{1}{d(d+2)}.
\]

Thus:
\[
\means{\vc{v}}{(\tp{\vc{g}}\vc{v})^2 v_k^2} = g_k^2 \cdot \frac{3}{d(d+2)} + \sum_{l\neq k} g_l^2 \cdot \frac{1}{d(d+2)} = \frac{\|\vc{g}\|^2 + 2g_k^2}{d(d+2)}.
\]

For two independent samples $\vc{v},\vc{w}$:
\[
\means{\vc{v},\vc{w}}{(\tp{\vc{g}}\vc{v})v_k(\tp{\vc{g}}\vc{w})w_k} = \Big(\sum_a g_a \mathbb{E}[v_a v_k]\Big)^2 = (g_k/d)^2 = \frac{g_k^2}{d^2}.
\]

Plugging back:
\[
\elem{\means{\vc{u}}{\hat{\vc{g}}^2_q(\vc{x})}}{k} \approx \frac{d^2}{q^2}\Big(q\cdot \frac{\|\vc{g}\|^2 + 2g_k^2}{d(d+2)} + q(q-1)\cdot \frac{g_k^2}{d^2}\Big).
\]

Simplify:
\[
= \frac{d(\|\vc{g}\|^2 + 2g_k^2)}{q(d+2)} + \frac{(q-1)g_k^2}{q}.
\]

And we obtain the vector form:
\[
\boxed{
\means{\vc{u}}{\hat{\nabla} f_{\varepsilon}^q(\vc{x})^2} \approx \frac{d\left(\|\nabla f(\vc{x})\|^2\ones_d + 2\nabla f(\vc{x})^2\right)}{q(d+2)} + \frac{q-1}{q}\nabla f(\vc{x})^2.
}
\]
\end{proof}

\clearpage

\subsection{Proof of Theorem~\ref{thm:smoothed-proxy-opt}}\label{app:smoothed-proxy-proof}
We re-state and prove Theorem~\ref{thm:smoothed-proxy-opt}, which provides a guarantee on an objective function $F$ while optimizing a smoothed proxy $F_\varepsilon$:

\begin{theorem}[Smoothed Proxy Optimization - Restated]\label{thm:smoothed-proxy-opt-restated}
Let $F:\reals^d \to \reals$ be an $L$-smooth function. For a centered smoothing distribution $\mathcal{P}$ with bounded support and satisfying $\means{\vc{u}\sim\mathcal{P}}{\norm{\vc{u}}^2}=C<\infty$, define the smoothed objective:
\[
F_\varepsilon(\vc{x}) \defeq \means{\vc{u}\sim\mathcal{P}}{F(\vc{x}+\varepsilon\vc{u})}, \quad \forall \vc{x} \in \reals^d ,
\]
where $\varepsilon>0$ is the smoothing parameter.  
Assume the optimal values
\[
F^* = \inf_{\vc{x}\in\reals^d} F(\vc{x}), \qquad
F_\varepsilon^* = \inf_{\vc{x}\in\reals^d} F_\varepsilon(\vc{x})
\]
are finite.  
Suppose an algorithm $\calA$ generates iterates $\{\vc{x}_t\}_{t=0}^{T-1}$ such that
\begin{equation}\label{eq:smoothed-guarantee-restated}
\frac{1}{T}\sum_{t=0}^{T-1}\mean{\norm{\nabla F_\varepsilon(\vc{x}_t)}^2}
\leq K_0\big(F_\varepsilon(\vc{x}_0)-F_\varepsilon^*\big)+K_1
\end{equation}
for some constants $K_0,K_1\geq 0$. Then, for the original function $F$, it holds that
\begin{equation}
\frac{1}{T}\sum_{t=0}^{T-1}\mean{\norm{\nabla F(\vc{x}_t)}^2}
\leq K_0\big(F(\vc{x}_0)-F^*\big)+K_1+\varepsilon^2L^2C\Big(\tfrac{K_0}{2}+2\Big).
\end{equation}
\end{theorem}
\begin{proof}
    Let $\vc{x}, \vc{u} \in \reals^d$ and $\varepsilon > 0$.
    First, observe that $F(\vc{x}+\varepsilon\vc{u}) \ge F^*$ since $F^* = \inf_{\vc{z}} F(\vc{z})$.
    Taking expectation over $\vc{u}\sim\mathcal{P}$ gives $F_\varepsilon(\vc{x}) = \mean{F(\vc{x}+\varepsilon\vc{u})} \ge F^*$.
    Consequently, the smoothed optimum satisfies
    \begin{equation}\label{eq:smoothed-optimum}
        F_\varepsilon^* = \inf_{\vc{x}} F_\varepsilon(\vc{x}) \ge F^*.   
    \end{equation}
    By Lemma~\ref{lemma:gradient-and-smoothed-gradient-under-smoothness} (due to smoothness of both $F$ and $F_\varepsilon$), we also have $\forall \vc{x}\in\reals^d$:
    \begin{equation}\label{eq:norm-grad-bound}
        \norm{\nabla F(\vc{x})}^2 \leq (\norm{\nabla F_\varepsilon(\vc{x})} + \varepsilon L \sqrt{C})^2 \leq 2\norm{\nabla F_\varepsilon(\vc{x})}^2+2\varepsilon^2L^2C.
    \end{equation}
    By Lemma~\ref{lemma:func-and-smoothed-func-under-smoothness} we also have
    \begin{align}\label{eq:smoothing-function-diff}
        \left|F(\vc{x})-F_\varepsilon(\vc{x})\right| \leq \frac{\varepsilon^2}{2} LC .
    \end{align}
    Finally, we get:
    \begin{align*}
        \frac{1}{T}\sum_{t=0}^{T-1}\mean{\norm{\nabla F(\vc{x}_t)}^2} &\overset{\text{\eqref{eq:norm-grad-bound}}}{\leq} \frac{1}{T}\sum_{t=0}^{T-1}\mean{\norm{\nabla F_\varepsilon(\vc{x}_t)}^2} +2\varepsilon^2L^2C \\
        &\overset{\text{\eqref{eq:smoothed-guarantee-restated}}}{\leq} K_0\big(F_\varepsilon(\vc{x}_0)-F_\varepsilon^*\big)+K_1+ 2\varepsilon^2L^2C \\
        &\overset{\text{\eqref{eq:smoothing-function-diff}}}{\leq}K_0\big(F(\vc{x}_0)-F_\varepsilon^*\big)+K_1+\varepsilon^2L^2C\Big(\tfrac{K_0}{2}+2\Big) \\
        &\overset{\text{\eqref{eq:smoothed-optimum}}}{\leq} K_0\big(F(\vc{x}_0)-F^*\big)+K_1+\varepsilon^2L^2C\Big(\tfrac{K_0}{2}+2\Big) .
    \end{align*}
\end{proof}

\subsection{Proof of Theorem~\ref{thm:adaptive-sgd-affine-meazo}}\label{app:adaptive-sgd-affine-meazo-proof}
We first present a useful lemma that provides an upper bound on the expected squared norm of the zeroth-order (ZO) gradient, where the expectation is taken over the perturbation randomness:

\begin{lemma}[ZO Estimator - Expected Squared Norm Bound]\label{lemma:zo-var-bound-uniform}
Let $\hat{\nabla} f_{\varepsilon}^q(\vc{x})$ be the zeroth-order (ZO) gradient estimator defined in \eqref{eq:zo-uniform}. 
If $f:\reals^d \to \reals$ is differentiable and $L$-smooth, then: 
\begin{align}
\means{\calU_q \sim (\sphere)^q}{\norm{\hat{\nabla} f_{\varepsilon}^q(\vc{x})}^2} 
    &\leq  \frac{1}{q}\left(2d\norm{\nabla f(\vc{x})}^2 +  \frac{\varepsilon^2 d^2}{2} L^2\right)
       + \left(1-\frac{1}{q}\right)\norm{\nabla f_\varepsilon(\vc{x})}^2
\end{align}
where $\varepsilon > 0$, $f_\varepsilon$ is defined in~\eqref{eq:smoothed_funct_ball} and $\calU_q \defeq \{\vc{u}_1, \ldots, \vc{u}_q\}$.   
\end{lemma}
\begin{proof}

First, note that $\hat{\nabla} f_{\varepsilon}^q(\vc{x})$ is an average of $q$ i.i.d. estimators:
    \[
        \hat{\nabla} f_{\varepsilon}^q(\vc{x}) \overset{\eqref{eq:zo-uniform}}{=} \frac{d}{q} \sum_{i=1}^q \frac{f(\vc{x}+\varepsilon \vc{u}_{i}) - f(\vc{x}-\varepsilon \vc{u}_{i})}{2\varepsilon} \vc{u}_{i} \eqdef \frac{1}{q} \sum_{i=1}^q \hat{\vc{g}}(\vc{x}, \vc{u}_i) .
    \]
By independence of the perturbations, we get:
    \begin{align}
    \means{\calU_q \sim (\sphere)^q}{\Big\|\frac{1}{q}\sum_{i=1}^q \hat{\vc{g}}(\vc{x}, \vc{u}_i)\Big\|^2}
    &= \frac{1}{q^2}\means{\calU_q}{\sum_{i=1}^q \|\hat{\vc{g}}(\vc{x}, \vc{u}_i)\|^2
        + \sum_{i\neq j} \langle \hat{\vc{g}}(\vc{x}, \vc{u}_i), \hat{\vc{g}}(\vc{x}, \vc{u}_j)\rangle} \nonumber \\
    \overset{\vc{u}_i \text{'s i.i.d.}}&{=} \frac{1}{q^2}\Bigg(
        q\means{\vc{u} \sim \sphere}{\|\hat{\vc{g}}(\vc{x}, \vc{u})\|^2}
        + q(q-1)\Big\langle \means{\vc{u} \sim \sphere}{\hat{\vc{g}}(\vc{x}, \vc{u})},
                                  \means{\vc{u} \sim \sphere}{\hat{\vc{g}}(\vc{x}, \vc{u})} \Big\rangle
    \Bigg) \nonumber \\
    &= \frac{1}{q}\means{\vc{u}}{\|\hat{\vc{g}}(\vc{x}, \vc{u})\|^2}
       + \Big(1-\frac{1}{q}\Big)\big\|\means{\vc{u}}{\hat{\vc{g}}(\vc{x}, \vc{u})}\big\|^2 \nonumber \\
    \overset{\text{Lemma \ref{lem:unbiased_zo_grad_sphere}}}&{=} \frac{1}{q}\means{\vc{u}}{\|\hat{\vc{g}}(\vc{x}, \vc{u})\|^2}
       + \left(1-\frac{1}{q}\right)\norm{\nabla f_\varepsilon(\vc{x})}^2 .
    \end{align}  

Taking the expected norm squared of a single estimator:
\begin{align}
    \means{\vc{u} \sim \sphere}{\|\hat{\vc{g}}(\vc{x}, \vc{u})\|^2} &= \means{\vc{u}}{\norm{d\frac{f(\vc{x}+\varepsilon\vc{u}) -f(\vc{x}-\varepsilon\vc{u})}{2\varepsilon}\vc{u}}^2} \\
    &= d^2\means{\vc{u}}{\left| \frac{f(\vc{x}+\varepsilon\vc{u}) -f(\vc{x}-\varepsilon\vc{u})}{2\varepsilon} \right|^2\norm{\vc{u}}^2} \\
    % &\leq d^2 \means{\vc{u}}{\left(\left |\left\langle \nabla f(\vc{x}),\vc{u}\right\rangle \right| + \left|\frac{\varepsilon}{2} L\norm{\vc{u}}^2\right|\right)^2\norm{\vc{u}}^2} \\
    \overset{\eqref{eq:zo_proj_grad_abs_upper_bound}}&{\leq} d^2 \means{\vc{u}}{\left(\left |\left\langle \nabla f(\vc{x}),\vc{u}\right\rangle \right| + \frac{\varepsilon}{2} L\norm{\vc{u}}^2 \right)^2\norm{\vc{u}}^2} \\
    \overset{\vc{u}\in \sphere}&{=}d^2\means{\vc{u}}{\left(\left |\left\langle \nabla f(\vc{x}),\vc{u}\right\rangle \right| + \frac{\varepsilon}{2} L\right)^2} \\
    % &= d^2\means{\vc{u}}{\left \langle \nabla f(\vc{x}),\vc{u}\right\rangle^2 + \varepsilon L \left|\left\langle \nabla f(\vc{x}),\vc{u}\right\rangle \right|+ \frac{\varepsilon^2}{4} L^2} \\
    % &\leq d^2\means{\vc{u}}{2\left \langle \nabla f(\vc{x}),\vc{u}\right\rangle^2 + \frac{5\varepsilon^2}{4} L^2}\\
    % &\leq 2d^2\means{\vc{u}}{\norm{\nabla f(\vc{x})}^2\norm{\vc{u}}^2 } +  \frac{5\varepsilon^2d^2}{4} L^2 \\
    % &= 2d^2\norm{\nabla f(\vc{x})}^2 +  \frac{5\varepsilon^2d^2}{4} L^2 
    \overset{(a+b)^2 \leq 2(a^2+b^2)}&{\leq} d^2\means{\vc{u}}{2\left \langle \nabla f(\vc{x}),\vc{u}\right\rangle^2 + \frac{\varepsilon^2}{2} L^2}\\
    &= 2d^2 \tp{\nabla f(\vc{x})}\means{\vc{u}}{\vc{u}\tp{\vc{u}}}\nabla f(\vc{x})+  \frac{\varepsilon^2d^2}{2} L^2 \\
    &= 2d\norm{\nabla f(\vc{x})}^2 +  \frac{\varepsilon^2d^2}{2} L^2 .
\end{align}
Therefore, we get:
\[
\boxed{
    \means{\calU_q \sim (\sphere)^q}{\norm{\hat{\nabla} f_{\varepsilon}^q(\vc{x})}^2} 
    \leq  \frac{1}{q}\left(2d\norm{\nabla f(\vc{x})}^2 +  \frac{\varepsilon^2 d^2}{2} L^2\right)
       + \left(1-\frac{1}{q}\right)\norm{\nabla f_\varepsilon(\vc{x})}^2.
}
\]
\end{proof}

We also state the following result on the total variance of the norm of the ZO estimator:
\begin{lemma}[Affine Variance Bound and Unbiasedness for ZO Gradient Estimator]\label{lemma:affine-var-bound}
Let $F: \reals^d \to \reals$ be the population loss defined as
\[
F(\vc{x}) = \means{\xi \sim \mathcal{D}}{f(\vc{x}; \xi)} \quad \forall \vc{x} \in \reals^d,
\]
where each individual sample loss $f(. ; \xi)$ is $L$-smooth. Let $f_\varepsilon$ and $F_\varepsilon$ be their smoothed counterparts according to~\eqref{eq:smoothed_funct_ball}.
Assume that, there exists $\sigma > 0$ such that, for all $\vc{x} \in \reals^d$, the variance of the sample loss gradient satisfies
\begin{equation}\label{eq:var_bound_sto_grad}
    \means{\xi}{\norm{\nabla f(\vc{x}; \xi) - \nabla F(\vc{x})}^2} \le \sigma^2 ,
\end{equation}
Let $q \in \naturals^*$, $\varepsilon > 0$ and let $\hat{\nabla} f_{\varepsilon}^q(\vc{x}; \xi)$ be the zeroth-order (ZO) gradient estimator defined in~\eqref{eq:zo-uniform} for uniform perturbations.
Then, for any $\vc{x} \in \reals^d$, we have:
\begin{align}
    \means{\xi,\calU_q}{\hat{\nabla} f_{\varepsilon}^q(\vc{x}; \xi)} &= \nabla F_\varepsilon(\vc{x}),
    \\
    \means{\xi,\,\calU_q}{\norm{\hat{\nabla} f_{\varepsilon}^q(\vc{x}; \xi) - \nabla F_\varepsilon(\vc{x})}^2}
    &\leq
    \frac{d\varepsilon^2 L^2}{2q}(8+d)
    + \left(\frac{2d - 1}{q}+1\right)\sigma^2
    + \frac{4d-1}{q}\norm{\nabla F_\varepsilon(\vc{x})}^2.
\end{align}
\end{lemma}
\begin{proof}
The unbiasedness of \eqref{eq:zo-uniform} is already established due to Lemma~\ref{lem:unbiased_zo_grad_sphere}:
\begin{equation}\label{eq:unbiasedness_zo_sto_grad}
    \means{\xi,\calU_q}{\hat{\nabla} f_{\varepsilon}^q(\vc{x}; \xi)} =    \means{\xi}{\means{\calU_q}{\hat{\nabla} f_{\varepsilon}^q(\vc{x}; \xi) \mid \xi}} =  \means{\xi}{\nabla f_\varepsilon(\vc{x};\xi)} = \nabla F_\varepsilon(\vc{x}).
\end{equation}
 
We recall that the variance of a random vector $X$ is given by $\var{\norm{X}} = \mean{\norm{X - \mean{X}}^2} = \mean{\norm{X}^2} - \norm{\mean{X}}^2$.
We aim at computing the variance of our ZO stochastic gradient estimator:
\begin{equation}\label{eq:var_zo_grad_decomposition}
    \means{\xi,\calU_q}{\norm{\hat{\nabla} f_{\varepsilon}^q(\vc{x}; \xi) - \nabla F_\varepsilon(\vc{x})}^2} \overset{\eqref{eq:unbiasedness_zo_sto_grad}}{=} \means{\xi,\calU_q}{\norm{\hat{\nabla} f_{\varepsilon}^q(\vc{x}; \xi)}^2} - \norm{\nabla F_\varepsilon(\vc{x})}^2 .
\end{equation}
Then, we focus on upper-bounding the first term of the variance decomposition
\begin{align}
    \means{\xi,\calU_q}{\norm{\hat{\nabla} f_{\varepsilon}^q(\vc{x}; \xi)}^2} \overset{\text{Lemma~\ref{lemma:zo-var-bound-uniform}}}&{\leq} \means{\xi}{ \frac{1}{q}\left(2d\norm{\nabla f(\vc{x}; \xi)}^2 +  \frac{\varepsilon^2 d^2}{2} L^2\right)
    + \left(1-\frac{1}{q}\right)\norm{\nabla f_\varepsilon(\vc{x}; \xi)}^2 } \nonumber \\
    \overset{\eqref{eq:var_bound_sto_grad}}&{\leq} \frac{2d}{q} \left( \norm{\nabla F(\vc{x})}^2 + \sigma^2 \right) 
    + \frac{\varepsilon^2}{2q} d^2 L^2
    + \left(1-\frac{1}{q}\right) \means{\xi}{ \norm{\nabla f_\varepsilon(\vc{x}; \xi)}^2 } \nonumber \\
    \overset{\text{Lemma~\ref{lemma:sto_grad_smoothed_var_bound}}}&{\leq}
    \frac{2d + q - 1}{q} \sigma^2 
    + \frac{\varepsilon^2}{2q} d^2 L^2
    + \frac{2d}{q} \norm{\nabla F(\vc{x})}^2
    + \left(1-\frac{1}{q}\right) \norm{\nabla F_\varepsilon(\vc{x})}^2 \nonumber \\
    \overset{\text{Lemma~\ref{lemma:gradient-and-smoothed-gradient-under-smoothness}}}&{\leq} 
    \left(\frac{2d -1}{q}+1\right) \sigma^2 
    + \left( d + 8 \right) \frac{d}{2q} \varepsilon^2 L^2
    + \left( \frac{4d - 1}{q} + 1 \right) \norm{\nabla F_\varepsilon(\vc{x})}^2 \label{eq:intermediate_zo_grad_squared_norm_bound}
\end{align}
Combining~\eqref{eq:var_zo_grad_decomposition} and~\eqref{eq:intermediate_zo_grad_squared_norm_bound}, we obtain
\begin{equation*}
\boxed{
    \means{\xi,\calU_q}{\norm{\hat{\nabla} f_{\varepsilon}^q(\vc{x}; \xi) - \nabla F_\varepsilon(\vc{x})}^2} \leq 
     \left(\frac{2d -1}{q}+1\right) \sigma^2 
    + \left( d + 8 \right) \frac{d}{2q} \varepsilon^2 L^2
    + \frac{4d - 1}{q} \norm{\nabla F_\varepsilon(\vc{x})}^2 .
}
\end{equation*}

\end{proof}

We can now prove Theorem~\ref{thm:adaptive-sgd-affine-meazo} (re-stated below):

\begin{theorem}[Adaptive SGD with Affine Variance Bound]\label{thm:adaptive-sgd-affine-meazo-restated}
Let $F:\reals^d\to\reals$ be the population loss
\[
F(\vc{x}) \;=\; \means{\xi\sim\mathcal{D}}{f(\vc{x};\xi)},
\]
where each sample loss $f(\vc{x};\xi)$ is $L$-smooth and $G$-Lipschitz (equivalently, $\norm{\nabla f(\vc{x};\xi)}\le G$ for all $\vc{x},\xi$). Assume $F$ is bounded from below by $F^*$, i.e., $F^*=\inf_{\vc{x}}F(\vc{x})>-\infty$.

Suppose we have an unbiased stochastic gradient estimator $g(\vc{x};\xi)$, \textit{i.e.,} $\means{\xi\sim\mathcal{D}}{g(\vc{x};\xi)} = \nabla F(\vc{x})$, satisfying the affine variance bound
\[
\means{\xi\sim\mathcal{D}}{\norm{g(\vc{x};\xi)-\nabla F(\vc{x})}^2}
\;\le\; \sigma_0^2+\sigma_1^2\,\norm{\nabla F(\vc{x})}^2.
\]
Consider the adaptive SGD update
\[
\vc{x}_{t+1} \;=\; \vc{x}_t \;-\; \frac{\eta}{\sqrt{v_t}+\zeta}\,g(\vc{x}_t;\xi_t),
\]
with stepsize $\eta>0$, stability constant $\zeta>0$, and second-moment tracker
\[
v_t \;=\; \beta\,v_{t-1} \;+\; (1-\beta)\,\gamma_t^2, \qquad 0<\beta<1,
\]
where $\{\gamma_t\}_{t\ge0}$ is predictable (measurable w.r.t.\ the filtration generated by $\{\vc{x}_t,\xi_t\}$), satisfies $0\le \gamma_t\le G$, and obeys the condition
\begin{equation}\label{eq:gamma-assumption-restated}
\means{\xi_t\sim\mathcal{D}}{\gamma_t\,\norm{g(\vc{x}_t;\xi_t)} \,\big|\, \vc{x}_t}
\;\le\; \sigma_0^2 \;+\; (\sigma_1^2+1)\,\norm{\nabla F(\vc{x}_t)}^2.
\end{equation}
Assume the parameters $(\beta,\eta,\zeta)$ are chosen so that
\begin{equation}\label{eq:meazo-norm-constants-restated}
\max\!\left\{
\frac{G(1+\sigma_1^2)\sqrt{1-\beta}}{\zeta},\;
\frac{L\,\eta}{2\,\zeta}
\right\}
\;\le\; \frac{1}{4}.
\end{equation}
Then, after $T$ iterations,
\begin{equation}
\frac{1}{T}\sum_{t=0}^{T-1}\mean{\norm{\nabla F(\vc{x}_t)}^2}
\;\le\;
2\big(\sqrt{\beta}\,G+\zeta\big)\left[
\frac{F(\vc{x}_0)-F^*}{\eta\,T} \;+\; \frac{\sigma_0^2}{2\,\zeta}
\right].
\end{equation}
\end{theorem}
\begin{proof}
Define $\vc{g}_t \defeq g(\vc{x}_t;\xi_t)$, and consider the descent lemma at iteration $t$:
\begin{equation}
    F(\vc{x}_{t+1}) \leq F(\vc{x}_t) -\frac{\eta}{\sqrt{v_t}+\zeta} \left\langle \nabla F(\vc{x}_t),  \vc{g}_t\right\rangle + \frac{L\eta^2}{2(\sqrt{v_t}+\zeta)^2}\norm{\vc{g}_t}^2
\end{equation}
Taking the expectation conditioned on $\vc{x}_t$:
\begin{align}
     \means{\xi_t}{F(\vc{x}_{t+1}) | \vc{x}_t} &\leq F(\vc{x}_t) \underbrace{-\eta\means{\xi_t}{\frac{\left\langle \nabla F(\vc{x}_t),  \vc{g}_t\right\rangle}{\sqrt{v_t}+\zeta} \mid \vc{x}_t}}_{\circled{A}} + \frac{L \eta^2}{2}\underbrace{\means{\xi_t}{\frac{\norm{\vc{g}_t}^2}{(\sqrt{v_t}+\zeta)^2}\mid \vc{x}_t}}_{\circled{B}}
\end{align}
Due to the dependence of $v_t$ on $\vc{g}_t$, evaluating \circled{A} becomes tricky (compared to the non-adaptive SGD setting). We use a common trick with adaptive methods \cite{zaheer2018adaptive} of introducing a conditionally independent quantity $\hat{v}_t$ given $\vc{x}_t$:
\begin{align}
    \circled{A} &=-\eta \mean{\frac{1}{\sqrt{v_t}+\zeta}\left\langle \nabla F(\vc{x}_t), \vc{g}_t\right\rangle \mid \vc{x}_t} \nonumber \\
    &=-\eta \mean{\left(\frac{1}{\sqrt{v_t}+\zeta} - \frac{1}{\sqrt{\hat{v}_t}+\zeta}+ \frac{1}{\sqrt{\hat{v}_t}+\zeta}\right)\left\langle \nabla F(\vc{x}_t),\vc{g}_t\right\rangle \mid \vc{x}_t} \nonumber \\
    &= -\frac{\eta}{\sqrt{\hat{v}_t}+\zeta} \mean{\left\langle \nabla F(\vc{x}_t), \vc{g}_t\right\rangle \mid \vc{x}_t} + \eta \mean{\left(\frac{1}{\sqrt{\hat{v}_t}+\zeta}- \frac{1}{\sqrt{v_t}+\zeta}\right)\left\langle \nabla F(\vc{x}_t),\vc{g}_t\right\rangle \mid \vc{x}_t} \nonumber \\
    &= -\frac{\eta}{\sqrt{\hat{v}_t}+\zeta} \norm{\nabla F(\vc{x}_t)}^2 + \eta \mean{\underbrace{\left(\frac{1}{\sqrt{\hat{v}_t}+\zeta}- \frac{1}{\sqrt{v_t}+\zeta}\right)\left\langle \nabla F(\vc{x}_t),\vc{g}_t\right\rangle}_{\circled{C}} \mid \vc{x}_t} ,
\end{align}
where we used in the last equation the unbiasedness of the stochastic gradient estimator $\vc{g}_t = g(\vc{x}_t;\xi_t)$.
The first term is negative, so we need to only upper bound the second term \circled{C}. Using $\hat{v}_t=\beta v_{t-1}$, we get:

\begin{align}
   \circled{C} &=\left(\frac{\sqrt{v_t} -\sqrt{\hat{v}_t}}{(\sqrt{\hat{v}_t}+\zeta)(\sqrt{v_t}+\zeta)}\right)\left\langle \nabla F(\vc{x}_t), \vc{g}_t\right\rangle \\
   &= \left(\frac{v_t -\hat{v}_t}{(\sqrt{\hat{v}_t}+\zeta)(\sqrt{v_t}+\zeta)(\sqrt{v_t} +\sqrt{\hat{v}_t})}\right)\left\langle \nabla F(\vc{x}_t), \vc{g}_t\right\rangle \\
   &= \left(\frac{(1-\beta)\gamma_t^2}{(\sqrt{\beta v_{t-1}}+\zeta)(\sqrt{v_t}+\zeta)(\sqrt{\beta v_{t-1}+(1-\beta)\gamma_t^2} +\sqrt{\beta v_{t-1}})}\right)\left\langle \nabla F(\vc{x}_t), \vc{g}_t\right\rangle\\
   &\overset{\text{(a)}}{\leq}\left(\frac{(1-\beta)\gamma_t^2}{(\sqrt{\beta v_{t-1}}+\zeta)(\sqrt{v_t}+\zeta)(\sqrt{(1-\beta)\gamma_t^2})}\right)\left\langle \nabla F(\vc{x}_t), \vc{g}_t\right\rangle\\ 
   &= \left(\frac{\sqrt{(1-\beta)}\gamma_t}{(\sqrt{\beta v_{t-1}}+\zeta)(\sqrt{v_t}+\zeta)}\right)\left\langle \nabla F(\vc{x}_t), \vc{g}_t\right\rangle\\
   &\overset{\text{(b)}}{\leq}\left(\frac{\sqrt{(1-\beta)}\gamma_t}{(\sqrt{\beta v_{t-1}}+\zeta)\zeta}\right) G \norm{\vc{g}_t}
\end{align}
where (a) is due to dropping positive terms in the denominator and (b) follows from the Cauchy–Schwarz inequality and the global Lipschitz property.

Plugging back we get:
\begin{align}
    \circled{A} &= -\frac{\eta}{\sqrt{\beta v_{t-1}}+\zeta} \norm{\nabla F(\vc{x}_t)}^2 + \eta \mean{\circled{C} \mid \vc{x}_t} \\
    &\leq -\frac{\eta}{\sqrt{\beta v_{t-1}}+\zeta} \norm{\nabla F(\vc{x}_t)}^2 + \eta \mean{\left(\frac{\sqrt{(1-\beta)}\gamma_t}{(\sqrt{\beta v_{t-1}}+\zeta)\zeta}\right) G \norm{\vc{g}_t} \mid \vc{x}_t} \\
    &\overset{\text{\eqref{eq:gamma-assumption-restated}}}{\leq} -\frac{\eta}{\sqrt{\beta v_{t-1}}+\zeta} \norm{\nabla F(\vc{x}_t)}^2 + \eta \left(\frac{\sqrt{(1-\beta)}}{(\sqrt{\beta v_{t-1}}+\zeta)\zeta}\right) G \left(\sigma^2_0 + (\sigma^2_1+1)\norm{\nabla F(\vc{x}_t)}^2\right) \\
    &= \frac{\eta}{\sqrt{\beta v_{t-1}}+\zeta}\left(\left(\frac{G(\sigma_1^2+1)\sqrt{1-\beta}}{\zeta}-1\right)\norm{\nabla F(\vc{x}_t)}^2 + \frac{G\sigma_0^2\sqrt{1-\beta}}{\zeta}\right)
\end{align}

As for the \circled{B} term, we use the fact $v_t \geq\beta v_{t-1}\geq 0$ and the affine variance bound to get:
\begin{align}
    \means{\xi_t}{\frac{\norm{\vc{g}_t}^2}{(\sqrt{v_t}+\zeta)^2}\mid \vc{x}_t} &\leq \means{\xi_t}{\frac{\norm{\vc{g}_t}^2}{(\sqrt{\beta v_{t-1}}+\zeta)^2}\mid \vc{x}_t} \\
    &\leq\frac{1}{(\sqrt{\beta v_{t-1}}+\zeta)\zeta}\means{\xi_t}{\norm{\vc{g}_t}^2\mid \vc{x}_t} \\
    &\leq\frac{\sigma_0^2 + (1+\sigma_1^2)\norm{\nabla F(\vc{x})}^2}{(\sqrt{\beta v_{t-1}}+\zeta)\zeta}
\end{align}

Plugging \circled{A} and \circled{B}:

\begin{align}
     \means{\xi_t}{F(\vc{x}_{t+1}) | \vc{x}_t} &\leq F(\vc{x}_t) +\frac{\eta}{\sqrt{\beta v_{t-1}}+\zeta}\left(\left(\frac{G(\sigma_1^2+1)\sqrt{1-\beta}}{\xi}-1\right)\norm{\nabla F(\vc{x}_t)}^2 + \frac{G\sigma_0^2\sqrt{1-\beta}}{\xi}\right)  \nonumber\\
&\quad\; +\frac{L \eta^2}{2}\frac{\sigma_0^2+ (1+\sigma_1^2)\norm{\nabla F(\vc{x})}^2}{(\sqrt{\beta v_{t-1}}+\zeta)\zeta} \\
&= F(\vc{x}_t) + \left(\frac{G(1+\sigma_1^2)\sqrt{1-\beta}}{\zeta}-1+\frac{L\eta}{2\zeta}\right)\frac{\eta\norm{\nabla F(\vc{x}_t)}^2}{\sqrt{\beta v_{t-1}}+\zeta} + \left(\frac{G\sqrt{1-\beta}}{\zeta}+\frac{L\eta}{2\zeta}\right)\frac{\eta\sigma_0^2}{\sqrt{\beta v_{t-1}}+\zeta} \\
&\overset{\text{\eqref{eq:meazo-norm-constants-restated}}}{\leq} F(\vc{x}_t) -\frac{\eta\norm{\nabla F(\vc{x}_t)}^2}{2\left(\sqrt{\beta v_{t-1}}+\zeta\right)} + \frac{\eta\sigma_0^2}{2\left(\sqrt{\beta v_{t-1}}+\zeta\right)} \\
&\overset{\gamma\leq G}{\leq}  F(\vc{x}_t) -\frac{\eta\norm{\nabla F(\vc{x}_t)}^2}{2(\sqrt{\beta}G+\zeta)} +\frac{\eta\sigma_0^2}{2\zeta}
\end{align}

Re-arranging, taking the full expectation, and telescoping yields:
\begin{align}
     \frac{\eta}{2\left(\sqrt{\beta}G+\zeta\right)} \sum_{t=0}^{T-1}\mean{\norm{\nabla F(\vc{x}_t)}^2} &\leq  F(\vc{x}_0) -\mean{F(\vc{x}_{T})} + T\frac{\eta\sigma_0^2}{2\zeta}
\end{align}
Which yields the desired result:
\[
\boxed{
     \frac{1}{T} \sum_{t=0}^{T-1}\mean{\norm{\nabla F(\vc{x}_t)}^2} \leq  2\left(\sqrt{\beta}G+\zeta\right)\left[\frac{F(\vc{x}_0) - F^*}{\eta T}  + \frac{\sigma_0^2}{2\zeta}\right]
}
\]
\end{proof}

\subsection{Proof of Theorem~\ref{thm:meazo}}\label{app:meazo-proof}

Using Theorem~\ref{thm:adaptive-sgd-affine-meazo}, we can now show the main result (restated below):
\begin{theorem}[MEAZO - Restated]\label{thm:meazo-restated}
Under Assumptions~\ref{assumption:smoothness}, \ref{assumption:bounded-var}, and \ref{assumption:bounded-grad}, 
define
\[
    \sigma_0^2 = \frac{d\varepsilon^2L^2}{2q}(8+d) + \left(\frac{2d-1}{q} +1\right)\sigma^2, 
    \qquad
    \sigma_1^2 = \frac{4d-1}{q}.
\]

If the parameters $\beta,\eta,\zeta$ satisfy
\begin{equation}
    \max\!\left\{
        \frac{G(1+\sigma_1^2)\sqrt{1-\beta}}{\zeta},\,
        \frac{L\eta}{2\zeta}
    \right\} \le \frac{1}{4},
\end{equation}
then after $T$ iterations of Algorithm~\ref{alg:meazo} with $\vc{u}\sim\text{Uniform}(\sphere)$, we have
\begin{align*}
\frac{1}{T}\sum_{t=0}^{T-1}\mean{\norm{\nabla F(\vc{x}_t)}^2}
&\le 
2\alpha\!\left[
    \frac{F(\vc{x}_0)-F^*}{\eta T} + \frac{\sigma_0^2}{2\zeta}
\right]
+ \varepsilon^2L^2\left(\frac{\alpha}{\eta T}+2\right),
\end{align*}
where $\alpha = \sqrt{\beta}G + \zeta$.
\end{theorem}
\begin{proof}
To prove the main result, we need to show that the conditions in Theorem~\ref{thm:adaptive-sgd-affine-meazo} are met, and then apply Theorem~\ref{thm:smoothed-proxy-opt} to map the guarantee back to $F$.

By Lemma~\ref{lemma:smoothing-preserves-smoothness}, both $f_\varepsilon(\cdot;\xi)$ and $F_\varepsilon$ are $L$-smooth. Likewise, by Lemma~\ref{lemma:smoothing-preserves-lipschitzness}, both $f_\varepsilon(\cdot;\xi)$ and $F_\varepsilon$ are $G$-Lipschitz.
For the ZO estimator in \eqref{eq:zo-uniform}, we have the smoothing distribution to be the unit ball $\vc{v}\sim\text{Uniform}(\ball)$, which implies $C=\frac{d}{d+2}$.

By Lemma~\ref{lemma:affine-var-bound} we have
 \begin{equation}
\means{\xi,\calU_q}{\hat{\nabla} f_{\varepsilon}^q(\vc{x}; \xi)} =  \nabla F_\varepsilon(\vc{x}),
 \end{equation}
and 
\begin{align}
    \means{\xi,\calU_q}{\norm{\hat{\nabla} f_{\varepsilon}^q(\vc{x}; \xi) - \nabla F_\varepsilon(\vc{x})}^2} \leq \underbrace{\frac{d\varepsilon^2L^2}{2q} \left(8+d\right) + \sigma^2\left(\frac{2d-1}{q}+1\right)}_{\sigma_0^2}+ \underbrace{\frac{4d-1}{q}}_{\sigma_1^2}\norm{\nabla F_\varepsilon(\vc{x})}^2.
\end{align} 
Let $\gamma_t$ be the tracked statistic in MEAZO:
\begin{equation}
    \gamma_t = \left| \frac{1}{q}\sum_{i=1}^q\frac{f(\vc{x}_t+\varepsilon \vc{u}_i;\xi_t)-f(\vc{x}_t-\varepsilon \vc{u}_i;\xi_t)}{2\varepsilon}\right|
\end{equation}
We have:
\begin{align}
    0 \leq \gamma_t 
    &\leq  \frac{1}{q}\sum_{i=1}^q\left|\frac{f(\vc{x}_t+\varepsilon \vc{u}_i;\xi_t)-f(\vc{x}_t-\varepsilon \vc{u}_i;\xi_t)}{2\varepsilon}\right|\\
    &\overset{\text{$G$-Lipschitz}}{\leq} \frac{1}{q}\sum_{i=1}^q \left(\frac{G\norm{2\varepsilon \vc{u}_i}}{2\varepsilon}\right)\\
    &\overset{\text{$\vc{u}_i\in\sphere$}}{=} \frac{1}{q}\sum_{i=1}^q G\\
    &=G
\end{align}
Thus, we check for the condition:
\begin{equation}
    \means{\xi_t,\mathcal{U}_q}{\gamma_t \norm{\hat{\nabla} f_{\varepsilon}^q(\vc{x}_t;\xi_t)} \mid \vc{x}_t} \overset{\text{Holder's}}{\leq}\sqrt{\underbrace{\means{\xi_t,\mathcal{U}_q}{\gamma_t^2 \mid \vc{x}_t}}_{\circled{A}}}\sqrt{\underbrace{\means{\xi_t,\mathcal{U}_q}{ \norm{\hat{\nabla} f_{\varepsilon}^q(\vc{x}_t;\xi_t)}^2 \mid \vc{x}_t}}_{\circled{B}}}
\end{equation}

Conditioning on $\vc{x}_t$, we first upper bound \circled{A}:
\begin{align}
    \circled{A} &= \means{\xi_t,\mathcal{U}_q}{\left(\frac{1}{q}\sum_{i=1}^q \frac{f(\vc{x}_t+\varepsilon \vc{u}_i;\xi_t)-f(\vc{x}_t-\varepsilon \vc{u}_i;\xi_t)}{2\varepsilon}\right)^2 }\\
    &\overset{\text{variance of mean of i.i.d. rule}}{=}\frac{1}{q}\means{\xi_t,\vc{u}}{\left(\frac{f(\vc{x}_t+\varepsilon \vc{u};\xi_t)-f(\vc{x}_t-\varepsilon \vc{u};\xi_t)}{2\varepsilon}\right)^2 }\\
    &\overset{\text{Lemma~\ref{lem:zo_proj_grad_upper_bound}}}{\leq} \frac{1}{q}\means{\xi_t,\vc{u}}{\left(\left| \left\langle \nabla f(\vc{x}_t;\xi_t),\vc{u}\right\rangle\right| + \frac{\varepsilon L}{2}\right)^2}\\
    &\overset{(a+b)^2\leq2a^2+2b^2}{\leq}  \frac{2}{q}\means{\xi_t,\vc{u}}{ \left\langle \nabla f(\vc{x}_t;\xi_t),\vc{u}\right\rangle^2 + \frac{\varepsilon^2 L^2}{4}}\\
    &\overset{d\geq 1}{\leq}\frac{2d^2}{q}\means{\xi_t,\vc{u}}{\left\langle \nabla f(\vc{x}_t;\xi_t),\vc{u}\right\rangle^2} + \frac{\varepsilon^2d^2 L^2}{2q}\\
    &=\frac{2d^2}{q}\means{\xi_t}{\means{\vc{u}\sim \text{Uniform}(\sphere)}{\left\langle \nabla f(\vc{x}_t;\xi_t),\vc{u}\right\rangle^2\mid \xi}} + \frac{\varepsilon^2d^2 L^2}{2q}\\
    &=\frac{2d}{q}\means{\xi_t}{\norm{\nabla f(\vc{x}_t;\xi_t)}^2} + \frac{\varepsilon^2 d^2 L^2}{2q}\\
    &\overset{\text{Variance Bound}}{\leq} \frac{2d}{q}\left(\norm{\nabla F(\vc{x}_t)}^2 +\sigma^2\right)+ \frac{\varepsilon^2 d^2L^2}{2q}\\
    &\overset{\text{Lemma~\ref{lemma:gradient-and-smoothed-gradient-under-smoothness}}}{\leq} \frac{2d}{q}\left(2\norm{\nabla F_\varepsilon(\vc{x}_t)}^2 +2\varepsilon^2L^2 +\sigma^2\right)+ \frac{\varepsilon^2d^2 L^2}{2q}\\
    &=\frac{d\varepsilon^2L^2}{2q} \left(8+d\right) + \sigma^2\frac{2d}{q}+ \frac{4d}{q}\norm{\nabla F_\varepsilon(\vc{x}_t)}^2\\
    &\overset{q\geq 1}{\leq} \frac{d\varepsilon^2L^2}{2q} \left(8+d\right) + \sigma^2\left(\frac{2d-1}{q}+1\right)+ \left(\frac{4d-1}{q}+1\right)\norm{\nabla F_\varepsilon(\vc{x}_t)}^2
\end{align}
From the affine variance bound we derived, we have
\begin{equation}
    \circled{B} \leq \frac{d\varepsilon^2L^2}{2q} \left(8+d\right) + \sigma^2\left(\frac{2d-1}{q}+1\right)+ \left(\frac{4d-1}{q}+1\right)\norm{\nabla F_\varepsilon(\vc{x}_t)}^2
\end{equation}
Combining the two bounds, we obtain the condition:
\begin{equation}
    \means{\xi_t,\mathcal{U}_q}{\gamma_t \norm{\hat{\nabla} f_{\varepsilon}^q(\vc{x}_t;\xi_t)} \mid \vc{x}_t} \leq \underbrace{\frac{d\varepsilon^2L^2}{2q} \left(8+d\right) + \sigma^2\left(\frac{2d-1}{q}+1\right)}_{\sigma_0^2}+ \underbrace{\left(\frac{4d-1}{q}+1\right)}_{\sigma_1^2+1}\norm{\nabla F_\varepsilon(\vc{x}_t)}^2.
\end{equation}
By our choice of $\beta, \eta, \zeta$, we have satisfied the conditions required by Theorem~\ref{thm:adaptive-sgd-affine-meazo}. Applying Theorem~\ref{thm:smoothed-proxy-opt} with $C=\frac{d}{d+2}\leq 1$ yields:
\[
\boxed{\frac{1}{T}\sum_{t=0}^{T-1}\mean{\norm{\nabla F(\vc{x}_t)}^2}
\;\le\;
2\big(\sqrt{\beta}\,G+\zeta\big)\left[
\frac{F(\vc{x}_0)-F^*}{\eta\,T} \;+\; \frac{\sigma_0^2}{2\,\zeta}
\right] + \varepsilon^2L^2\left(\frac{\sqrt{\beta}\,G+\zeta}{\eta T}+2\right).}
\]

\end{proof}

\clearpage

\subsection{Proof of Corollary~\ref{cor:meazo-convergence-rate-informal}}\label{app:proof-corollary-informal}
We provide the proof of Corollary~\ref{cor:meazo-convergence-rate-informal} restated below:
\begin{corollary}[MEAZO Convergence Rate - Informal]\label{cor:meazo-convergence-rate-informal-restated}
If the assumptions and conditions in Theorem~\ref{thm:meazo} hold, then after $T$ iterations of Algorithm~\ref{alg:meazo} with $\vc{u}\sim\mathrm{Uniform}(\sphere)$, the iterates satisfy
\[
\frac{1}{T}\sum_{t=0}^{T-1}\mean{\norm{\nabla F(\vc{x}_t)}^2}
\;\le\;
O\!\left(\ \underbrace{\frac{F(\vc{x}_0)-F^*}{\eta\,T} + \sigma^2}_{\text{Standard FO rate}} + \underbrace{\frac{d}{q} \sigma^2}_{\text{Cross-stochastic gradient noise}} + \underbrace{\frac{d^2}{q} \varepsilon^2}_{\text{ZO gradient estimation error}}  \right),
\]
where the $O(\cdot)$ notation hides absolute constants independent of $d,q,T,$ and $\sigma^2$, and non-dominant terms in $\varepsilon^2$.
\end{corollary}
\begin{proof}
Recall from that with
\[
\sigma_0^2 \;=\; \frac{d\varepsilon^2 L^2}{2q}(8+d) + \left(\frac{2d-1}{q}+1\right)\sigma^2 .
\]
Let $d_q \coloneqq \frac{4d-1}{q}+1$, then notice that $1+\sigma_1^2 = 1 + \frac{4d-1}{q} = d_q$.

Theorem~\ref{thm:meazo} gives
\begin{align*}
    &\frac{1}{T}\sum_{t=0}^{T-1}\mean{\norm{\nabla F(\vc{x}_t)}^2} \\
    &\leq 2\big(\sqrt{\beta}\,G+\zeta\big)\left[
        \frac{F(\vc{x}_0)-F^*}{\eta\,T} + \frac{\sigma_0^2}{2\,\zeta}
    \right] 
    + \varepsilon^2L^2\left(\frac{\sqrt{\beta}\,G+\zeta}{\eta T}+2\right) \\
    &= 2\big(\sqrt{\beta}\,G+\zeta\big) \frac{F(\vc{x}_0)-F^*}{\eta\,T} 
    + \frac{2\big(\sqrt{\beta}\,G+\zeta\big)}{2\,\zeta} \frac{d\varepsilon^2 L^2}{2q}(8+d) + \left(\frac{2d-1}{q}+1\right)\sigma^2 \\
    &+ \varepsilon^2L^2\left(\frac{\sqrt{\beta}\,G+\zeta}{\eta T}+2\right) \\
    &= O \left( \frac{F(\vc{x}_0)-F^*}{\eta\,T} \right)
    + O \left( \frac{d \varepsilon^2}{q} + \frac{d^2 \varepsilon^2}{q} \right) 
    + O \left( \left(\frac{d}{q} + 1\right) \sigma^2 \right)
    + O \left( \varepsilon^2 \left(\frac{1}{\eta T} + 2\right) \right) \\
    &= O \left( \frac{F(\vc{x}_0)-F^*}{\eta\,T} + \sigma^2 \right)
    + O \left( \frac{d}{q} \sigma^2 \right)
    + O \left( \frac{\varepsilon^2}{\eta T} + \varepsilon^2 + \frac{d \varepsilon^2}{q} + \frac{d^2 \varepsilon^2}{q} \right) \\
    &= O \left( \frac{F(\vc{x}_0)-F^*}{\eta\,T} + \sigma^2
    + \frac{d}{q} \sigma^2
    + \frac{d^2}{q} \varepsilon^2 \right),
\end{align*}
where we started by expanding $\sigma_0^2$ and then dropped non-dominant terms in $\varepsilon^2$.
\end{proof}
% We observe that the first term correspond to the standard convergence rate of FO Adam to a small neighborhood of size $O \left( \sigma^2 \right)$ around stationary point proved in~\cite{zaheer2018adaptive}, which is often sufficient for large-scale machine learning problems in line with the risk–computation tradeoff emphasized by~\cite{bottou2010large}.

% As defined in Assumption~\ref{assumption:bounded-var} this constant $O \left( \sigma^2 \right)$ is inherent to the stochasticity of the gradients, themselves estimated with ZO method.
% Thus the second term implies a convergence of MEAZO to a small neighborhood of size $O \left( \frac{d}{q} \sigma^2 \right)$ around stationary point. 
% Unlike the previous noise term, the latter can be reduced by increasing the number of samples $q$. 
% %%%%%% We can remove this next sentence if it sounds to much like a flaw of our convergence proof
% % Note that it is also coherent with standard decreasing step size analyses, for instance Corollary 3.3 in~\cite{ghadimi2013stochastic}, where $\eta = O \left( \frac{1}{\sigma \sqrt{d T}} \right)$ to get the noise term converge in $O \left( \frac{\sqrt{d}}{\sqrt{T}} \right)$.

% To get a convergence in $O\left( \frac{d}{T} \right)$ for the third ZO gradient estimation term, we retrieve the need of setting the perturbation magnitude to a small value often proportional to $\varepsilon = O \left( \frac{1}{\sqrt{d T}} \right)$, like for ZO-AdaMM under non-convex assumptions~\cite{chen2019zo}.

\clearpage

\clearpage
\subsection{An Alternative Proof for ZO-SGD}\label{app:alternative-zo-sgd-proof}
Using the same proof technique for MEAZO, we provide here an alternative proof for ZO-SGD. We first show this basic result for affine variance-bounded SGD:
\begin{theorem}[SGD with Affine Variance Bound]\label{thm:sgd-affine}
Let $F: \reals^d \to \reals$ be the population loss defined as 
\[
F(\vc{x}) = \means{\xi \sim \mathcal{D}}{f(\vc{x}; \xi)},
\]
where each individual sample loss $f(\vc{x}; \xi)$ is $L$-smooth. Assume that $F$ is bounded from below by $F^*$, and that for stochastic samples $\xi \sim \mathcal{P}$, we have access to an unbiased sample gradient estimator that satisfies:
\[
\means{\xi}{\norm{g(\vc{x}; \xi) - \nabla F(\vc{x})}^2} \le \sigma_0^2 + \sigma_1^2 \norm{\nabla F(\vc{x})}^2\]
Consider the SGD update:
\[
\vc{x}_{t+1} = \vc{x}_t - \eta g(\vc{x}_t; \xi_t),
\]
with constant step size $\eta > 0$.

If we run SGD with $\eta < \frac{2}{(1+\sigma_1^2)L}$ for $T$ iterations, we have:
\begin{align}
\frac{1}{T}\sum_{t=0}^{T-1}\mean{\norm{\nabla F(\vc{x}_t)}^2}
&\leq \frac{D}{\eta T\left(1-\frac{L\eta}{2}\left(1+\sigma_1^2\right) \right)} + \frac{L\eta\sigma_0^2}{2-L\eta\left(1+\sigma_1^2\right)}
\end{align}
where $D=F(\vc{x}_0)-F^* >0$.
\end{theorem}
% \hd{The fact vanilla SGD converges in the affine setting is not surprising, but I could not find a proof of it, so I provide one here}
\begin{proof}
Define $\vc{g}_t=g(\vc{x}_t;\xi_t)$, and consider the descent lemma at iteration $t$:
\begin{equation}
    F(\vc{x}_{t+1}) \leq F(\vc{x}_t) -\eta \left\langle \nabla F(\vc{x}_t),  \vc{g}_t\right\rangle + \frac{L \eta^2}{2}\norm{\vc{g}_t}^2
\end{equation}
Taking the expectation conditioned on $\vc{x}_t$:
\begin{align}
     \means{\xi_t}{F(\vc{x}_{t+1}) | \vc{x}_t} &\leq F(\vc{x}_t) -\eta\means{\xi_t}{\left\langle \nabla F(\vc{x}_t),  \vc{g}_t\right\rangle \mid \vc{x}_t} + \frac{L \eta^2}{2}\means{\xi_t}{\norm{\vc{g}_t}^2\mid \vc{x}_t}\\
     &=F(\vc{x}_t) -\eta\norm{\nabla F(\vc{x}_t)}^2+ \frac{L \eta^2}{2}\means{\xi_t}{\norm{\vc{g}_t}^2\mid \vc{x}_t}\\
     &\overset{\text{affine var. bound}}{\leq} F(\vc{x}_t) -\eta\norm{\nabla F(\vc{x}_t)}^2+ \frac{L \eta^2}{2}\left(\sigma_0^2 + (1+\sigma^2_1) \norm{\nabla F(\vc{x}_t)}^2\right)
\end{align}
Taking the full expectation yields:
\begin{equation}
     \mean{F(\vc{x}_{t+1})} \leq \mean{F(\vc{x}_t)} + \eta\left(\frac{L\eta}{2}\left(1+\sigma_1^2\right) -1\right)\mean{\norm{\nabla F(\vc{x}_t)}^2} + \frac{L\eta^2\sigma_0^2}{2}
\end{equation}
Summing over $t=0, \cdots,T-1$ and telescoping:
\begin{equation}
    \mean{F(\vc{x}_{T})} - F(\vc{x}_0)\leq \sum_{t=0}^{T-1}\left[\eta\left(\frac{L\eta}{2}\left(1+\sigma_1^2\right) -1\right)\mean{\norm{\nabla F(\vc{x}_t)}^2} + \frac{L\eta^2\sigma_0^2}{2}\right]
\end{equation}
Assume $1-\frac{L\eta(1+\sigma_1^2)}{2} >0 \implies \eta <\frac{2}{L(1+\sigma_1^2)}$, then dividing by $T$ and re-arranging implies:
\begin{align}
    \frac{1}{T}\sum_{t=0}^{T-1}\mean{\norm{\nabla F(\vc{x}_t)}}^2 &\leq \frac{F(\vc{x}_0)-\mean{F(\vc{x}_{T})}}{\eta T\left(1-\frac{L\eta}{2}\left(1+\sigma_1^2\right) \right)} + \frac{\frac{L\eta^2\sigma_0^2}{2}}{\eta \left(1-\frac{L\eta}{2}\left(1+\sigma_1^2\right)\right)} \\
    &=\frac{F(\vc{x}_0)-\mean{F(\vc{x}_{T})}}{\eta T\left(1-\frac{L\eta}{2}\left(1+\sigma_1^2\right) \right)} + \frac{L\eta\sigma_0^2}{2-L\eta\left(1+\sigma_1^2\right)}  \\
    &\leq \frac{F(\vc{x}_0)-F^*}{\eta T\left(1-\frac{L\eta}{2}\left(1+\sigma_1^2\right) \right)} + \frac{L\eta\sigma_0^2}{2-L\eta\left(1+\sigma_1^2\right)} \\ 
    &= \frac{D}{\eta T\left(1-\frac{L\eta}{2}\left(1+\sigma_1^2\right) \right)} + \frac{L\eta\sigma_0^2}{2-L\eta\left(1+\sigma_1^2\right)}
\end{align}
\end{proof}

\begin{theorem}[ZO-SGD]\label{thm:zo_sgd_convergence}
Let $F: \reals^d \to \reals$ be the population loss defined as 
\[
F(\vc{x}) = \means{\xi \sim \mathcal{D}}{f(\vc{x}; \xi)},
\]
where each individual sample loss $f(\vc{x}; \xi)$ is $L$-smooth. Assume that for stochastic samples $\xi \sim \mathcal{D}$, the variance of the sample loss gradient satisfies
\[
\means{\xi}{\norm{\nabla f(\vc{x}; \xi) - \nabla F(\vc{x})}^2} \le \sigma^2.
\]

At iteration $t$, the gradient is approximated by the \eqref{eq:zo-uniform} with Uniform perturbations. Consider the ZO-SGD update:
\[
\vc{x}_{t+1} = \vc{x}_t - \eta \hat{\nabla} f_{\varepsilon}^q(\vc{x}_t;\xi_t),
\]
with constant step size $\eta > 0$ and smoothing parameter $\varepsilon > 0$.
Define
\[
    \sigma_0^2 = \frac{d\varepsilon^2L^2}{2q}(8+d) + \left(\frac{2d-1}{q} +1\right)\sigma^2, 
    \qquad
    \sigma_1^2 = \frac{4d-1}{q}.
\]

If we run SGD with $\eta < \frac{2}{(1+\sigma_1^2)L}$ for $T$ iterations, we have:
\begin{align}
\frac{1}{T}\sum_{t=0}^{T-1}\mean{\norm{\nabla F(\vc{x}_t)}^2} \leq \frac{F(\vc{x}_0)-F^*}{\eta T \left(1-\frac{L\eta}{2}\left(1+\sigma_1^2\right) \right)} + \frac{L\eta\sigma_0^2}{2-L\eta\left(1+\sigma_1^2\right)}+ \varepsilon^2L^2\left(\frac{1}{2LT\eta\left(1-\frac{L\eta}{2}\left(1+\sigma_1^2\right) \right)}+2\right)
\end{align}

\end{theorem}
\begin{proof}
To prove the result for ZO-SGD, we need to show that the conditions in Theorem~\ref{thm:sgd-affine} are met for the smoothed function $F_\varepsilon$ with ZO gradient estimate $\hat{\nabla} f_{\varepsilon}^q (\vc{x} ; \xi)$, and then apply Theorem~\ref{thm:smoothed-proxy-opt-restated} to map the guarantee back to $F$.

\paragraph{Step 1: Affine bound on the smoothed functions gradient.}
By assumption, sample loss functions $f (\cdot;\xi)$ are $L$-smooth and their gradient satisfies the variance bound w.r.t. the population loss gradient $\nabla F (\cdot)$.
So we can prove the affine variance bound for $\hat{\nabla} f_{\varepsilon}^q(\vc{x}; \xi)$ w.r.t. $F_\varepsilon$ by Lemma~\ref{lemma:affine-var-bound}, that is
\begin{equation}
    \means{\xi,\calU_q}{\hat{\nabla} f_{\varepsilon}^q(\vc{x}; \xi)} =  \nabla F_\varepsilon(\vc{x}),
\end{equation}
and 
\begin{align*}
    \means{\xi,\calU_q}{\norm{\hat{\nabla} f_{\varepsilon}^q(\vc{x}; \xi) - \nabla F_\varepsilon(\vc{x})}^2} \leq \underbrace{\frac{d\varepsilon^2L^2}{2q} \left(8+d\right) + \sigma^2\left(\frac{2d-1}{q}+1\right)}_{\sigma_0^2}+ \underbrace{\frac{4d-1}{q}}_{\sigma_1^2}\norm{\nabla F_\varepsilon(\vc{x})}^2.
\end{align*}

\paragraph{Step 2: Convergence w.r.t. the smoothed population loss.}
By Lemma~\ref{lemma:smoothing-preserves-smoothness}, both $f_\varepsilon(\cdot;\xi)$ and $F_\varepsilon$ are $L$-smooth.
Let us set $\eta < \frac{2}{(1+\sigma_1^2)L}$.
So one can apply Theorem~\ref{thm:sgd-affine}, with the above (smoothed) stochastic gradient variance bound, to prove the convergence of SGD w.r.t. $F_\varepsilon$
\begin{align*}
    \frac{1}{T}\sum_{t=0}^{T-1}\mean{\norm{\nabla F_\varepsilon (\vc{x}_t)}^2}
    &\leq \underbrace{\frac{1}{\eta T\left(1-\frac{L\eta}{2}\left(1+\sigma_1^2\right) \right)}}_{K_0} \left( F_\varepsilon(\vc{x}_0)-F_\varepsilon^* \right) + \underbrace{\frac{L\eta\sigma_0^2}{2-L\eta\left(1+\sigma_1^2\right)}}_{K_1} .
\end{align*}

\paragraph{Step 3: Convergence w.r.t. the population loss.}
As the sample loss $f(\cdot, \xi)$ are $L$-smooth, the population loss $F$ is $L$-smooth too.
For the ZO estimator in \eqref{eq:zo-uniform}, we have the smoothing distribution to be the unit ball $\vc{v}\sim\text{Uniform}(\ball)$, which implies $C = \means{\vc{v}\sim\ball}{\norm{\vc{v}}^2} = \frac{d}{d+2}$.
Therefore, from Theorem~\ref{thm:smoothed-proxy-opt-restated} with $C=\frac{d}{d+2}\leq1$, $\sigma_0^2$, $\sigma_1^2$, $K_0$ and $K_1$ as above, we get after $T$ iterations of ZO-SGD with $\eta<\frac{2}{L(1+\sigma_1^2)}$:
\[
\boxed{\frac{1}{T}\sum_{t=0}^{T-1}\mean{\norm{\nabla F(\vc{x}_t)}^2} \leq \frac{F(\vc{x}_0)-F^*}{\eta T\left(1-\frac{L\eta}{2}\left(1+\sigma_1^2\right) \right)} + \frac{L\eta\sigma_0^2}{2-L\eta\left(1+\sigma_1^2\right)}+ \varepsilon^2L^2\left(\frac{1}{2L\eta\left(1-\frac{L\eta}{2}\left(1+\sigma_1^2\right) \right)}+2\right)} .
\]
\end{proof}

A useful sanity check is that, in the limit where $q\to \infty$ (eliminating ZO variance) and $\varepsilon\to0$ (eliminating ZO bias), our bound reduces exactly to the classical SGD guarantee for $L$‑smooth objectives with variance‑bounded noise (Equation~(2.4) of \cite{ghadimi2013stochastic}):

\[
\frac{1}{T}\sum_{t=0}^{T-1}\mean{\norm{\nabla F(\vc{x}_t)}^2} \leq \frac{F(\vc{x}_0)-F^*}{\eta T\left(1-\frac{L\eta}{2} \right)} + \frac{L\eta\sigma^2}{2-L\eta}.
\]
% \hd{Corollary to show final rate. Will work on this after I get confirmation that the MEAZO corollary is correct.}

% \begin{corollary}
%     Let $\tilde{d}_q \defeq 1 + \sigma^1$, with above notations. 
%     Suppose that the step size is set to the following constant value
%     \begin{equation}\label{eq:constant_step_size_zo_sgd}
%         \eta \defeq \min \left\{ \frac{1}{L \tilde{d}_q} , \frac{\tilde{D}}{\sigma_0 \sqrt{\tilde{d}_q T}}\right\}, \quad \forall t \in [T],
%     \end{equation}
%     for some $\tilde{D} > 0$, and that the smoothing parameter is set such that
%     \begin{equation}\label{eq:constant_smoothing_parameter_zo_sgd}
%         \varepsilon \leq \frac{D_F}{\tilde{d}_q \sqrt{2 T}} .
%     \end{equation}
%     Then, under the same assumptions as in Theorem~\ref{thm:zo_sgd_convergence}, and assuming that $q \leq d^{?}$, we have
%     \begin{equation}\label{eq:zo_sgd_convergence_ineq_simplified}
%         \frac{1}{T}\sum_{t=1}^{T}\mean{\norm{\nabla F(\vc{x}_t)}^2} \leq \ldots
%     \end{equation}
% \end{corollary}

% \section{OLD PROOFS}
% \input{proofs/old_stuff}

%%%%%%%%%%%%%%%%%%%%%%%%%%%%%%%%%%%%%%%%%%%%%%%%%%%%%%%%%%%%%%%%%%%%%%%%%%%%%%%
%%%%%%%%%%%%%%%%%%%%%%%%%%%%%%%%%%%%%%%%%%%%%%%%%%%%%%%%%%%%%%%%%%%%%%%%%%%%%%%

\end{document}